\documentclass[letterpaper, 10 pt, conference]{ieeeconf}  

\IEEEoverridecommandlockouts                              

\usepackage{xcolor}
\usepackage{comment}
\usepackage{cite}

\usepackage[mathcal]{euscript}
\usepackage[cmex10]{amsmath}
\usepackage{amssymb}

\usepackage{amsthm} 
  
\usepackage{dsfont} 
\usepackage{mathabx} 

\usepackage{graphicx}
\usepackage{pgfplots}
\usetikzlibrary{arrows.meta}
\usepackage{multirow}
\usepackage{enumerate}

\usepackage[ruled, vlined, linesnumbered]{algorithm2e}
\SetAlCapFnt{\small}
\SetNlSty{}{}{} 
\SetInd{0.5em}{0.2em}
\SetAlgoSkip{9.5pt}

\newtheoremstyle{plain}
	  {}
	  {}
	  {\itshape}
	  {}
	  {\bfseries}
	  {}
	  {5pt plus 1pt minus 1pt}
	  {}

\newtheoremstyle{definition}
  	  {}
	  {}
	  {\normalfont}
	  {}
	  {\bfseries}
	  {}
	  {5pt plus 1pt minus 1pt}
	  {}
  
\theoremstyle{plain}
\newtheorem{assumption}{Assumption}
	
\newcommand{\refeq}[1]			{(\ref{#1})} 
\newcommand{\reffig}[1]			{Fig. \ref{#1}} 
\newcommand{\refsec}[1]			{Section \ref{#1}}
\newcommand{\refapp}[1]			{Appendix \ref{#1}}
\newcommand{\reftab}[1]			{Table \ref{#1}}
\newcommand{\refthm}[1]			{Theorem \ref{#1}}
\newcommand{\refprop}[1]		{Proposition \ref{#1}}

\newcommand{\reflem}[1]			{Lemma \ref{#1}}
\newcommand{\refdef}[1]			{Definition \ref{#1}}
\newcommand{\refasm}[1]			{Assumption \ref{#1}}
\newcommand{\refalg}[1]			{Algorithm \ref{#1}}

\newcommand{\reffn}[1] 		    {\textsuperscript{\ref{#1}}}

\theoremstyle{plain}

\newcommand{\freespace}	{\mathcal{F}} 

\newcommand{\workspace}	{\mathcal{W}} 
\newcommand{\obstspace}	{\mathcal{O}} 

\newcommand{\maxsenserange}{\mathrm{r}_{\max}}
\newcommand{\scanradius}{\mathrm{r}_{\max}}
\newcommand{\pointcloud}{P}
\newcommand{\scanpoint}{\mathrm{p}}
\newcommand{\scanpoints}{P}
\newcommand{\scanset}{\mathcal{S}}
\newcommand{\numberofscan}{n}
\newcommand{\scanpoly}{\mathrm{scanpoly}}
\newcommand{\safescanpoly}{\mathrm{safepoly}}
\newcommand{\scancenter}{\vect{c}}

\newcommand{\safepoly}[1]{\mathrm{safepoly}(#1)}
\newcommand{\saferpoly}[1]{\mathrm{saferpoly}(#1)}

\newcommand{\obstdist}{\mathrm{dist2obst}}
\newcommand{\bndrydist}{\mathrm{dist2bnd}}

\newcommand{\activepolicy}{\mathrm{activescan}}

\newcommand{\policycost}{\mathrm{scancost}}
\newcommand{\policygoal}{\mathrm{scangoal}}
\newcommand{\navcost}{\mathrm{navcost}}
\newcommand{\localcost}{\mathrm{localcost}}

\newcommand{\R}  	{\mathbb{R}} 

\newcommand{\radius} 	{\rho}

\DeclareMathOperator*{\conv}{conv}
\newcommand{\interior}{\mathrm{int}}

\newcommand{\ball}{\mathrm{B}}  

\newcommand{\erode}{\mathrm{erode}}

\newcommand{\pos} 		{\vect{x}} 			

\newcommand{\goal}		{\vect{x^{*}}}
\newcommand{\ctrlgoal}	{\vect{y}} 

\newcommand{\gain}   	{\kappa} 			
\newcommand{\ctrl}      {\vect{u}} 			

\newcommand{\graph}{G}
\newcommand{\vertexset}{V}
\newcommand{\edgeset}{E}

\newcommand{\subgraph}{\overline{G}}
\newcommand{\subvertexset}{\overline{V}}
\newcommand{\subedgeset}{\overline{E}}

\newcommand{\proj} {\Pi}

\let\originalleft\left
\let\originalright\right
\renewcommand{\left}{\mathopen{}\mathclose\bgroup\originalleft}
\renewcommand{\right}{\aftergroup\egroup\originalright}
	
\newcommand{\plist}[1] 	{\left(#1\right)} 
\newcommand{\blist}[1]	{\left[ #1 \right]} 
\newcommand{\clist}[1]	{\left\{#1\right\}} 

\newcommand{\vect}[1]   {\mathrm{#1}}

\newcommand{\tr}[1] {{#1}^{\mathrm{T}}} 
\newcommand{\norm}[1]  {\|#1\|}

\newcommand{\argmin}{\operatornamewithlimits{arg\ min}} 
\newcommand{\diff} {\mathrm{d}}

\newtheoremstyle{plain}
	  {}
	  {}
	  {\itshape}
	  {}
	  {\bfseries}
	  {}
	  {5pt plus 1pt minus 1pt}
	  {}

\newtheoremstyle{definition}
  	  {}
	  {}
	  {\normalfont}
	  {}
	  {\bfseries}
	  {}
	  {5pt plus 1pt minus 1pt}
	  {}
  
\theoremstyle{plain}

\newtheorem{theorem}{Theorem}
\newtheorem{lemma}{Lemma}
\newtheorem{proposition}{Proposition}

\theoremstyle{definition}
\newtheorem{definition}{Definition}

\title{\LARGE \bf 
Key-Scan-Based Mobile Robot Navigation:
\\
\scalebox{0.95}{Integrated Mapping, Planning, and Control using Graphs of Scan Regions}
\\
{\large(Technical Report)}
}

\author{Dharshan Bashkaran Latha and \"{O}m\"{u}r Arslan
\thanks{The authors are with the Department of Mechanical Engineering, Eindhoven University of Technology, P.O. Box 513, 5600 MB Eindhoven, The Netherlands. The authors are also affiliated with the Eindhoven AI Systems Institute. Emails:  d.bashkaran.latha@student.tue.nl, o.arslan@tue.nl}%
}

\begin{document}

\maketitle
\thispagestyle{empty}
\pagestyle{empty}

\begin{abstract}
Safe autonomous navigation in a priori unknown environments is an essential skill for mobile robots to reliably and adaptively perform diverse tasks (e.g., delivery, inspection, and interaction) in unstructured cluttered environments. 
Hybrid metric-topological maps, constructed as a pose graph of local submaps, offer a computationally efficient world representation for adaptive mapping, planning, and control at the regional level.
In this paper, we consider a pose graph of locally sensed star-convex scan regions as a metric-topological map, with star convexity enabling simple yet effective local navigation strategies.
We design a new family of safe local scan navigation policies and present a perception-driven feedback motion planning method through the sequential composition of local scan navigation policies, enabling provably correct and safe robot navigation over the union of local scan regions.
We introduce a new concept of bridging and frontier scans for automated key scan selection and exploration for integrated mapping and navigation in unknown environments.
We demonstrate the effectiveness of our key-scan-based navigation and mapping framework using a mobile robot equipped with a 360$^{\circ}$ laser range scanner in 2D cluttered environments through numerical ROS-Gazebo simulations and real hardware~experiments.
\end{abstract}

\section{Introduction}
\label{sec.introduction}

The ability to safely and smoothly navigate in unknown unstructured environments is crucial for autonomous robots to reliably and adaptively perform various tasks, such as logistics \cite{renan_nascimento_RAS2021, taranta_etal_INDIN2021}, assistance \cite{bettencourt_lima_ICARSC2021,  gul_rahiman_alhady_sahal_CE2019}, inspection and surveillance \cite{paola_etal_IJARS2010, halder_afsari_AS2023}.
Closing the gap between perception and action for autonomous navigation in such application settings positively impacts adaptability, flexibility, and robustness \cite{oriolo_ulivi_vendittelli_TSMC1998, placed_etal_TRO2023}.
Hybrid metric-topological maps, for instance, constructed as a pose graph of local submaps, offer a computationally efficient world representation for adaptive mapping, strategic planning, and reliable control at the regional level \cite{thrun_AI1998, konolige_marder_marthi_ICRA2011}.

In this paper, we consider a pose graph of local scan regions to systematically and tightly integrate mapping, planning, and control for improved performance and computationally efficiency in both perception and action.
We present a perception-driven feedback motion planning approach for safe global navigation in unknown environments by incrementally deploying and sequentially composing simple local scan navigation policies using a graph of star-convex scan polygons.
We introduce the notion of bridging and frontier scans for key scan selection and exploration to enhance the topological connectivity and coverage of scan regions while avoiding redundant measurements, as illustrated in \reffig{fig.key_scan_based_navigation_mapping}.

\begin{figure}[t]
\centering
\begin{tabular}{@{}c@{\hspace{0.75mm}}c@{\hspace{0.75mm}}c@{}}
\begin{tabular}{@{}c@{\hspace{0.5mm}}c@{}}
\includegraphics[width=0.16\columnwidth]{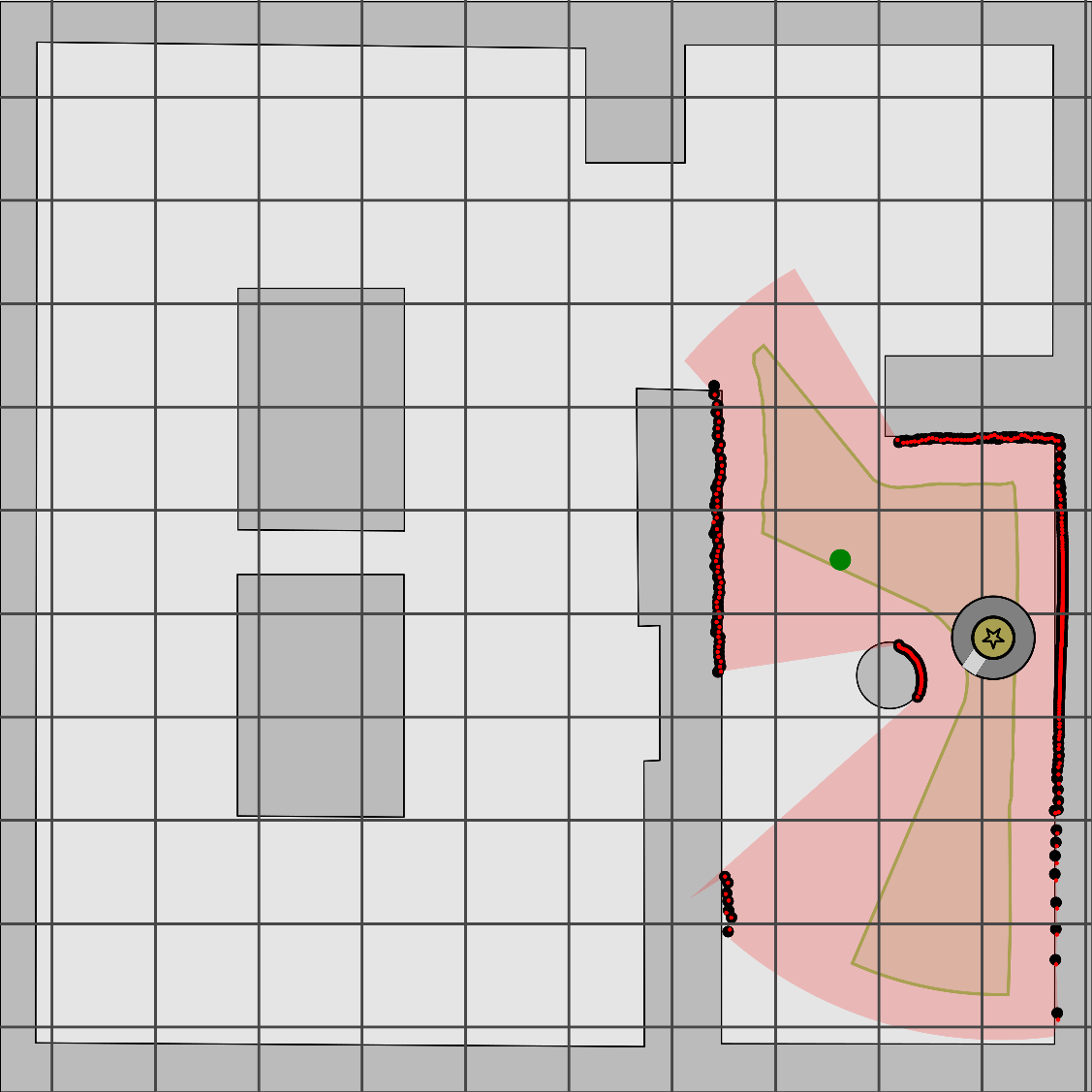}
&
\includegraphics[width=0.16\columnwidth]{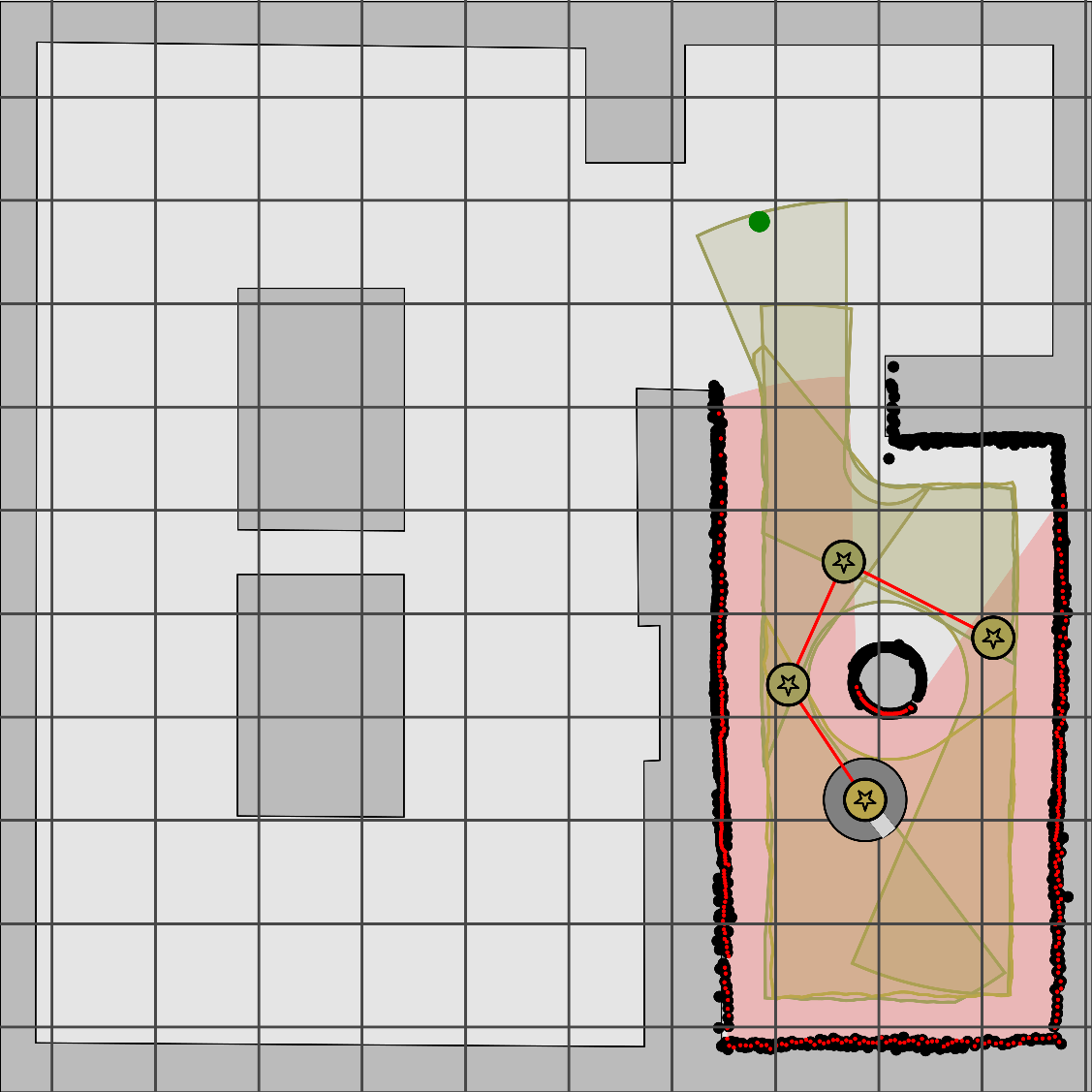}
\\[-0.9mm]
\includegraphics[width=0.16\columnwidth]{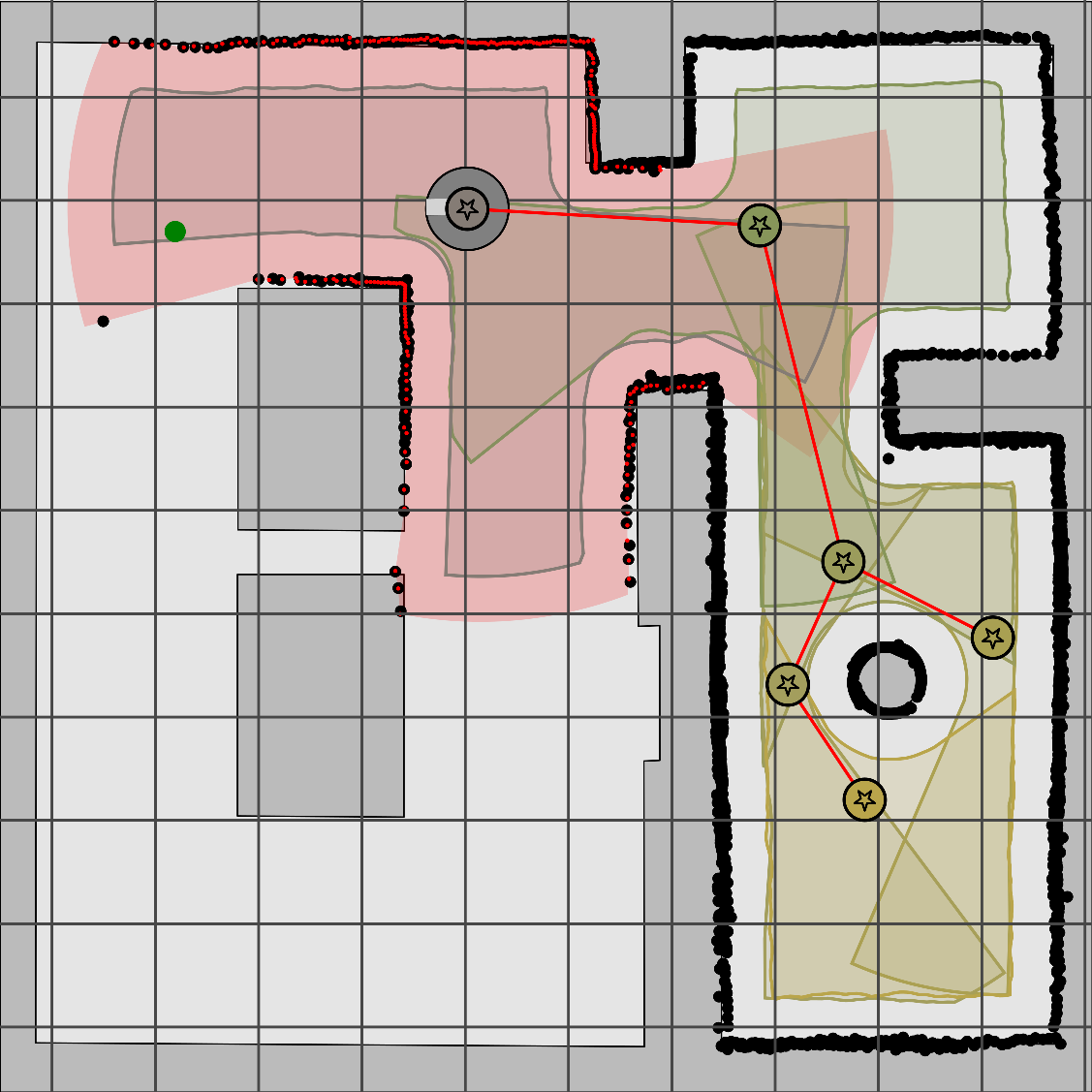}
&
\includegraphics[width=0.16\columnwidth]{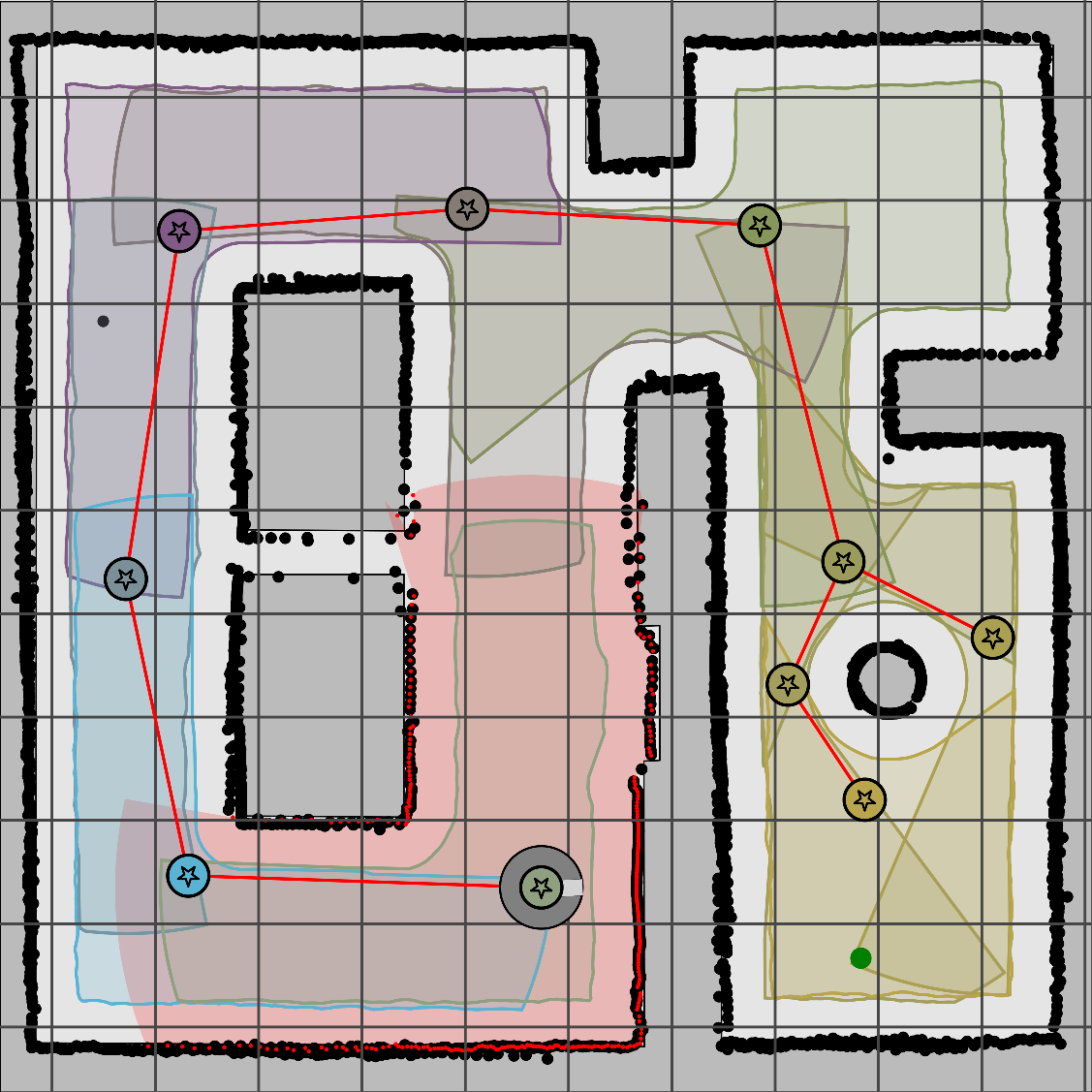}
\end{tabular}
&
\begin{tabular}{@{}c@{}}
\includegraphics[width=0.325\columnwidth]{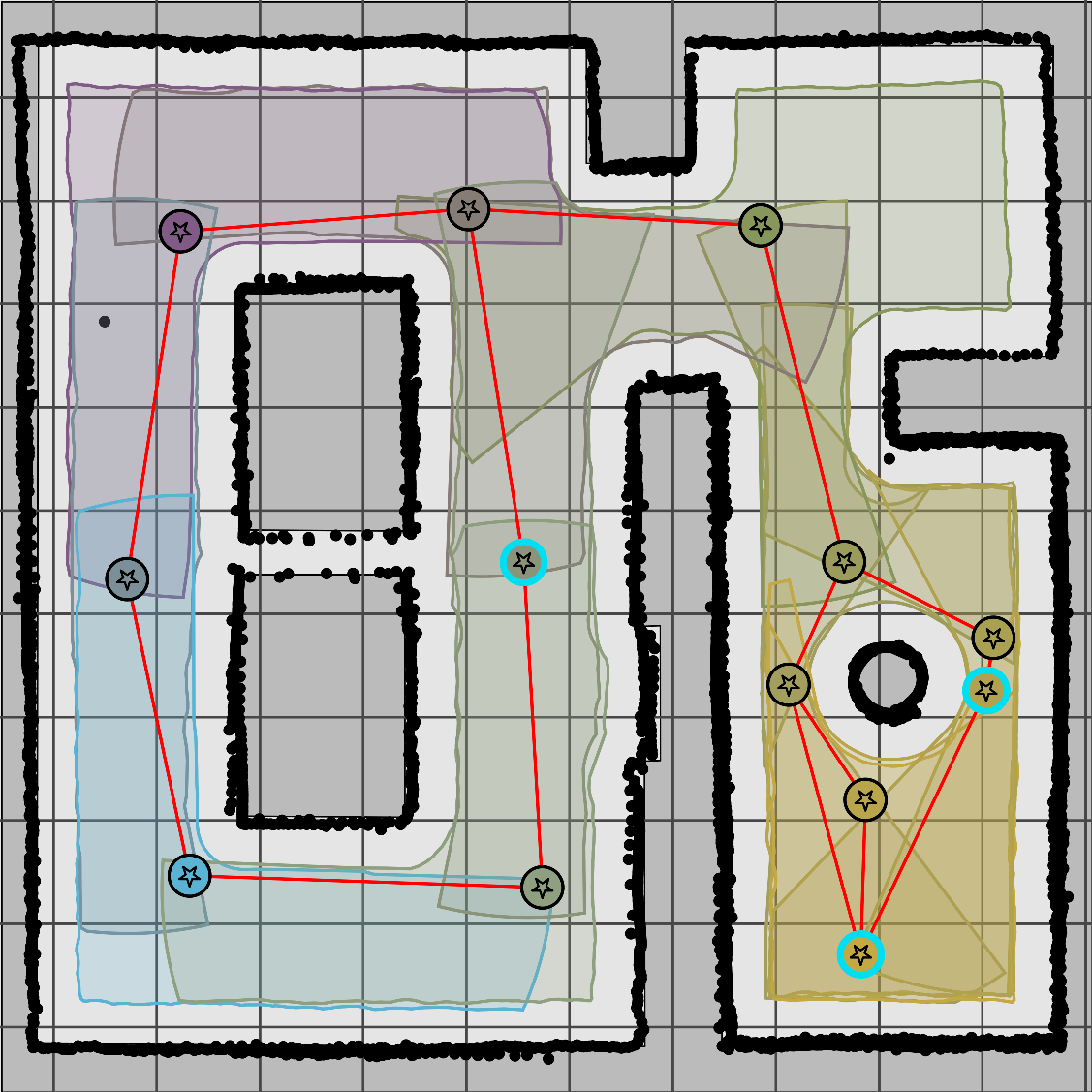}
\end{tabular}
& 
\begin{tabular}{@{}c@{}}
\includegraphics[width=0.325\columnwidth]{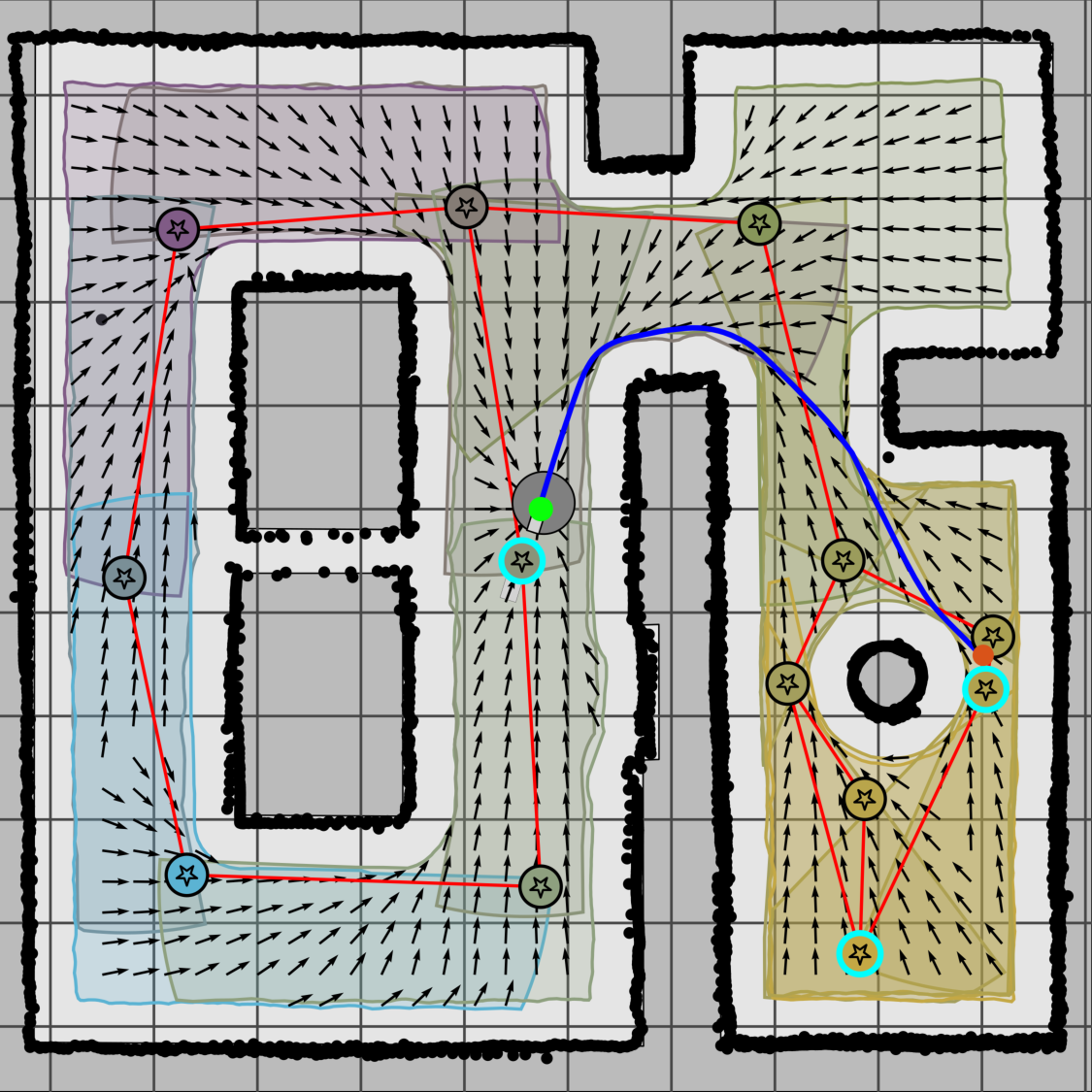}
\end{tabular}
\\[-0.65mm]
\includegraphics[width=0.325\columnwidth]{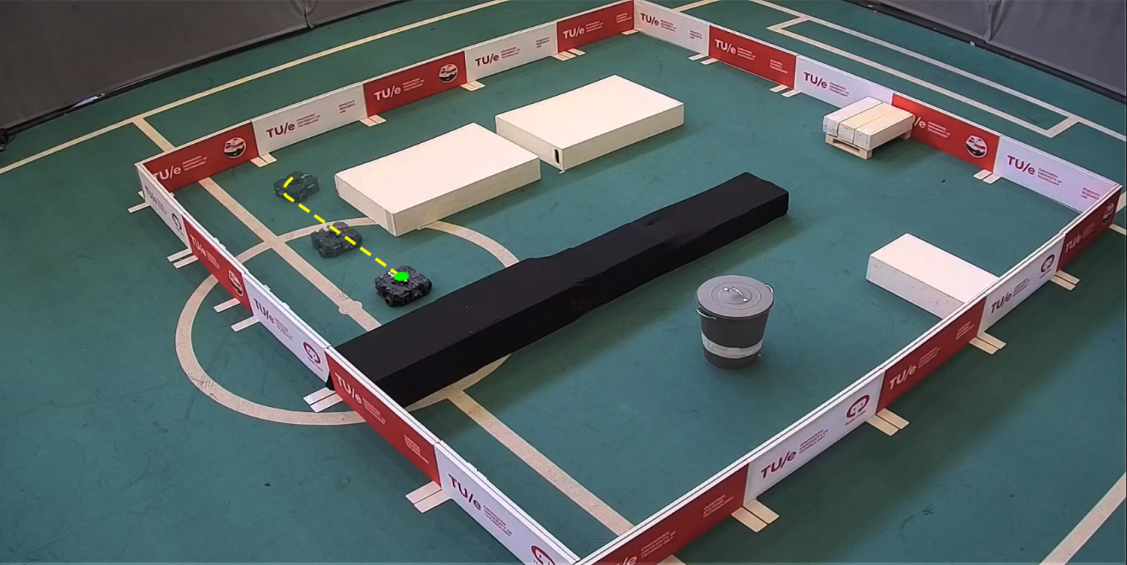}
&
\includegraphics[width=0.325\columnwidth]{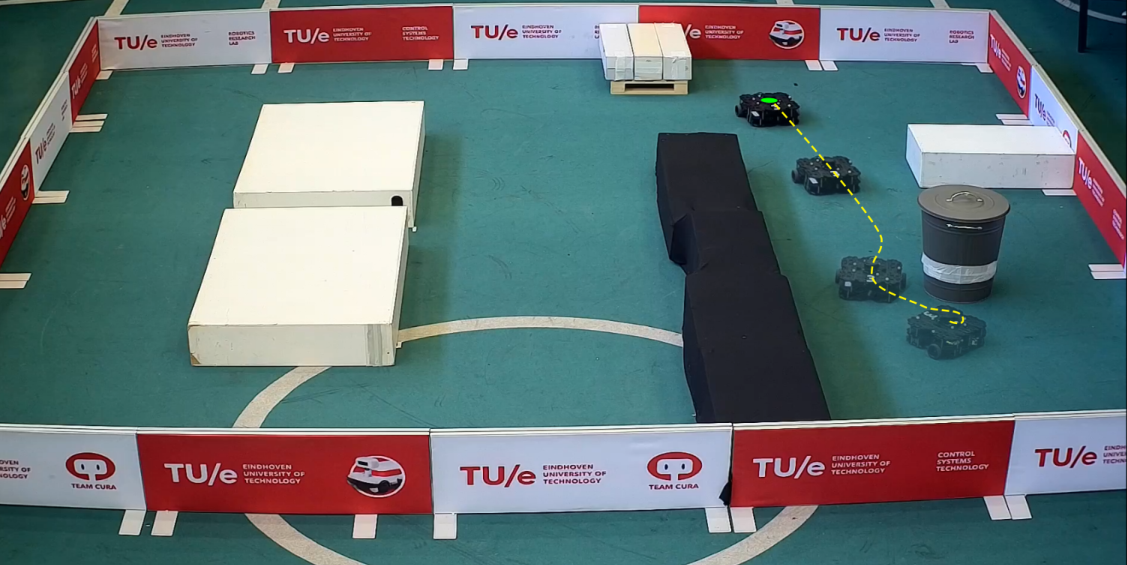}
&
\includegraphics[width=0.325\columnwidth]{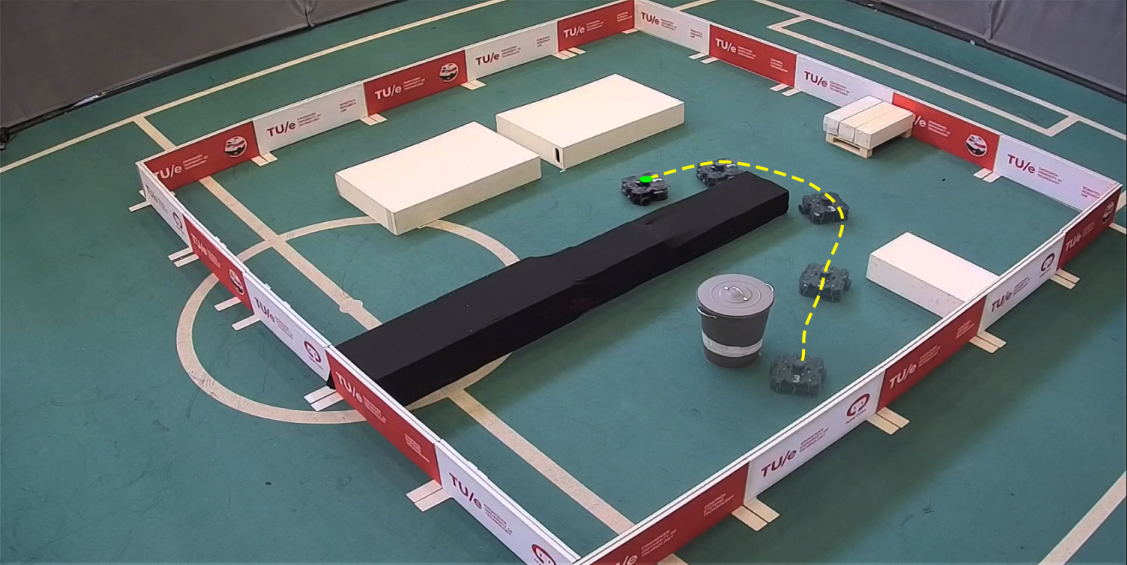}

\end{tabular}
\vspace{-3mm}
\caption{Key-scan-based mobile robot navigation in an unknown cluttered environment by online sequential deployment and composition of local scan navigation policies. (top-left) Automated key scan selection and deployment based on frontier and bridging scans. (top-middle) An incrementally built motion graph of star-convex scan polygons. (top-right) Global feedback motion planning via sequential composition of local scan navigation policies. (bottom) Example robot trajectories during autonomous exploration.}
\label{fig.key_scan_based_navigation_mapping}
\vspace{-2mm}
\end{figure}

\subsection{Motivation and Related Literature}
\label{sec.Literature}

\subsubsection{Integrated Robot Motion Planing and Control}

Safe and smooth motion planning and control is essential for autonomous robots operating around people and other robots, but it is known to be computationally challenging for many robotic systems in complex cluttered environments \cite{canny_ComplexityRobotMotionPlanning1988}. 
Due to real-time operation requirements and onboard computation limitations, many existing robot motion planning and control methods adopt a decoupled high-level planning and low-level control approach, by first finding a collision-free reference path, and then executing the reference plan as accurately as possible through feedback control \cite{choset_etal_PrinciplesOfRobotMotion2005, siciliano_etal_RoboticsModellingPlanningControl2009, lynch_park_ModernRobotics2017}.
However, due to its open-loop nature, such decoupled robot motion planning and control methods often suffer from frequent replanning cycles in practice to ensure safety and consistent system performance \cite{aguiar_hespanha_kokotovic_Automatica2008}, especially in unknown environments \cite{koenig_likhachev_TRO2005, ding_gao_wang_shen_TRO2019, tordesillas_etal_TRO2021}. 
Optimization-based planning and control approaches, such as model predictive control and trajectory optimization, aim to close the gap between high-level planning and low-level control, but this often comes with high computational costs and performance issues due to local minima and initialization, especially in large cluttered environments \cite{nguyen_etal_ECC2021}.
Such technical challenges of optimization-based motion planning can be mitigated by performing motion optimization over a graph of convex sets obtained by a convex decomposition of globally known environments \cite{deits_tedrake_ICRA2015, chen_liu_shen_ICRA2016, liu_etal_RAL2017, gao_etal_JFR2019, marcucci_etal_JoO2024, marcucci_etal_SR2023}, but one still needs replanning under disturbances and in unknown environments \cite{ding_gao_wang_shen_TRO2019, zhou_etal_TRO2021}.    
Feedback motion planners constructed based on artificial potentials \cite{khatib_IJRR1986} and navigation functions \cite{rimon_kod_TRA1992} offer a robust and reliable tightly coupled motion planning and control solution, but they are usually difficult to construct for arbitrary environments without local minima \cite{koren_borenstein_ICRA1991} or high numerical computation costs \cite{barraquand_langlois_latombe_TSMC1992}.
Fortunately, sequential composition \cite{burridge_rizzi_koditschek_IJRR1999} of local feedback control policies allows for computationally efficient feedback motion planning of complex robotic systems in cluttered environments, assuming a (convex) spatial decomposition of globally known environments \cite{belta_isler_pappas_TRO2005, conner_choset_rizzi_tro2009, conner_howie_rizzi_AR2011, arslan_saranli_TRO2012, arslan_guralnik_kod_TRO2016}.
In this paper, as a step toward closing the gap between perception and action \cite{arslan_koditschek_IJRR2019}, we present a perception-driven, integrated planning and control approach for safe and reliable global robot navigation in unknown environments through the online sequential deployment and composition of simple local navigation policies (e.g., navigation through the scan center) over automatically selected critical scan regions measured by onboard sensors.

\subsubsection{Integrated Robot Perception and Action}

Achieving truly safe, reliable, and adaptive robot autonomy in unknown unstructured environments requires leveraging the interaction and dependency between perception and action.
Because robot perception of the environment greatly influences both the quality and process of planning and control, and vice versa \cite{choset_etal_PrinciplesOfRobotMotion2005}.
For example, occupancy grid maps are a widely used metric world model that can be  incrementally built as a fine tessellation of the environment into simple grid shapes \cite{elfes_C1989}. 
However, robot motion planning and control over such grid maps suffer from high computational costs in large environments due to increasing connectivity and collision detection complexity \cite{lavalle_PlanningAlgorithms2006, choset_etal_PrinciplesOfRobotMotion2005}.
To tackle this computational issue, dense metric (e.g., occupancy grid) maps are often segmented into local regions to build sparse topological skeleton maps \cite{thrun_AI1998, bormann_etal_ICRA2016}, enabling computationally efficient high-level global motion planning \cite{blochliger_etal_ICRA2018, chen_etal_IROS2022}.
Alternatively, hybrid metric-topological maps, built as a pose graph of optimally aligned local (e.g., occupancy grid) submaps or sensed (e.g., scan) regions, offer computationally efficient, adaptive incremental mapping \cite{lu_milios_AR1997, grisetti_etal_MITS2010, mendes_koch_lacroix_SSRR2016} and global motion planning \cite{konolige_marder_marthi_ICRA2011}, without constantly requiring topological segmentation of metric maps.
In this paper, we develop an incrementally built motion graph of star-convex scan polygons, enriched with local feedback navigation policies, enabling computationally efficient global feedback motion planning for safe and reliable navigation in unknown environments.


Robot actions also play a significant role in shaping both the quality and process of perception \cite{placed_etal_TRO2023}.
For example, active perception for mapping aims to leverage robot motion planning and control (i.e., action) to obtain a more accurate map representation of the environment \cite{placed_etal_TRO2023, lluvia_lazkano_ansuategi_Sensors2021}.
The most widely used strategy for active mapping is frontier-based exploration  by navigating towards the boundary region between the known obstacle-free space and the unknown unexplored space in a map \cite{yamauchi_CIRA1997}. 
In large environments, frontier-based exploration is usually combined with an active loop-closing strategy to revisit previously visited regions and reduce uncertainty and missing information in mapping and localization, with the active loop-closing decision based on the discrepancy between the topological and metric maps \cite{stachniss_hahnel_burgard_IROS2004, valencia_etal_IROS2012, fermin-leon_neira_castellanos_ECMR2017} or local perceptual saliency \cite{kim_eustice_IJRR2015, suresh_etal_ICRA2020}.
Due to the computational efficiency of high-level motion planning in topological maps, active frontier-based exploration and loop-closing strategies are often adapted for graph-based topological exploration in large environments \cite{yang_etal_ICRA2021}.
In this paper, we introduce new notions of bridging and frontier scanning positions for exploration in a pose graph of scan regions, aiming to build a more accurate and complete topological and metric representation of the environment that supports better robot motion planning and control.

\subsection{Contributions and Organization of the Paper}
\label{sec.Contribution}

This paper introduces an integrated mapping, planning, and control approach using the sequential composition of local scan navigation policies over an incrementally built graph of scan regions for key-scan-based mobile robot navigation in unknown environments.  
In summary, the three major contributions of our paper are as follows:
\begin{itemize}
\item We introduce simple local feedback control policies for safe navigation over star-convex scan polygons using the central connectivity of scan centers (\refsec{sec.local_navigation}).
\item We present a new method for constructing a motion graph of star-convex scan polygons using reciprocal center inclusion and describe how to use it for global feedback motion planning through the sequential composition of local scan navigation policies, ensuring provably correct and safe global navigation  (\refsec{sec.planning_over_graph_of_scans}).
\item We propose new key scan selection criteria to identify bridging and frontier scans to complete missing connectivity in a motion graph of scans and apply it for auto-nomous \mbox{exploration for active mapping (\refsec{sec.automated_deployment}).}    
\end{itemize}
We demonstrate the effectiveness of our key-scan-based mapping and navigation framework using a mobile robot in numerical ROS-Gazebo simulations and real physical hardware experiments (\refsec{sec.numerical_simulations}).
On a more conceptual level, we believe that our results demonstrate that systematically integrated mapping, planning, and control enables action for better perception and perception for better action.

The rest of the paper is organized as follows.
\refsec{sec.local_navigation} presents how to perform safe local navigation over a scan region. 
\refsec{sec.planning_over_graph_of_scans} describes how to perform feedback motion planning over a graph of scan regions. \refsec{sec.automated_deployment} presents autonomous key scan selection and exploration for active mapping. 
\refsec{sec.numerical_simulations} demonstrates numerical simulations and experimental validation. 
We conclude in \refsec{sec.conclusions} with a summary of our work and future research directions.

\section{Safe Local Navigation over a Scan Region}
\label{sec.local_navigation}

In this section, we briefly describe our robot motion and perception model and present two local scan navigation strategies that leverage the star-convexity of scan regions for simple yet effective safe local navigation.

\subsection{Robot Motion and Perception Model}

For ease of exposition,%
\footnote{Our results can be generalized to 3D robot navigation settings using omnidirectional 3D point-cloud sensing with an appropriate ordering relation of 3D points based on 3D triangular meshes. We defer this discussion to a future paper on sensor-based drone navigation, where 3D perception, planning, and control are more relevant.}
we consider a fully-actuated mobile robot moving in a 2D planar Euclidean space $\R^2$  with a circular robot body, centered at position $\pos \in \R^2$ with radius $\radius \geq 0$, whose equation of motion is given by 
\begin{align}\label{eq.equation_of_motion}
\dot{\pos} = \ctrl
\end{align}
where $\ctrl \in \R^2$ is the robot's velocity control input.
The robot is assumed to be operating in a bounded workspace $\workspace \subset \R^2$ with a priori unknown obstacles $\obstspace \subset \R^2$ and with the known localization%
\footnote{Given a collection of overlapping scans, the global localization of the robot can be estimated by scan matching \cite{segal_haehnel_thrun_RSS2009} and pose/factor-graph optimization \cite{dellaert_kaess_FTR2017}. In a follow-up paper, we plan to study the systematic integration of key-scan-based navigation with tightly coupled localization, mapping, planning, and control, which is outside the scope of this paper.} 
of its position $\pos$. 
Hence, the unknown free space of collision-free robot positions is given by 
\begin{align}\label{eq.free_space}
\freespace := \clist{\pos \in \workspace \, \big |\, \ball(\pos, \radius) \subseteq \workspace \setminus \obstspace}
\end{align}  
where $\ball(\pos, \radius):=\clist{\vect{y} \in \R^2 \, | \, \norm{\vect{y} - \pos} \leq \radius} $  denotes the closed Euclidean ball with center $\pos \in \R^2$ and radius $\radius \geq 0$ and $\norm{.}$ denotes the standard Euclidean norm.

Since the obstacles are unknown, the robot is assumed to be equipped with a 2D 360$^{\circ}$ point-cloud scanning sensor, with a maximum range of $\maxsenserange > \radius$ (greater than the body radius $\radius$), that
senses at a constant angular resolution and returns a set of $n$ counter-clockwise-ordered sensed (e.g., obstacle) points%
\footnote{For example, one can convert an ordered set of 2D laser range readings $\plist{r_0, \ldots, r_\numberofscan} \in \R ^{\numberofscan+1}$ taken at a constant angular regulation $\Delta \theta$ at angles $\plist{\theta_0, \ldots, \theta_\numberofscan}$  relative to the sensor center $\scancenter$ and  sensor orientation $\theta$  into the counter-clockwise ordered set of obstacle points $\plist{\scanpoint_0, \ldots, \scanpoint_{\numberofscan}}$ relative to the scan center $\scancenter$ in the global world coordinates as 
\begin{align*}
\scanpoint_i = \scancenter + r_i \begin{bmatrix}
\cos(\theta + \theta_i) \\
\sin(\theta + \theta_i)
\end{bmatrix}.
\end{align*}
} 
$\pointcloud = \plist{\scanpoint_0, \scanpoint_1, \ldots, \scanpoint_\numberofscan} \in \R^{{\numberofscan+1} \times 2}$   relative to the sensor center $\scancenter \in \R^2$ with identical first and last sensor readings, i.e., $\scanpoint_0 = \scanpoint_\numberofscan$, which is a simplifying assumption to ease the notation and handle circular ordering of scan points effectively.
For convenience, we assume that the sensor center and the robot center coincide, i.e., $\scancenter = \pos$.
Accordingly, we define the star-convex%
\footnote{A set $X \subseteq \R^n$ is said to be star-convex if and only if there exists a point $\vect{x} \in X$ such that $ [\vect{x}, \vect{y}] = \clist{\alpha \vect{x} + (1 - \alpha) \vect{y} \big| \alpha \in [0,1]} \subseteq X$ for all $\vect{y} \in X$, where $\vect{x}$ is referred to as a star center. 
Intuitively, a star-convex set $X$ is a collection of points that are visible (i.e., connected by collision-free straight line paths) to the star center, with visibility (i.e., obstacle-free space) constrained to the set $X$.
In this paper, we combine and exploit both of these perception and action related interpretations of star-convex regions for sensor-based safe navigation.
} 
polygon of the 2D scan points $ \scanpoints=\plist{\scanpoint_0, \ldots, \scanpoint_{\numberofscan}}$ relative to the scan center (a.k.a., star center) $\scancenter$ as
\begin{align}\label{eq.scan_poylgon}
\scanpoly(\scancenter, \scanpoints) := \bigcup\limits_{i = 1}^{\numberofscan} \conv(\scancenter, \scanpoint_{i-1}, \scanpoint_{i})
\end{align}
where $\conv$ denotes the convex hull operator, which, in our case, corresponds to a triangle for any given three vertex points.
Hence, the boundary, denoted by $\partial \scanpoly(\scancenter, \scanpoints)$, of the star-convex  polygon of scan $(\scancenter, \scanpoints)$ is given by
\begin{align}\label{eq.scan_polygon_boundary}
\partial \scanpoly(\scancenter, \pointcloud) := \bigcup_{i=1}^{n} [\scanpoint_{i-1}, \scanpoint_{i}]
\end{align}
where $\blist{\vect{p}, \vect{q}}:=\clist{\alpha \vect{p} + (1-\alpha)\vect{q} \,\big|\, \alpha \in [0,1]}$ denotes the line segment between points $\vect{p}$ and $\vect{q}$.

Although finite resolution scanning of obstacles might miss sharp, spiky obstacle corners, we consider this is less of a problem in human-centric environments and can be overcome via high angular resolution and high-definition artificial point clouds generated by fusing consecutive sensor data.
Accordingly, we assume that scan polygons truly captures the local obstacle-free space around the robot. 

\begin{assumption}\label{asm.star_polygon_safety}
\emph{(Obstacle-Free Scan Polygon)} For any collision-free scan center $\scancenter \in \freespace$ and the scan points $\scanpoints = \plist{\scanpoint_0, \ldots, \scanpoint_\numberofscan}$ sensed at point $\scancenter$, we assume that:
\begin{itemize}
\item The scan center has a positive clearance from obstacles with respect to the robot's body, i.e., $\min\limits_{\scanpoint_i \in \scanpoints} \norm{\scanpoint_i - \scancenter} > \radius$.
\item The scan points that are strictly within the maximum sensing range $\scanradius$ are actual obstacle points, i.e., \mbox{$\norm{\scanpoint_i - \scancenter} < \scanradius \Longrightarrow \scanpoint_i \in \obstspace$} for all $i = 0, \ldots, \numberofscan$. 
\item The interior\footnote{Note that $\interior(\scanpoly(\scancenter, \scanpoints)) = \scanpoly(\scancenter, \scanpoints) \setminus \partial \scanpoly(\scancenter, \scanpoints)$.} of the polygon of the scan $(\scancenter,\scanpoints)$ is free of obstacles $\obstspace$, i.e., $\interior(\scanpoly(\scancenter, \scanpoints))  \cap \obstspace = \varnothing$.
\end{itemize}
\end{assumption}

\noindent This assumption ensures that any obstacle point observed by the point-cloud sensor can only lie on the scan polygon boundary $\partial\scanpoly(\scancenter, \scanpoints)$. As a result, a local collision-free space around a scan center $\scancenter$ can be constructed by eroding the scan polygon $\scanpoly(\scancenter, \scanpoints)$ by the robot body shape as
\begin{align*}
\erode(\scanpoly(\scancenter, \pointcloud), \radius) \subseteq \freespace
\end{align*} 
where the erosion of a set $A$ by a radius of $\radius$ is defined as  
\begin{align*}
\erode(A, \radius) := \clist{\vect{y} \in A \Big| \ball(\vect{y}, \radius) \subseteq A}.    
\end{align*}
It is important to observe that erosion by itself is sufficient to guarantee safety, but it may not guarantee a simply-connected set or star-convex set  with respect to the scan center $\scancenter$, as illustrated in \reffig{fig.local_navigation_over_scan_polygon}. 
Hence, to take the advantage of star convexity in planning and control, we find it useful to define the collision-free star-convex safe polygon of a scan $(\scancenter, \scanpoints)$~as
{\small
\begin{align}\label{eq.safe_scan_polygon}
\!\!\safepoly{\scancenter, \!\scanpoints} \!&:=\! \clist{\vect{y} \!\in\! \R^{2} \Big|  \blist{\scancenter, \vect{y}} \!\subseteq\! \erode(\scanpoly(\scancenter, \!\scanpoints), \radius)\! }\!\!  
\\
\!\!\saferpoly{\scancenter, \!\scanpoints} \!&:=\! \clist{\vect{y} \!\in\! \R^{2} \Big|  \blist{\scancenter, \vect{y}} \!\subseteq\! \erode(\scanpoly(\scancenter, \!\scanpoints), \radius \!+ \! \epsilon)\! }\!\!\!  \nonumber 
\end{align}
}%
where the safer scan polygon defines the local planning domain for goal selection and provides a positive $\epsilon \!> \! 0$ margin\footnote{As a design choice, we find it more practical to define $\saferpoly{\scancenter, \scanpoints}$ as a strictly interior set of $\safepoly{\scancenter,\scanpoints}$, though mathematically one could equivalently use the interior of $\safepoly{\scancenter, \scanpoints}$. Thus, $\epsilon$ represents a very small value on the order of numerical computation precision.} for  continuous control over the safe polygon, and the distance to the boundary of a safer scan polygon is defined~as
{\small
\begin{align}\label{eq.distance_to_boundary}
\!\!\bndrydist_{(\scancenter, \scanpoints)} (\pos) \!:= \! \min\limits_{\vect{y} \in \partial\saferpoly{\scancenter, \scanpoints}\!)} \norm{\vect{x} - \vect{y}}.
\end{align}
}%
Note that the distance to the safe scan polygon boundary is a conservative measure of safety since the point-cloud sensor is assumed to have a finite maximum range of $\maxsenserange$, and any scan point $\scanpoint \!\in\! \pointcloud$ with $\norm{\scanpoint - \scancenter} \!= \!\maxsenserange$ might not hit an obstacle but merely reaches the maximum sensing range from the scan center $\scancenter$.
Hence, using \refasm{asm.star_polygon_safety}, we measure the distance  to the sensed obstacle points of a scan $(\scancenter, \scanpoints)$ as
{\small
\begin{align}\label{eq.distance_to_obstacles}
\obstdist_{(\scancenter, \scanpoints)}(\pos) := \min_{\substack{\scanpoint \, \in\, \scanpoints, \,  \norm{\scanpoint - \scancenter} <\, \maxsenserange}} \norm{\pos - \scanpoint}.
\end{align} 
}%

\subsection{Navigation Control over a Star-Convex Scan Region}
\label{subsec.navigation}

Thanks to its star convexity, the collision-free polygon $\safescanpoly(\scancenter, \scanpoints)$ of scan points $\pointcloud= \plist{\scanpoint_0, \ldots, \scanpoint_{\numberofscan}}$ around a collision-free scan center $\scancenter \in \freespace$ allows for a simple and safe navigation strategy to move between any two points, $\vect{x}$ and $\vect{y}$, within the safe polygon $\safescanpoly(\scancenter, \pointcloud)$ as follows:
\begin{enumerate}[i)]
\item If the straight line segment $[\vect{x}, \vect{y}]$ joining $\vect{x}$ and $\vect{y}$ is in $\safescanpoly(\scancenter, \pointcloud)$, then move directly between these points.
\item Otherwise, first move toward a shared visible point from both points $\vect{x}$ and $\vect{y}$ (e.g., the scan center $\scancenter$), until condition (i) holds, and then move to the destination.
\end{enumerate}
As a potential selection for a shared visible point, we consider the following two local navigation strategies: moving through the scan center or moving towards a project goal.

\smallskip

\subsubsection{Move-Through-Scan-Center Navigation Law}
\label{sec.move_through_star_center}
For any given safe robot and goal positions $\pos, \ctrlgoal \in \safescanpoly(\scancenter, \scanpoints)$, we define the move-through-scan-center navigation policy, denoted by $\ctrl_{\ctrlgoal, (\scancenter, \scanpoints)}(\pos)$ associated with the goal $\ctrlgoal$ and the scan $(\scancenter, \scanpoints)$, which specifies the robot's velocity command~as
{\small
\begin{align} \label{eq.local_navigation_policy}
\dot{\pos} = \ctrl_{\ctrlgoal, (\scancenter, \scanpoints)} (\pos)\! :=\! 
\begin{cases}
-\gain \plist{\pos \!-\! \ctrlgoal} & \!\text{if } \blist{\pos, \ctrlgoal} \!\subseteq\! \safescanpoly(\scancenter, \pointcloud) 
 \\
-\gain \plist{\pos \!- \!\scancenter} & \!\text{otherwise},
\end{cases}
\end{align}
}%
where $\gain >0$ is a fixed positive control gain.
By construction, as illustrated in \reffig{fig.local_navigation_over_scan_polygon}, the move-through-scan-center policy asymptotically and safely steers all robot positions $\pos\in \safepoly{\scancenter, \scanpoints}$ to any given goal $\ctrlgoal\in \saferpoly{\scancenter, \scanpoints}$.

\begin{proposition}\label{prop.move_through_scan_center_convergence}
\emph{(Convergence of Move-Through-Scan-Center Policy)}
Given any goal $\vect{y} \!\in\! \saferpoly{\scancenter, \scanpoints}$ in the safe polygon of a scan $(\scancenter, \scanpoints)$, under \refasm{asm.star_polygon_safety}, the move-through-scan-center navigation policy $\ctrl_{\vect{y}, (\scancenter, \scanpoints)}(\pos)$ asymptotically bring any robot position $\pos \!\in\! \safepoly{\scancenter,\scanpoints}$ to the goal $\ctrlgoal$ without collisions, while non-increasing the perimeter of the triangle $\conv(\pos, \scancenter, \ctrlgoal)$ that defines a local navigation cost~as 
\begin{align*}
\navcost_{(\scancenter, \scanpoints)}(\pos, \ctrlgoal) := \norm{\pos - \scancenter} + \norm{\scancenter - \ctrlgoal} + \norm{\pos - \ctrlgoal}.
\end{align*}
\end{proposition}
\begin{proof}
See \refapp{app.move_through_scan_center_convergence}.
\end{proof}

\subsubsection{Move-To-Projected-Scan-Goal Navigation Law} 
\label{sec.move_to_projected_goal}

Alternatively, instead of moving through the scan center, one might aim to move toward the closest point from the scan center to the goal that is visible from the robot's position.
Hence, for any given safe robot and goal positions $\pos, \ctrlgoal \in \safescanpoly(\scancenter, \pointcloud)$, we define the move-to-projected-scan-goal navigation policy, denoted by $\overline{\ctrl}_{\ctrlgoal, (\scancenter, \pointcloud)}(\pos)$ associated with goal $\ctrlgoal$ and scan $(\scancenter, \pointcloud)$, specifying the robot's velocity as
\begin{align}\label{eq.move_to_visible_projected_goal} 
\dot{\pos} &= \overline{\ctrl}_{\ctrlgoal, (\scancenter, \pointcloud)} (\pos) := -\gain \plist{\pos - \proj_{\pos, (\scancenter, \pointcloud)}(\ctrlgoal)}
\end{align}
where $\gain > 0$ is a constant positive control gain and the projected scan goal $\proj_{\pos, (\scancenter, \pointcloud)}(\ctrlgoal)$ is defined as the closest point to the goal $\ctrlgoal$ from the scan center $\scancenter$ that is visible from $\pos$ within the safe scan polygon $\safescanpoly(\scancenter, \scanpoints)$ as
\begin{align}\label{eq.visible_projected_goal}
\proj_{\pos, (\scancenter, \pointcloud)}(\ctrlgoal) := \argmin_{\substack{\overline{\ctrlgoal} \, \in \, \blist{\scancenter, \ctrlgoal}\\ \blist{\pos, \overline{\ctrlgoal}} \, \subseteq \, \safescanpoly(\scancenter, \scanpoints)}} \norm{\overline{\ctrlgoal} - \ctrlgoal}
\end{align}
which is Lipschitz continuous with respect to $\pos$ and $\ctrlgoal$.\footnote{The projected scan goal $\proj_{\pos, (\scancenter, \pointcloud)}(\ctrlgoal)$ is a Lipschitz continuous function of both the robot position $\pos$ and the goal $\ctrlgoal$, since $\clist{\overline{\ctrlgoal} , \in , \blist{\scancenter, \ctrlgoal} \big| \blist{\pos, \overline{\ctrlgoal}} , \subseteq , \safescanpoly(\scancenter, \scanpoints)}$ is a convex set, the metric projection onto a convex set is Lipschitz continuous \cite{webster_Convexity1995}, and a continuous selection of Lipschitz continuous functions is also Lipschitz \cite{liu_JCO1995}.}

\begin{figure}[t]
\centering
\begin{tabular}{@{}c@{\hspace{0.2mm}}c@{\hspace{0.2mm}}c@{\hspace{0.2mm}}c@{}}
\includegraphics[width=0.247\columnwidth]{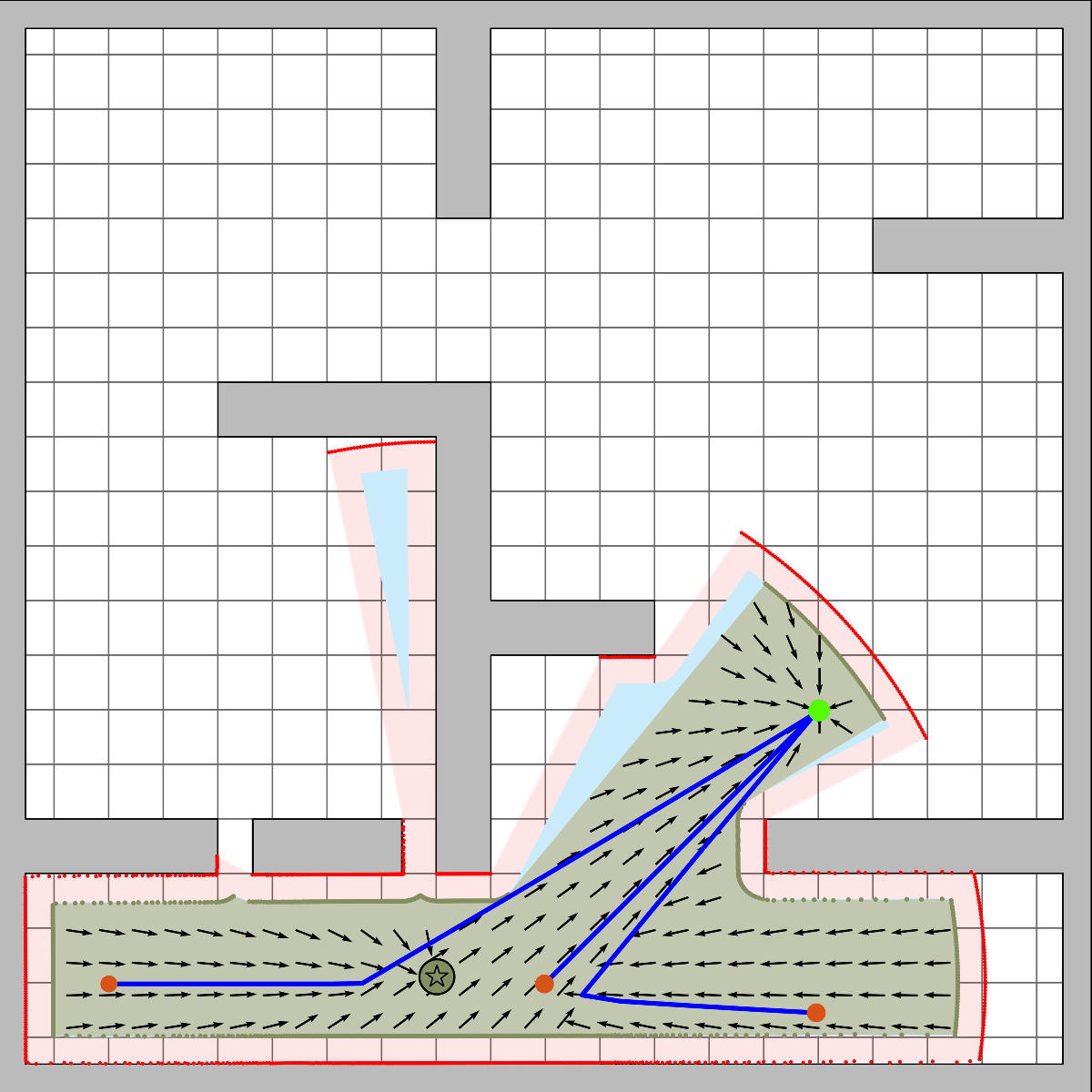}
&
\includegraphics[width=0.247\columnwidth]{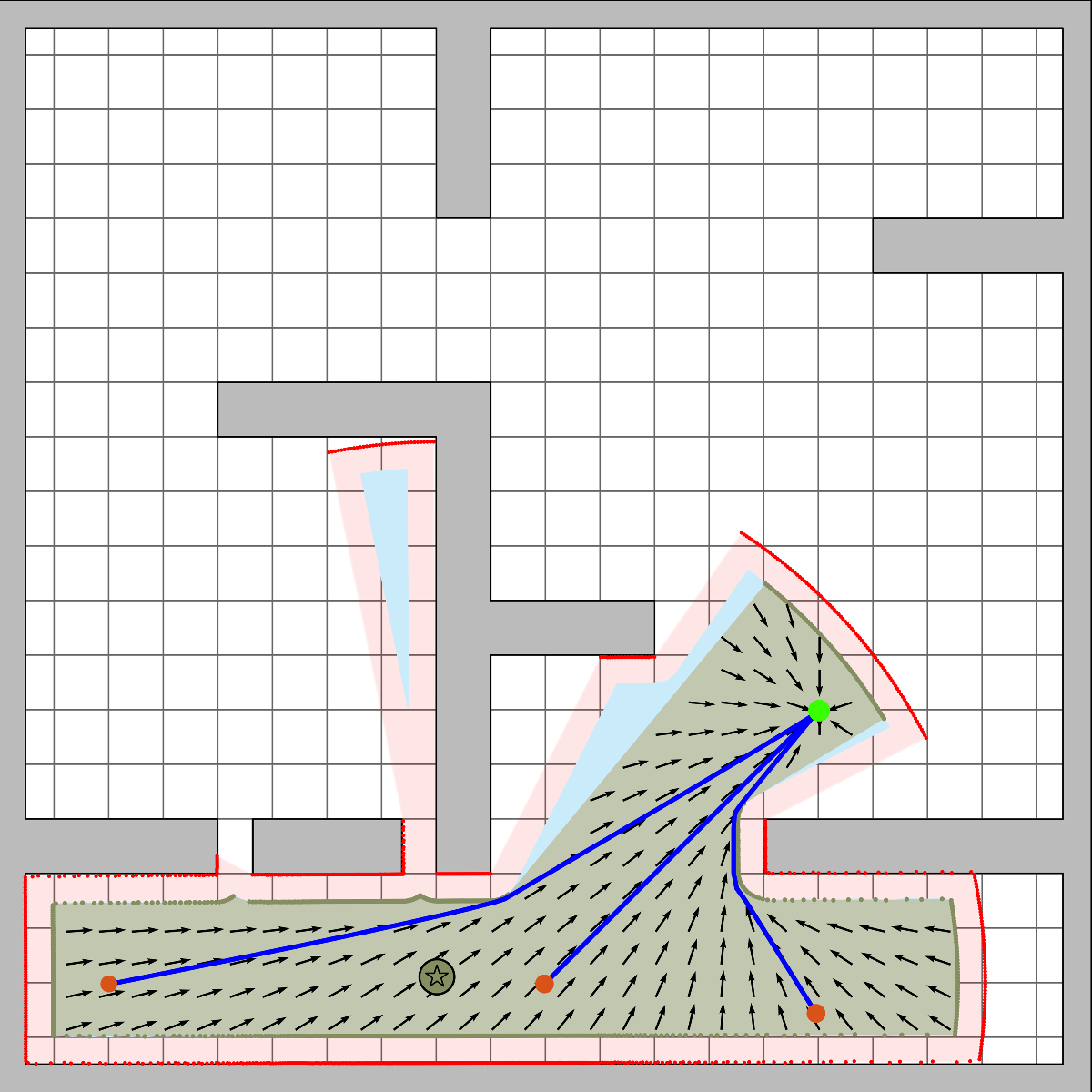}
&
\includegraphics[width=0.247\columnwidth]{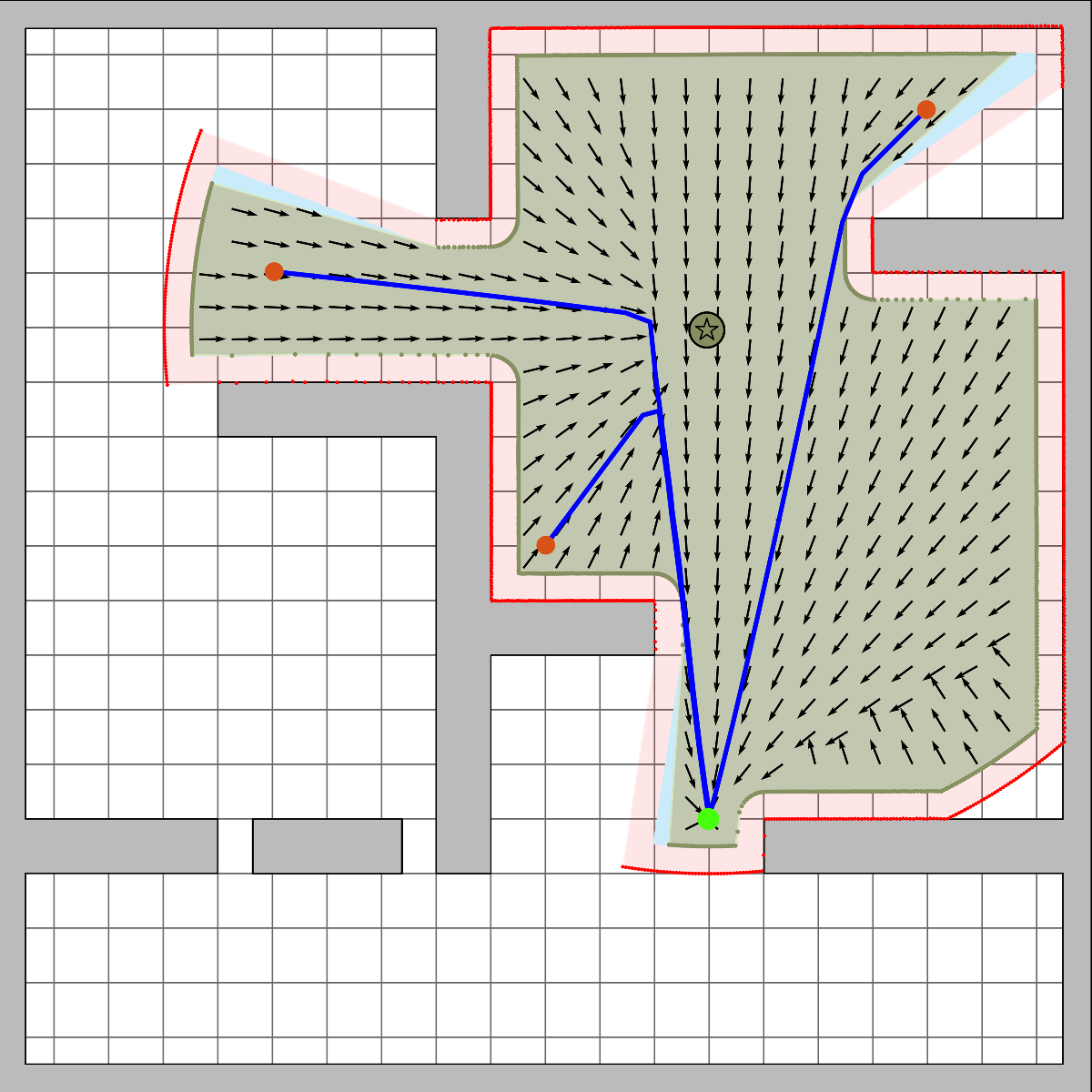}
&
\includegraphics[width=0.247\columnwidth]{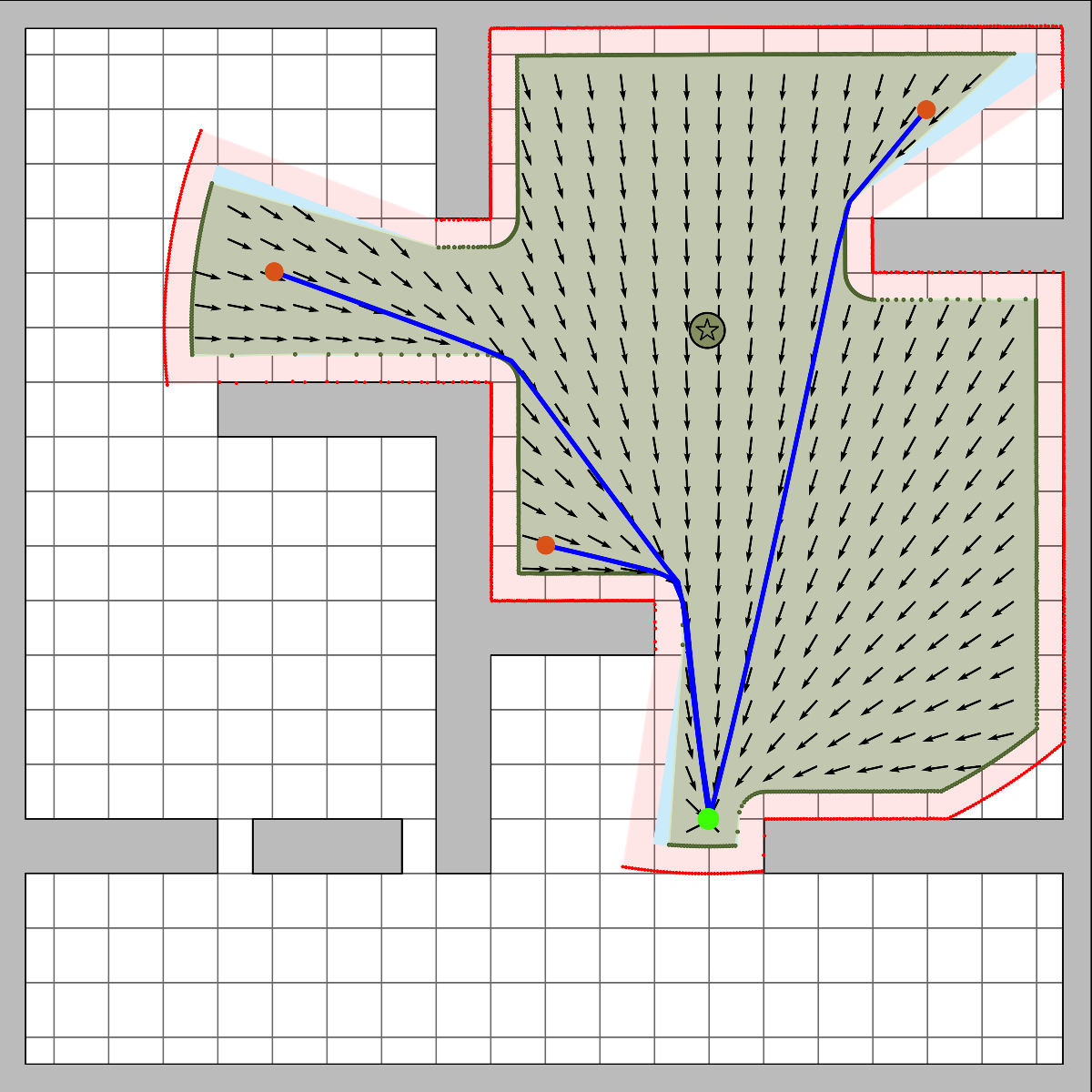}
\\[-2.5mm]
\tiny{(a)} & \tiny{(b)} & \tiny{(c)}& \tiny{(d)}
\end{tabular}
\vspace{-4.5mm}
\caption{Star-convex safe scan polygon (green patch) constructed by eroding the original scan polygon (red patch) and removing invisible points (cyan patch) from the scan center (star-circle). The vector field (black arrows) and example trajectories (blue lines) for (a,c) the move-through-scan-center policy and (b,d) the move-to-projected-scan-goal policy toward a given local goal position (green point).}
\label{fig.local_navigation_over_scan_polygon}
\vspace{-3mm}
\end{figure}

In addition to sharing the same non-increasing local navigation cost $\navcost_{\pos, (\scancenter, \scanpoints)}$ with the move-through-scan-center policy, the move-to-projected-scan-goal policy strictly decreases the distance to the projected scan goal. 

\begin{proposition}\label{prop.move_to_project_goal_convergence}
\emph{(Convergence of Move-To-Projected-Scan-Goal Policy)}
Given any goal $\ctrlgoal \in \saferpoly{\scancenter, \scanpoints}$ in the safe polygon of a scan $(\scancenter, \scanpoints)$, under \refasm{asm.star_polygon_safety}, the move-to-projected-scan-goal navigation policy $\overline{\ctrl}_{\ctrlgoal, (\scancenter, \scanpoints)}(\pos)$ in \refeq{eq.move_to_visible_projected_goal} asymptotically bring all robot position $\pos \in \safescanpoly(\scancenter, \scanpoints)$ to the goal $\ctrlgoal$  while avoiding collisions and decreasing the length of the piecewise straight path between $\pos$ and $\ctrlgoal$ joined through the visible projected goal $\proj_{\pos, (\scancenter, \scanpoints)}(\ctrlgoal)$ as a local navigation cost that is defined as\footnote{
Note that the perimeter $\norm{\pos - \proj_{\pos, (\scancenter, \scanpoints)}(\ctrlgoal)} + \norm{\proj_{\pos, (\scancenter, \scanpoints)}(\ctrlgoal) - \ctrlgoal} + \norm{\pos - \ctrlgoal}$ of the triangle $\conv(\pos,\proj_{\pos, (\scancenter, \scanpoints)}(\ctrlgoal), \ctrlgoal)$ or the perimeter $\norm{\pos - \scancenter}+\norm{\scancenter - \ctrlgoal} + \norm{\pos - \ctrlgoal}$ of the triangle $\conv(\pos, \scancenter, \ctrlgoal)$ might also be used as a local navigation of the move-to-projected-scan-goal policy since they are both non-increasing under the move-to-projected-scan-goal navigation policy.
}%
\begin{align*}\label{eq.move_to_visible_projected_goal_path_length}
\overline{\navcost}_{(\scancenter, \scanpoints)}(\pos, \ctrlgoal) := \norm{\pos \!- \proj_{\pos, (\scancenter, \scanpoints)}\!(\ctrlgoal)\!} + \norm{\proj_{\pos, (\scancenter, \scanpoints)}\!(\ctrlgoal)\! - \! \ctrlgoal}.
\end{align*}
\end{proposition}
\begin{proof}
See \refapp{app.move_to_project_goal_convergence}.
\end{proof}

\section{\!Motion Planning {\small over} Graphs {\small of} Scan Regions}
\label{sec.planning_over_graph_of_scans}

In this section, we describe how to construct a graph of star-convex scan polygons for global motion planning over the spatial cover (i.e., union) of these polygonal regions. 
Optimizing the visit sequence of these scan polygons enables systematic and effective sequential composition \cite{burridge_rizzi_koditschek_IJRR1999} of local navigation policies, ensuring safe and robust global navigation across their collectively covered domains.

\subsection{Motion Graph of Safe Scan Polygons}
\label{sec.motion_graph}

One can define various notions of a graph for a collection of spatial regions, such as the classical approach based on set intersection \cite{konolige_marder_marthi_ICRA2011}. 
To simplify the algorithmic design and implementation complexity in practice, we choose to construct a graph of star-convex scan polygons based on the reciprocal safe visibility of scan centers, \mbox{as illustrated in \reffig{fig.global_navigation_over_scan_polygons}.}

\begin{definition}\label{def.motion_graph}
\emph{(Motion Graph of Safe Scan Polygons)} The motion graph $\graph(\scanset):=(\vertexset,\edgeset)$ of an ordered set of scans $\scanset=\plist{(\scancenter_1, \scanpoints_1), \ldots, (\scancenter_m, \scanpoints_m)}$, with scan centers $\scancenter_1, \ldots, \scancenter_m$ and scan points $\scanpoints_1, \ldots, \scanpoints_m$, is an undirected graph where
\begin{itemize}

\item Vertices: Each vertex $i \in \vertexset = \clist{1, \ldots, m}$ corresponds to a pair $(\scancenter_i, \pointcloud_i)$ of scan center $\scancenter_i$ and scan points $\scanpoints_i$.

\item Edges: An edge $(i,j) \in \edgeset\subseteq \vertexset \times \vertexset$ associated with scans $(\scancenter_i, \pointcloud_i)$ and $ (\scancenter_j, \pointcloud_j)$ exists if and only if the scan centers can be safely visible to each other, i.e.,\footnote{Under \refasm{asm.star_polygon_safety}, the conditions in \refeq{eq.edge_connectivity} are equivalent. To minimize reliance on this assumption, we use both conditions to enhance robustness and reliability in practice, ensuring the undirected connectivity of the graph.}
\begin{subequations}\label{eq.edge_connectivity}
\begin{align}
\scancenter_i &\in \saferpoly{\scancenter_j, \scanpoints_j} 
\\
\scancenter_j & \in \saferpoly{\scancenter_i, \scanpoints_i}  
\end{align} 
\end{subequations}
where $\saferpoly{\scancenter, \scanpoints}$ is the safer scan polygon in \refeq{eq.safe_scan_polygon}.

\end{itemize} 
\end{definition}

\noindent It is important to observe that, by construction, the spatial embedding of the motion graph of safe scan polygons by connecting the adjacent scan centers via straight line segments yields a collision-free path in the free space $\freespace$ of the robot, as illustrated in \reffig{fig.global_navigation_over_scan_polygons}, i.e.,
\begin{align}
\bigcup_{\substack{(i,j) \in E\\ \graph(\scanset) = (\vertexset,\edgeset)}} \blist{\scancenter_i, \scancenter_j} \subseteq \freespace.
\end{align}
Hence, the motion graph of scan polygons can be considered a high-level topological roadmap \cite{lavalle_PlanningAlgorithms2006} for motion planning with a safe geometric embedding in the robot's free space.
For example, as a planning heuristic, one might assign to each edge $(i,j) \in \edgeset$ of the motion graph $\graph(\scanset)=(\vertexset, \edgeset)$ a weight equal to the Euclidean distance $\norm{\scancenter_i - \scancenter_j}$ between the centers of adjacent scans $(\scancenter_i, \pointcloud_i)$ and $(\scancenter_j, \pointcloud_j)$. 
Search-based optimal motion planning, such as A* or Dijkstra's algorithm, can then be performed over the motion graph of scan polygons to find a route that visits a set of local scan regions in an optimal order, heuristically minimizing the travel distance of a mobile robot based on scan-center distances.
Accordingly, we below present an integrated planning and control approach, based on the sequential composition \cite{burridge_rizzi_koditschek_IJRR1999} of local scan navigation policies in \refsec{sec.local_navigation}, for  safe and robust global robot navigation.

\subsection{\!\!Integrated Planning \!{\small\&} \!Control via Graphs of Scan Polygons\!}

In this part, we describe a feedback motion planning approach to generate a piecewise continuous velocity field in the robot's free space for safe global navigation over the union of safe scan polygons using optimal routes in their motion graph.
For a given ordered list of scans, denoted by $\scanset = \plist{(\scancenter_1, \scanpoints_1), \ldots, (\scancenter_m, \scanpoints_m)}$, we assume that: 

\begin{figure}[t]
\centering
\begin{tabular}{@{}c@{\hspace{0.5mm}}c@{\hspace{0.5mm}}c@{}}
\includegraphics[width=0.325\columnwidth]{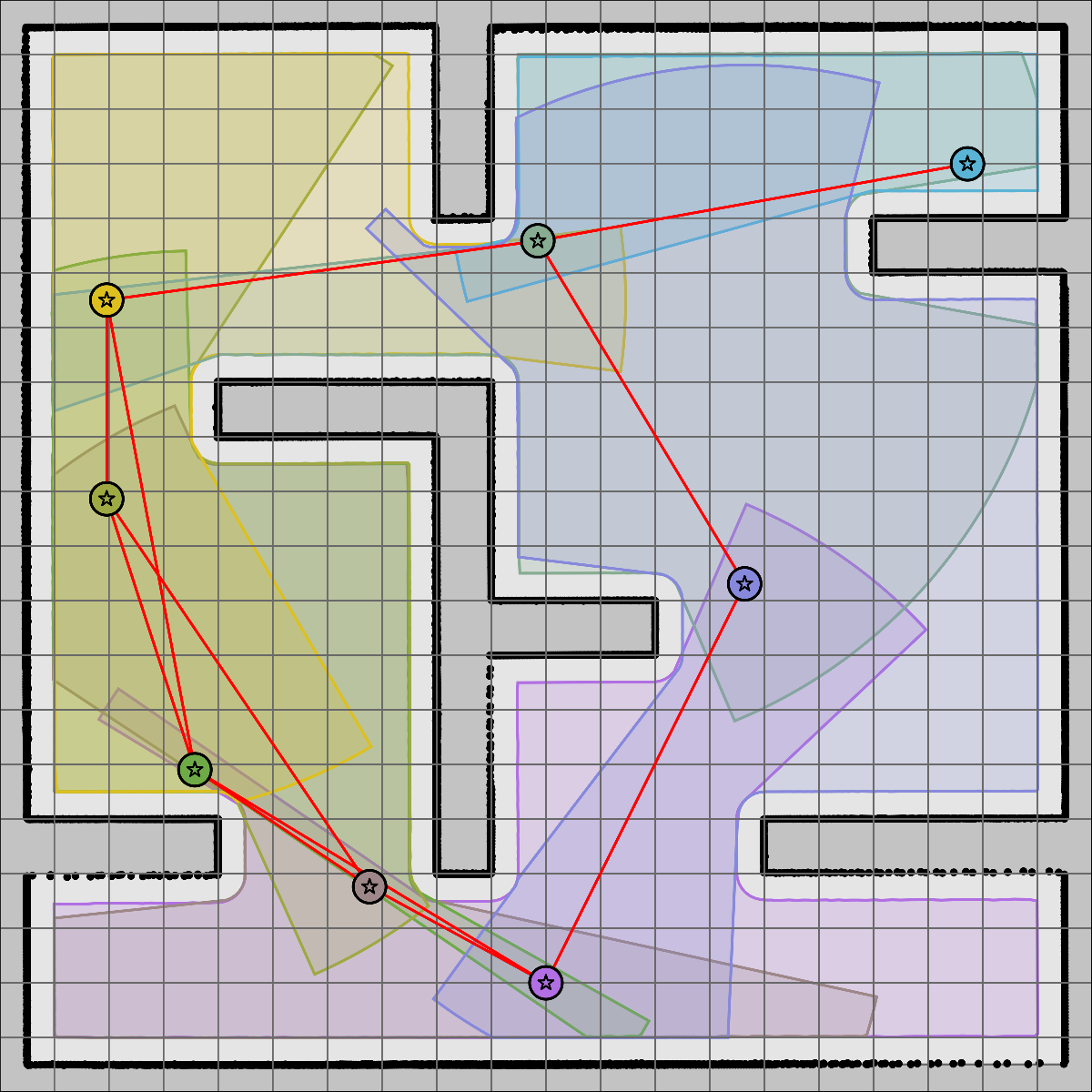}
&
\includegraphics[width=0.325\columnwidth]{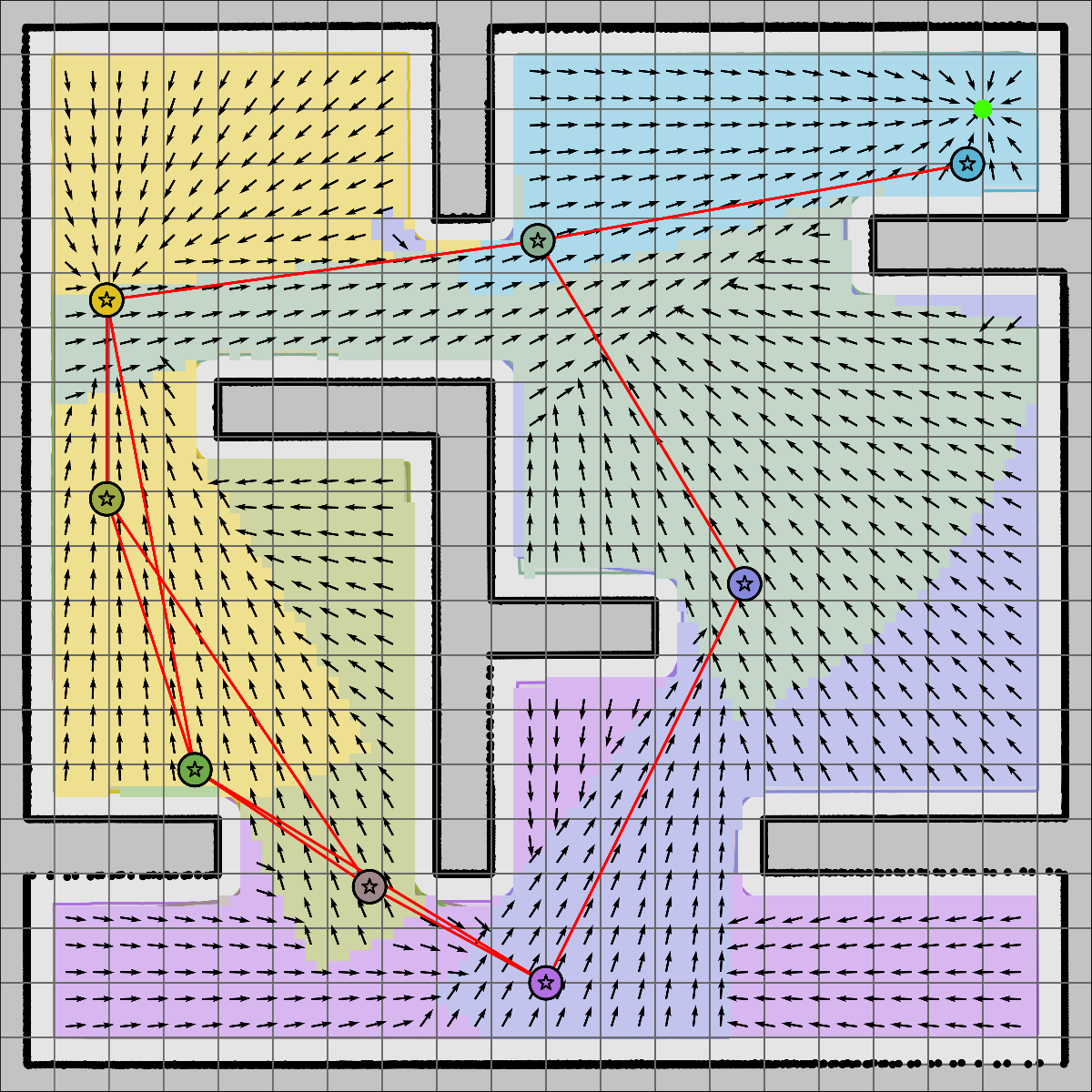}
&
\includegraphics[width=0.325\columnwidth]{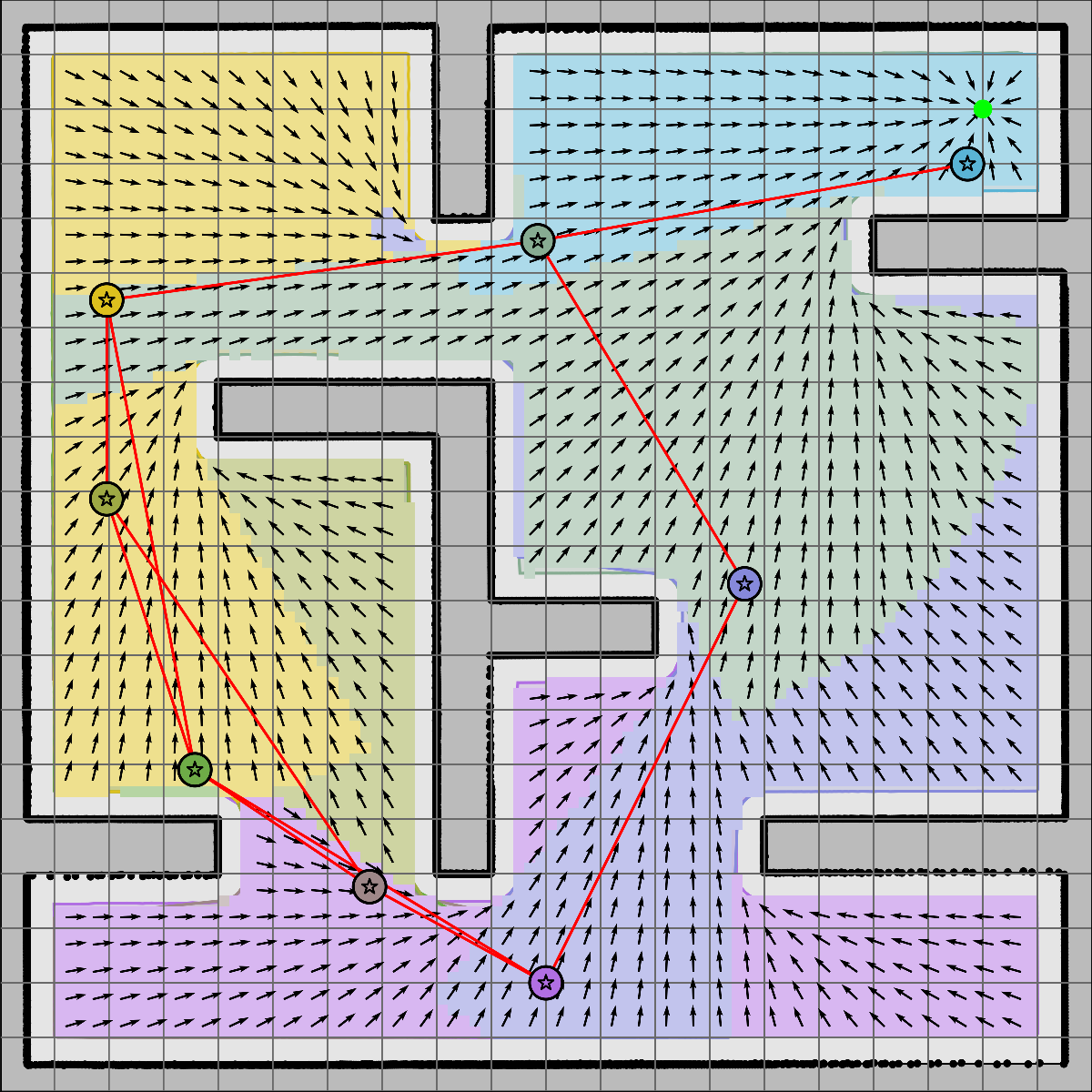}
\end{tabular}
\vspace{-1mm}
\caption{(left) Graph of star-convex scan regions (colored polygon patches with circular scan center icons). 
The sequential composition of (middle) the move-to-star-center vector field based on the distance-to-scan-center local cost and (right) the move-to-projected-scan-goal vector field  based on the distance-to-projected-scan-goal local cost, where colored regions highlight associated active scan regions.}
\label{fig.global_navigation_over_scan_polygons}
\vspace{-2mm}
\end{figure}

\indent $\bullet$ \emph{Local Cost Heuristic:} The local travel cost between any pair of points $\vect{x}, \vect{y} \! \in \! \safepoly{\scancenter, \scanpoints}$ within the safe polygon of a scan $(\scancenter, \scanpoints)$ can be heuristically measured by a positive function, denoted by $\localcost_{(\scancenter, \scanpoints)}(\vect{x}, \vect{y})$.
For example, one can use a constant travel cost (i.e., $\localcost_{(\scancenter, \scanpoints)}(\vect{x}, \vect{y}) \!=\! 1$) to describe a uniform regional cost, or use the distance to the scan center (i.e., $\localcost_{(\scancenter, \scanpoints)}(\vect{x}, \vect{y}) = \norm{\vect{x} \!-\! \scancenter} \!+\! \norm{\scancenter \!-\! \vect{y}}$) to capture a navigation behavior corresponding to the move-to-scan-center policy in \refsec{sec.move_through_star_center}, and see \reftab{tab.local_transition_cost}.\reffn{fn.LocalCost} 

\indent $\bullet$ \emph{Local Navigation Policy:} The fully-actuated velocity-controlled robot model in \refeq{eq.equation_of_motion} can be asymptotically brought from any start position $\vect{x} \! \in \! \safepoly{\scancenter, \scanpoints}$ within the positively invariant safe polygon of a scan $(\scancenter, \scanpoints)$ to any goal $\vect{y} \! \in \! \saferpoly{\scancenter, \scanpoints}$  using a local scan navigation policy $\ctrl_{(\scancenter, \scanpoints)}(\pos)$ associated with a non-increasing local navigation cost $\navcost_{(\scancenter, \scanpoints)}(\pos, \ctrlgoal)$. 
For example, one can use the move-to-scan-center or move-to-projected-scan-goal policy in \refsec{sec.local_navigation}, or any other navigation policy with similar properties.
Note that to align the planning heuristic with the control effort, one can select $\localcost_{\scancenter, \scanpoints}(\vect{x}, \vect{y}) = \navcost_{\scancenter, \scanpoints}(\vect{x}, \vect{y})$, which may be preferable in some cases but is not technically necessary, allowing for more complex planning strategies.

\begin{table}[b]
\vspace{-3mm}
\caption{Local Transition Costs for Graphs of Scan Regions}
\label{tab.local_transition_cost}
\centering
\vspace{-3mm}
\begin{tabular}{@{}c@{\hspace{1mm}}c@{}}
\hline
\hline 
\\[-2mm]

Cost Type & $\localcost_{(\scancenter, \scanpoints)}(\pos, \ctrlgoal)$
\\
\hline
\\[-2mm]
Uniform Constant Cost  & $1$
\\
Distance to Goal  & $\norm{\pos - \ctrlgoal}$
\\
Centroidal Distance & $\norm{\pos - \scancenter} + \norm{\scancenter - \ctrlgoal}$
\\
Projected Goal Distance \!\!\! & $\norm{\pos \!-\! \proj_{\pos, (\scancenter, \scanpoints)}\!(\ctrlgoal)} \!+\! \norm{\proj_{\pos, (\scancenter, \scanpoints)}\!(\ctrlgoal) \!-\! \ctrlgoal}$
\\
Centroidal Perimeter Cost &  $\norm{\pos - \ctrlgoal} + \norm{\pos - \scancenter} + \norm{\scancenter - \ctrlgoal}$
\\
Projected Perimeter Cost & $\norm{\pos \!-\! \ctrlgoal} \!+\! \norm{\pos \!-\! \proj_{\pos, (\scancenter, \scanpoints)}\!(\ctrlgoal)} \!+\! \norm{\proj_{\pos, (\scancenter, \scanpoints)}\!(\ctrlgoal) \!-\! \ctrlgoal}$ 
\\
Sym. Proj. Goal Distance  \!\!\!\!&  $\norm{\pos \!-\! \proj_{\pos, (\scancenter, \scanpoints)}\!(\ctrlgoal)} + \norm{\ctrlgoal \!-\! \proj_{\ctrlgoal, (\scancenter, \scanpoints)}\!(\pos)}$
\\
\hline
\\[-2.5mm]
\hline
\end{tabular}
\end{table}

Accordingly, in \refalg{alg.planning_over_star_convex_regions}, using a Dijkstra-like optimal graph search  over the motion graph $\graph(\scanset)$ (\refdef{def.motion_graph}) of the scan collection $\scanset = \plist{(\scancenter_1, \scanpoints_1), \ldots, (\scancenter_m ,\scanpoints_m)}$ with the local cost heuristic $\localcost_{(\scancenter, \scanpoints)}(\vect{x}, \vect{y})$, we determine the optimal travel cost $\policycost_{\goal, \scanset}(i)$ and the optimal local goal assignment $\policygoal_{\goal, \scanset}(i)$ for traversing through the safe polygon of a scan $(\scancenter_i, \scanpoints_i)$ towards a given global goal $\goal \in \bigcup_{i=1}^{m}\saferpoly{\scancenter_i, \scanpoints_i}$, which satisfies the following Bellman's optimality condition
{\small
\begin{align*}
\policycost_{\goal, \scanset}(i) & \leq  \policycost_{\goal, \scanset}(j) \\
& \hspace{10mm}+ \underbrace{\localcost_{(\scancenter_j, \scanpoints_j)}(\scancenter_i, \policygoal_{\goal,\scanset}(j))}_{\substack{\text{the local cost of moving from the scan center $\scancenter_i$} \\ \text{to the local goal of the scan $(\scancenter_j, \scanpoints_j)$} \\ \text{through safe scan polygon $\safepoly{\scancenter_j, \scanpoints_j}$}}}
\end{align*}
}%
for all $i = 1, \ldots, m$ and its neighbors $j \! \in \! \mathrm{neighbor}_{\graph(\scanset)}(i)$ of scan $(\scancenter_i, \scanpoints_i)$ in the motion graph $\graph(\scanset)$, where the inequality is strict for any scan $(\scancenter_i , \scanpoints_i)$ containing the global goal $\goal \in \saferpoly{\scancenter_i, \scanpoints}$ and is tight otherwise.

\begin{algorithm}[t]
\small
\caption{\small \mbox{Motion Planning over Graphs of Scan Polygons:} 
\\ \mbox{\hspace{17mm}} \mbox{Optimal Cost and Goal Assignment}}
\label{alg.planning_over_star_convex_regions}
\KwIn{%
\mbox{$\scanset = \plist{(\scancenter_1, \pointcloud_1),  \ldots, (\scancenter_m, \pointcloud_m)}$: A Collection of Scans}  \\ 
\hspace{8.5mm} $\goal \in \bigcup_{i=1}^{m} \saferpoly{\scancenter_i, \scanpoints_i}$: Global Goal Position \\ 
\mbox{\hspace{8.5mm} $\localcost_{(\scancenter, \scanpoints)}(\vect{x}, \vect{y})$: Local Cost over a Scan Polygon  }
}
\KwOut{%
\mbox{$\policycost_{\goal, \scanset}(i)$: Optimal Cost of a Scan Polygon}
\\
\mbox{\hspace{11mm} $\policygoal_{\goal, \scanset}(i)$: Local Goal of a Scan Polygon}
\vspace{-2mm}
\\
\hrulefill
}
$\graph(\scanset)=(\vertexset, \edgeset) \gets \mathrm{motiongraph}(\scanset)$ \\
$\mathrm{scanlist} \gets \varnothing$ \\
\For{$i \gets 1$ \KwTo $m$}{
$\policycost_{\goal, \scanset}(i) \gets \infty$ \\
$\policygoal_{\goal, \scanset}(i) \gets \scancenter_i$\\
\If{$\goal \in \saferpoly{\scancenter_i, \scanpoints_i}$}{
$\policycost_{\goal, \scanset}(i) \gets \localcost_{(\scancenter_i, \scanpoints_i)}(\scancenter_i, \goal)$  \\
$\policygoal_{\goal, \scanset}(i) \gets \goal$ \\
$\mathrm{scanlist} \gets \mathrm{scanlist} \cup \clist{i}$
}
}
\While{$\mathrm{scanlist} \neq \varnothing$}{
$i \gets \argmin\limits_{i \, \in \, \mathrm{scanlist}} \, \policycost_{\goal, \scanset}(i)$\\
$\mathrm{scanlist} \gets \mathrm{scanlist} \setminus \clist{i}$\\
\For{\emph{\textbf{all}} $j \in \mathrm{neighbor_{\graph(\scanset)}}(i)$}{
\mbox{$\mathrm{tempcost} \gets \localcost_{(\scancenter_i, \scanpoints_i)}(\scancenter_j, \policygoal_{\goal, \scanset}(i))$}\\
\If{\mbox{$\policycost_{\goal, \scanset}(j) \!> \policycost_{\goal,\scanset}(i) \!+\! \mathrm{tempcost}$}}{
\mbox{$\policycost_{\goal,\scanset}(j) \gets \policycost_{\goal,\scanset}(i) + \mathrm{tempcost}$}\\
$\policygoal_{\goal,\scanset}(j) \gets \scancenter_i$ \\
$\mathrm{scanlist} \gets \mathrm{scanlist} \cup \clist{j}$
}
}

}
\Return{$\policycost_{\goal, \scanset}, \policygoal_{\goal, \scanset}$}
\end{algorithm}

\addtocounter{footnote}{1}
\footnotetext{\label{fn.LocalCost}Note that a local cost heuristic might be asymmetric, e.g., the distance to the projected goal and the projected perimeter cost in \reftab{tab.local_transition_cost}.}

Safe scan polygons, by design, need to have overlaps in order to generate a connected motion graph \mbox{(\refdef{def.motion_graph})}. 
Hence, while navigating towards the global goal $\goal$, the robot's position might fall within more than one scan region.
To systematically and deterministically select a unique scan containing the robot position, we use the optimal travel cost $\policycost_{\goal, \scanset}$ of scans, to assign each robot position to a scan, which yields a non-overlapping, mutually exclusive, and exhaustive tessellation of the union of safe scan polygons into non-overlapping subregions (tiles), as seen in \reffig{fig.global_navigation_over_scan_polygons}.

\addtocounter{footnote}{1}
\footnotetext{\label{fn.zeno}To avoid infinitely many Zeno switchings between equally good scan policies in finite time, we assume that the active scan selection in \refeq{eq.active_scan} returns the smallest scan index among equally good and optimal scan policies.}

\begin{definition}\label{def.active_scan}
\emph{(Active Scan Polygon and Global Navigation Policy)}
For any set of scans $\scanset = \plist{(\scancenter_1, \pointcloud_1), \ldots, (\scancenter_m, \pointcloud_m)}$, the active scan index, denoted by $\activepolicy_{\goal, \scanset}(\pos)$, for a robot positioned at  $\pos \in \bigcup_{i=1}^{m} \safescanpoly(\scancenter_i, \pointcloud_i)$ moving towards a goal $\goal \in \bigcup_{i=1}^{m} \saferpoly{\scancenter_i, \scanpoints_i}$, is defined as the index of a scan containing the robot position within its safe scan polygon and ensuring the minimum total travel cost over the motion graph $\graph(\scanset)$ to the goal $\goal$ as\reffn{fn.zeno}~\reffn{fn.online_sequential_composition}
{\small
\begin{align}\label{eq.active_scan}
\activepolicy_{\goal, \scanset}(\pos)  :=  \!\! \argmin_{\substack{ i \, \in \clist{1, \ldots, m}\\ \pos \,\in\, \safescanpoly(\scancenter_i, \pointcloud_i)}}  \!\! \policycost_{\goal, \scanset}(i)  
\end{align}  
}%
based on the optimal cost $\policycost_{\goal,\scanset}(i)$ of the scan $(\scancenter_i, \scanpoints_i)$ as determined in \refalg{alg.planning_over_star_convex_regions}. 
Accordingly, we design a global feedback navigation policy to safely steer the robot position $\pos $ towards the goal  $\goal$ using an active convergent local scan navigation policy $\ctrl_{\vect{y}, (\scancenter, \pointcloud)}$ as
\begin{align}\label{eq.global_navigation_policy}
\dot{\pos} &= \ctrl_{\goal, \scanset} (\pos) =  \ctrl_{\ctrlgoal^*, (\scancenter_{i^*}, \scanpoints_{i^*})}(\pos)  
\end{align}
where  $i^* \!= \activepolicy_{\goal, \scanset}(\pos)$ denotes the active scan index and  $\ctrlgoal^* \!= \policygoal_{\goal, \scanset}(i^*)$ is the associated local scan goal.
\end{definition}

\addtocounter{footnote}{1}
\footnotetext{\label{fn.online_sequential_composition}Alternatively, one can select an active scan policy by using the non-increasing navigation cost $\navcost_{(\scancenter, \scanpoints)}(\pos, \ctrlgoal)$ of the local scan navigation policy $\ctrl_{\ctrlgoal, (\scancenter, \scanpoints)}(\pos)$  with the optimal travel cost $\policycost_{\goal,\scanset}(i)$ and the optimal local goal assignment $\policygoal_{\goal,\scanset}(i)$  of the scan $(\scancenter_i ,\scanpoints_i)$ as
\vspace{-4mm}
\begin{align*}
\activepolicy_{\goal, \scanset}(\pos) \! =\! \hspace{-4mm} \argmin_{\substack{i \in \clist{1, \ldots, m}\\ \pos \in \safescanpoly(\scancenter_i, \pointcloud_i)}} \hspace{-2mm} \begin{array}{@{}l@{}}
\\
\policycost_{\goal, \scanset}(i) \\
\quad  + \navcost_{(\scancenter_i, \scanpoints_i)}\!(\pos, \policygoal_{\goal, \scanset}\!(i)\!)
\end{array}
\end{align*}
which might exhibit Zeno-like many switching between local navigation policies in a sliding-mode control fashion around the policy boundaries, unless there is a strong alignment between planning, control and robot dynamics.
For example, the move-to-projected-scan-goal policy $\overline{\ctrl}_{\ctrlgoal, (\scancenter, \scanpoints)}$ in \refeq{eq.move_to_visible_projected_goal} for the fully-actuated kinematic robot model in \refeq{eq.equation_of_motion} with $\localcost_{(\scancenter, \scanpoints)}(\pos, \ctrlgoal) = \overline{\navcost}_{(\scancenter, \scanpoints)}(\pos, \ctrlgoal) = \norm{\pos \!-\! \proj_{\pos, (\scancenter, \scanpoints)}(\ctrlgoal)} \!+\! \norm{\proj_{\pos, (\scancenter, \scanpoints)}(\ctrlgoal) \!-\! \ctrlgoal}$ ensure a finite number of switching between local policies.  
}

By construction, the global convergence and safety of the sequential composition of local navigation policies is inherited from the safe convergence of individual policies~\cite{burridge_rizzi_koditschek_IJRR1999}.

\begin{theorem}\label{thm.sequential_composition_convergence}
\emph{(Global Convergence of Sequential Composition of Local Scan Navigation Policies)} Given a set of scans $\scanset = \plist{(\scancenter_1, \pointcloud_1), \ldots, (\scancenter_m, \pointcloud_m)}$, if their motion graph $\graph(\scanset)$ is connected, under \refasm{asm.star_polygon_safety}, the global feedback navigation policy  $\ctrl_{\goal,\scanset}(\pos)$ in \refeq{eq.global_navigation_policy} asymptotically brings all safe robot position $\pos \in \bigcup_{i=1}^{m} \safescanpoly(\scancenter_i, \pointcloud_i)$ in its positively invariant domain $\bigcup_{i=1}^{m} \safescanpoly(\scancenter_i, \pointcloud_i)$ to any goal position $\goal \! \in\! \bigcup_{i=1}^{m} \saferpoly{\scancenter_i, \pointcloud_i}$ without collisions along the way.
\end{theorem}
\begin{proof}
See \refapp{app.sequential_composition_convergence}.
\end{proof}

\section{\!\!Autonomous Exploration of Key Scan Regions}
\label{sec.automated_deployment}

In this section, we describe how to perform key scan selection based on frontier and bridging scan criteria and then present an autonomous exploration strategy that uses frontier and bridging scans to incrementally build a motion graph of key scan regions for global mapping and navigation.

\subsection{Key Scan Selection for Automated Deployment}

Inspired by the classical frontier-based exploration for active mapping \cite{yamauchi_CIRA1997}, we consider a key scan selection criterion to expand the collectively covered region of scans by exploring new, unknown areas via frontier scans as follows.

\begin{definition}\label{def.frontier_scan}
\emph{\!(Frontier Scan)}
A scan $(\scancenter, \scanpoints)$ is said to be a \emph{frontier scan} with respect to a collection of scans $\scanset \!=\! \plist{\!(\scancenter_1, \pointcloud_1), \ldots, (\scancenter_m, \pointcloud_m)\!}$ if and only if its scan center distance to the boundary of the union of safer scan polygons $\bigcup_{i=1}^{m} \saferpoly{\scancenter_i, \scanpoints_i}$ is less than a critical threshold $\varepsilon > 0$, and its center distance to obstacles is greater than $\delta > \varepsilon$, i.e.,
{\small
\begin{align*}
\max_{\substack{i = 1, \ldots, m \\ \scancenter \, \in \,\saferpoly{\scancenter_i, \scanpoints_i}}} \!\bndrydist_{(\scancenter_i, \pointcloud_i)}\!(\scancenter) \leq \varepsilon 
\\
\min_{\substack{i = 1, \ldots, m \\ \scancenter \, \in \,\saferpoly{\scancenter_i, \scanpoints_i}}}\obstdist_{(\scancenter_i, \scanpoints_i)}\!(\scancenter) \geq \delta 
\end{align*} 
}%
where $\bndrydist_{(\scancenter, \scanpoints)}(\vect{x})$ returns the distance to the boundary of a scan polygon as in \refeq{eq.distance_to_boundary} and $\obstdist_{(\scancenter, \scanpoints)}(\vect{x})$ returns the distance to the scanned obstacle points as in \refeq{eq.distance_to_obstacles}.
\end{definition}

A useful observation for effective key scan selection is the redundancy property of multiple scans for safety verification and identifying missing (topological) connections.
\begin{lemma}\label{lem.CircularSafety}
\emph{\!\!(Safe Convex Hull of Scan Centers)}~Under \refasm{asm.star_polygon_safety}, any triple of scans \mbox{$(\scancenter_i, \scanpoints_i)$, $(\scancenter_j, \scanpoints_j)$, $(\scancenter_k, \scanpoints_k)$} with $\max \plist{\norm{\scancenter_i - \scancenter_j},\norm{\scancenter_i - \scancenter_k}, \norm{\scancenter_j - \scancenter_k}} \leq \maxsenserange - \radius$ can be interchangeably used to check the safety of the convex hull $\conv(\scancenter_i, \scancenter_j, \scancenter_k)$ of safe scan centers $\scancenter_i, \scancenter_j, \scancenter_k \in \freespace$ as
\begin{subequations}
\begin{align*}
\conv(\scancenter_i, \scancenter_j, \scancenter_k) \subseteq \freespace \Longleftrightarrow [\scancenter_i, \scancenter_j] \subseteq \safescanpoly(\scancenter_{k}, \scanpoints_k)
\\
\Longleftrightarrow [\scancenter_i, \scancenter_k] \subseteq \safescanpoly(\scancenter_{j}, \scanpoints_j)
\\
\Longleftrightarrow [\scancenter_j, \scancenter_k] \subseteq \safescanpoly(\scancenter_{i}, \scanpoints_i). 
\end{align*} 
\end{subequations}
\end{lemma}
\begin{proof}
See \refapp{app.CircularSafety}.
\end{proof}

Hence, as a loop closing heuristic \cite{stachniss_hahnel_burgard_IROS2004}, we assess whether a new scan introduces novel (topological) connectivity into the motion graph of a collection of scans by considering the restricted local connectivity of the original motion graph from the perspective of the new scan as~follows.

\begin{definition}\label{def.scan_contrained_motion_graph}
\emph{(Scan-Constrained Motion Graph)}
The constrained subgraph $\subgraph_{(\scancenter, \scanpoints)}(\scanset) \!=\! (\subvertexset_{\!(\scancenter, \scanpoints)}, \subedgeset_{(\scancenter, \scanpoints)})$ of the motion graph $\graph(\scanset) \! = \!(\vertexset, \edgeset)$ of a collection of scans $\scanset=\plist{\!(\scancenter_1,\scanpoints_1), \ldots, (\scancenter_m, \scanpoints_m)\!}$ to the safer polygon of a scan $(\scancenter, \scanpoints)$ is defined by its constrained vertex and edge sets as:

\indent $\bullet$  Constrained Vertices: A vertex $i \!\in\! \vertexset\!=\!\clist{1, \ldots, m}$ is a constrained vertex in $\subvertexset_{\!(\scancenter, \scanpoints)}$ if and only if the centers of scans $(\scancenter_i, \scanpoints_i)$ and $(\scancenter, \scanpoints)$ are within each other's safer polygon,~i.e.,
\begin{align*}
\subvertexset_{(\scancenter, \scanpoints)} \!=\! \clist{i \in \vertexset  \Big|  \scancenter_i \!\in\! \saferpoly{\scancenter, \scanpoints}, \scancenter \!\in\! \saferpoly{\scancenter_i, \scanpoints_i}\!}.
\end{align*}

\indent $\bullet$ Constrained Edges: Any edge $(i,j) \in \subedgeset_{(\scancenter, \scanpoints)}$ between vertices $i, j \in \subvertexset_{(\scancenter, \scanpoints)}$ exists if and only if $(i,j) \in \edgeset$ and the line segment joining scan centers $\scancenter_i$ and $\scancenter_j$ is contained in the safer polygon of the scan $(\scancenter, \scanpoints)$ (and the equivalent conditions in \reflem{lem.CircularSafety}), i.e.,\footnote{Here, the equivalent relations from \reflem{lem.CircularSafety} are used and are only needed to increase robustness against \refasm{asm.star_polygon_safety}.} 
\begin{align*}
\blist{\scancenter_i , \scancenter_j} &\subseteq \saferpoly{\scancenter, \scanpoints},
\\
\blist{\scancenter , \scancenter_i} &\subseteq \saferpoly{\scancenter_j, \scanpoints},
\\
\blist{\scancenter , \scancenter_j} &\subseteq \saferpoly{\scancenter_i, \scanpoints}.
\end{align*}
\end{definition}

As expected, for any $i = 1, \ldots, m$, the constrained motion graph $\subgraph_{(\scancenter_i, \scanpoints_i)}(\scanset)$ to the existing scan $(\scancenter_i, \scanpoints_i)$ is always a connected subgraph of $\graph(\scanset)$ with vertices $\subvertexset = \clist{i} \cup \clist{j \in \vertexset \mid (i, j) \in \edgeset}$ and edge set $\subedgeset \supseteq \clist{(i, j) \in \edgeset \mid j \in \vertexset}$, even if the motion graph $\graph(\scanset)$ might not be connected.
Therefore, if a constrained motion graph $\subgraph_{(\scancenter, \scanpoints)}(\scanset)$ to a new scan $(\scancenter, \scanpoints)$ is unconnected, we say that the inclusion of the new scan improves the connectivity of the original motion graph.
Because a star-convex safer scan polygon $\saferpoly{\scancenter, \scanpoints}$ is simply connected and topologically equivalent to a point, the unconnectedness of the Euclidean embedding of the constrained motion graph $\bigcup_{(i,j) \in \subedgeset_{(\scancenter, \scanpoints)}} \blist{\scancenter_i, \scancenter_j} = \bigcup_{(i,j) \in \edgeset} \blist{\scancenter_i, \scancenter_j} \cap \saferpoly{\scancenter, \scanpoints}$ implies missing connectivity in the motion graph.
In particular, an unconnected constrained motion graph $\subgraph_{(\scancenter, \scanpoints)}(\scanset)$ of a connected motion graph $\graph(\scanset)$ implies that the new scan $(\scancenter, \scanpoints)$ captures novel (topological) connections, as illustrated in \reffig{fig.bridging_scan}.
Accordingly, we refer to such novel connection scans as loop-closing bridging scans.

\begin{definition}\label{def.bridging_scan}
\emph{(Bridging Scan)}
A scan $(\scancenter,\scanpoints)$ is said to be a \emph{bridging scan} for a collection of scans $\scanset = \plist{(\scancenter_1, \pointcloud_1) \ldots, (\scancenter_m, \pointcloud_m)}$ if and only if the constrained motion graph $\subgraph_{(\scancenter, \scanpoints)}(\scanset)$ is unconnected.\footnote{One can check the connectivity of an undirected, unweighted graph by examining the positivity of either the second smallest eigenvalue of the graph Laplacian or the elements of the reachability matrix of the graph.}
\end{definition}

Determining whether a previously unvisited position corresponds to a frontier or bridging scan without knowing the actual scan readings at that position is critical for exploration in unknown environments. 
By \refdef{def.frontier_scan}, a frontier scan position can be detected without knowing the scan measurements at that position.
Similar to Definitions \ref{def.scan_contrained_motion_graph} and \ref{def.bridging_scan}, one can detect candidate bridging scan positions without knowledge of the actual scan measurements at these positions by leveraging the safety equivalence of scans in \reflem{lem.CircularSafety}, i.e., if the scan points $\scanpoints$ at a scan center $\scancenter$ are unavailable, one can equivalently check if $\blist{\scancenter_i, \scancenter_j} \in \saferpoly{\scancenter, \scanpoints}$ using $\blist{\scancenter, \scancenter_i} \in \saferpoly{\scancenter_j, \scanpoints}$ and $\blist{\scancenter, \scancenter_j} \in \saferpoly{\scancenter_i, \scanpoints}$.

\begin{figure}
\centering
\begin{tabular}{@{}c@{\hspace{0.5mm}}c@{\hspace{0.5mm}}c@{}}
\includegraphics[width=0.33\linewidth]{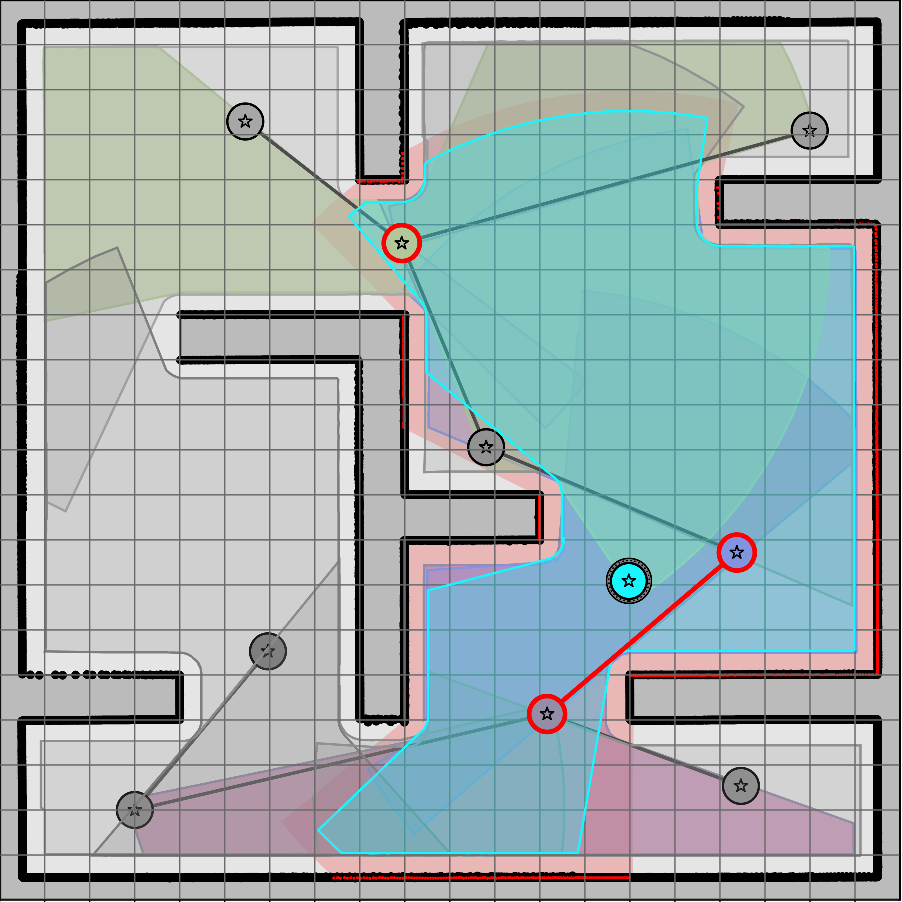}
&
\includegraphics[width=0.33\linewidth]{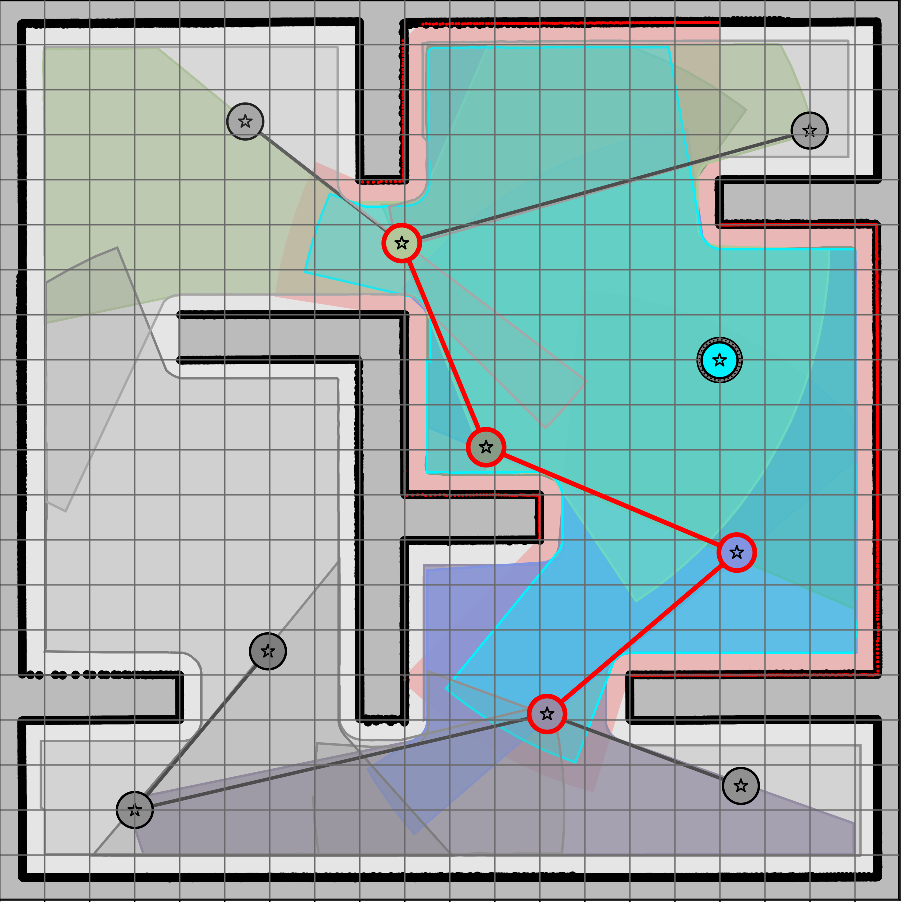}
&
\includegraphics[width=0.33\linewidth]{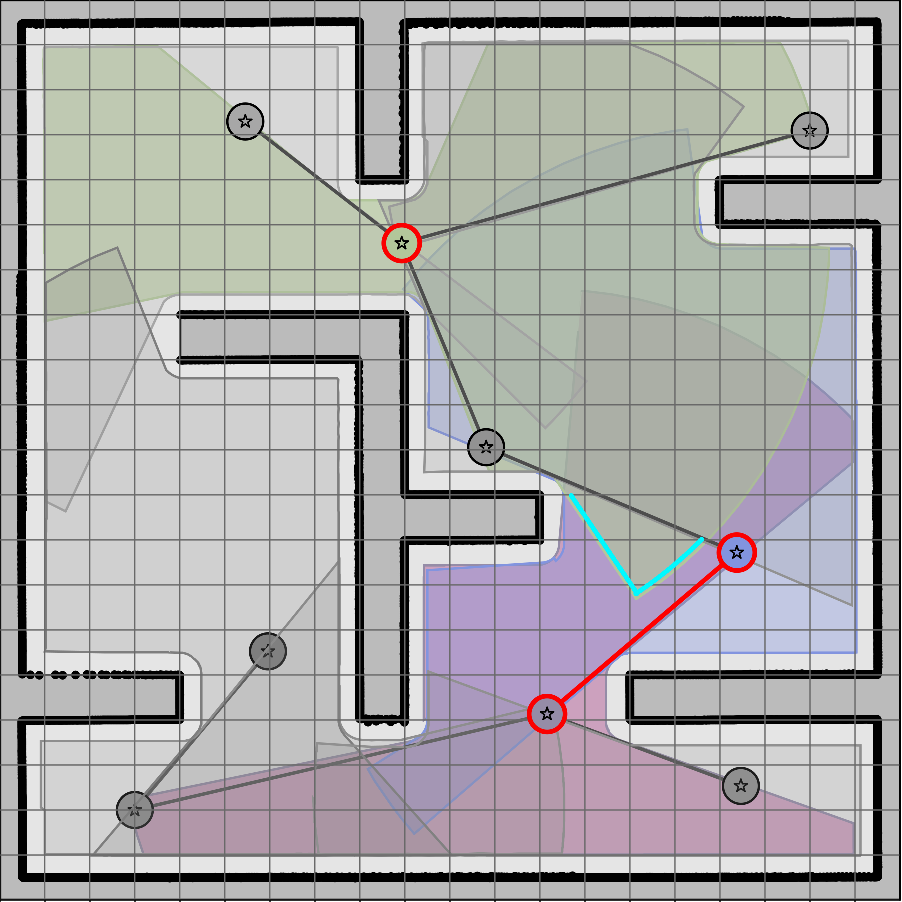}
\end{tabular}
\vspace{-4mm}
\caption{\!(left) A bridging scan (cyan) where the constrained motion graph (red) is unconnected. (middle) A non-bridging scan (cyan) where the constrained motion graph (red) is connected. (right) A set of bridging scan positions (cyan points) on the boundary of an existing scan where the constrained motion graph (red) onto these points is unconnected.}
\label{fig.bridging_scan}
\vspace{-4mm}
\end{figure}

\begin{definition}\label{def.position_constrained_motion_graph}
\emph{\!\!(Position-Constrained Motion Graph)}
The position-constrained subgraph $\subgraph_{\scancenter}(\scanset)\! =\! (\subvertexset_{\!\scancenter}, \subedgeset_{\scancenter})$ of the motion graph $\graph(\scanset)=(\vertexset, \edgeset)$ of a set of scans $\scanset = \plist{(\scancenter_1, \scanpoints_1), \ldots, (\scancenter_m, \scanpoints_m)\!}$ with respect to a scan observation position $\scancenter$ is defined by its constrained vertices and edges as:

\indent $\bullet$ Constrained Vertices: Each vertex $i \in \subvertexset_{\!\scancenter}$ is a vertex in $\vertexset \!=\! \clist{1, \ldots, m}$ associated with scan $(\scancenter_i, \scanpoints_i)$ whose safer polygon contains the scan position $\scancenter$, i.e.,
\begin{align*}
\subvertexset_{\scancenter} = \clist{i \in \vertexset \, \big | \,  \scancenter \in \saferpoly{\scancenter_i, \scanpoints_i}}.
\end{align*} 

\indent $\bullet$ Constrained Edges: An edge $(i,j) \!\in\! \subedgeset_{\scancenter}$ between $i,j \!\in\! \subvertexset_{\!\scancenter}$ exists if and only if $(i, j) \!\in \!\edgeset$ and the scan position $\scancenter$ is in the safer polygons of scans $(\scancenter_i, \scanpoints_i)$ and $(\scancenter_j, \scanpoints_j)$,~i.e., 
\begin{align*}
\blist{\scancenter, \scancenter_i} \subseteq \saferpoly{\scancenter_j, \scanpoints_j} \text{ and } 
\blist{\scancenter, \scancenter_j} \subseteq \saferpoly{\scancenter_i, \scanpoints_i}.
\end{align*}

\end{definition}

\begin{definition}\label{def.bridging_scan_position}
\emph{(Bridging Scan Position)}
A scan position $\scancenter \!\in\! \bigcup_{i=1}^{m}\saferpoly{\scancenter_i, \scanpoints_i}$ is said to be a \emph{bridging scan position} for a set of scans $\scanset \!\!=\!\! \plist{\!(\scancenter_1, \! \scanpoints_1), \ldots, (\scancenter_m,\! \scanpoints_m)\!}$ if and only if the position-constrained motion subgraph $\subgraph_{\scancenter}(\scanset)$ is unconnected. 
\end{definition}

As with bridging scans, a bridging scan position constrains the Euclidean embedding of the motion graph over $\bigcup_{i \in \subvertexset_{\scancenter}} \saferpoly{\scancenter_i, \scanpoints_i}$, which is a simply connected space and topologically equivalent to a point. 
If  the constrained motion graph embedding $\bigcup_{(i,j) \in \subedgeset_{\scancenter}} \blist{\scancenter_i, \scancenter_j}= \bigcup_{(i,j) \in \edgeset} \blist{\scancenter_i, \scancenter_j} \cap \bigcup_{i \in \subvertexset_{\scancenter}} \saferpoly{\scancenter_i, \scanpoints_i}$ is unconnected, this indicates a missing (topological) connection in the motion graph, see \reffig{fig.bridging_scan}, which can be resolved by collecting a new scan at~$\scancenter$.

\subsection{Autonomous Exploration via Frontier \& Bridging Scans}

One can use frontier and bridging scans in several ways to incrementally collect new key scans and deploy new local navigation policies over unknown environments for  mapping and navigation. 
Below, we present an autonomous exploration strategy that prioritizes frontier exploration over bridging exploration, completing area coverage
first and then focusing on loop closing, as illustrated in \reffig{fig.autonomous_exploration_frontier_bridging_scans}.

To efficiently detect potential frontier and bridging scan positions, we restrict the search for scan center candidates to the boundary regions of the safe polygons of existing scans as shown in \reffig{fig.autonomous_exploration_frontier_bridging_scans}.
These candidates are classified into bridging and frontier scan positions and then grouped into clusters based on their connectedness. 
Within each cluster of bridging and frontier scan positions, we determine the cluster midpoint that minimizes the Euclidean distance to all other points in the cluster. 
Subsequently, we calculate the shortest path distance on the motion graph of existing scans to these cluster centers, giving priority to frontier scans over bridging scans, and select the closest one as the potential scan observation point. 
Using the local navigation policies of existing scans, we autonomously drive the robot to the selected observation point to collect a new scan. 
The robot repeats this procedure autonomously until no more bridging or frontier scan positions remain, which implies area exploration and loop closing are complete.
Here is the summary of autonomous exploration steps:

\indent $\bullet$ Frontier Exploration: Determine the frontier scan positions on the boundary of the existing safe scan polygons. 
If a frontier exists, navigate to the closest midpoint of the frontier clusters based on the shortest path on the motion graph; otherwise, proceed to the next step of bridge exploration.

\indent $\bullet$ Bridging Exploration: Determine the bridging scan positions on the boundary of the existing safe scan polygons. If a bridging scan position is found, navigate to the closest midpoint of the bridging clusters, prioritizing connectivity improvement in the motion graph; otherwise, proceed to the next step to assess the progress of the exploration.

\indent $\bullet$ Termination: If there are no frontier or bridging scan positions, exploration for mapping and navigation is completed; otherwise, return to the first step of frontier exploration.

In \reffig{fig.autonomous_exploration_frontier_bridging_scans}, we present an example of autonomous exploration steps for integrated mapping and navigation, where local scan navigation policies are deployed with each key scan selection and directly used for global navigation.
One important observation is that while frontier-based exploration accurately creates a metric map of the environment, the resulting global navigation field can be improved due to missing topological connections and shortcuts in the motion graph. 
Frontier-based exploration tends to produce a tree-like motion graph.
In contrast, the combined frontier- and bridging-based exploration captures both metric and topological connectivity of the environment in the motion graph of scans, leading to a more effective global navigation planner.

\begin{figure}[t]
\centering
\begin{tabular}{@{}c@{\hspace{1.0mm}}c@{}}
\begin{tabular}{@{}c@{\hspace{1.0mm}}c@{}}
\includegraphics[width=0.235\linewidth]{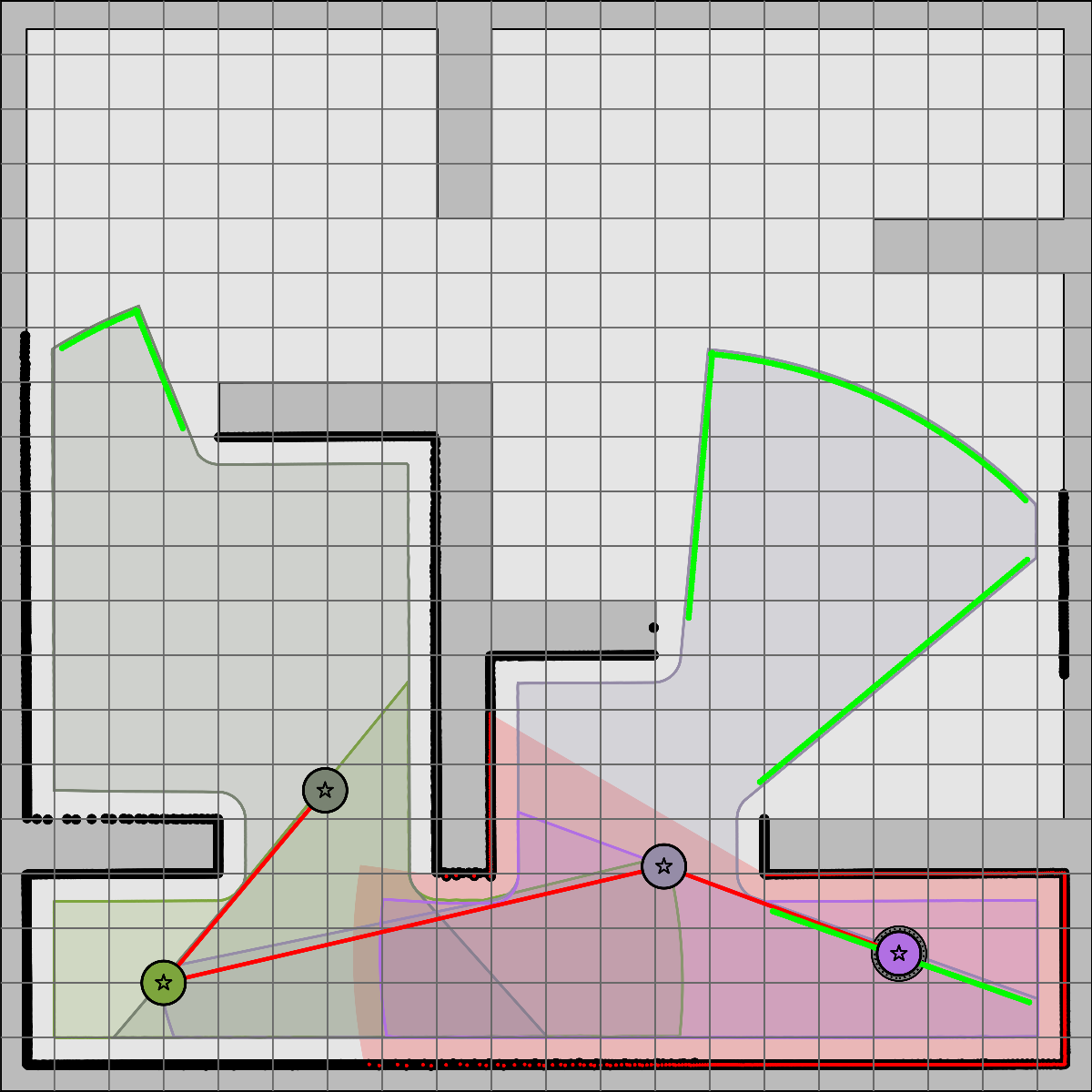}
&
\includegraphics[width=0.235\linewidth]{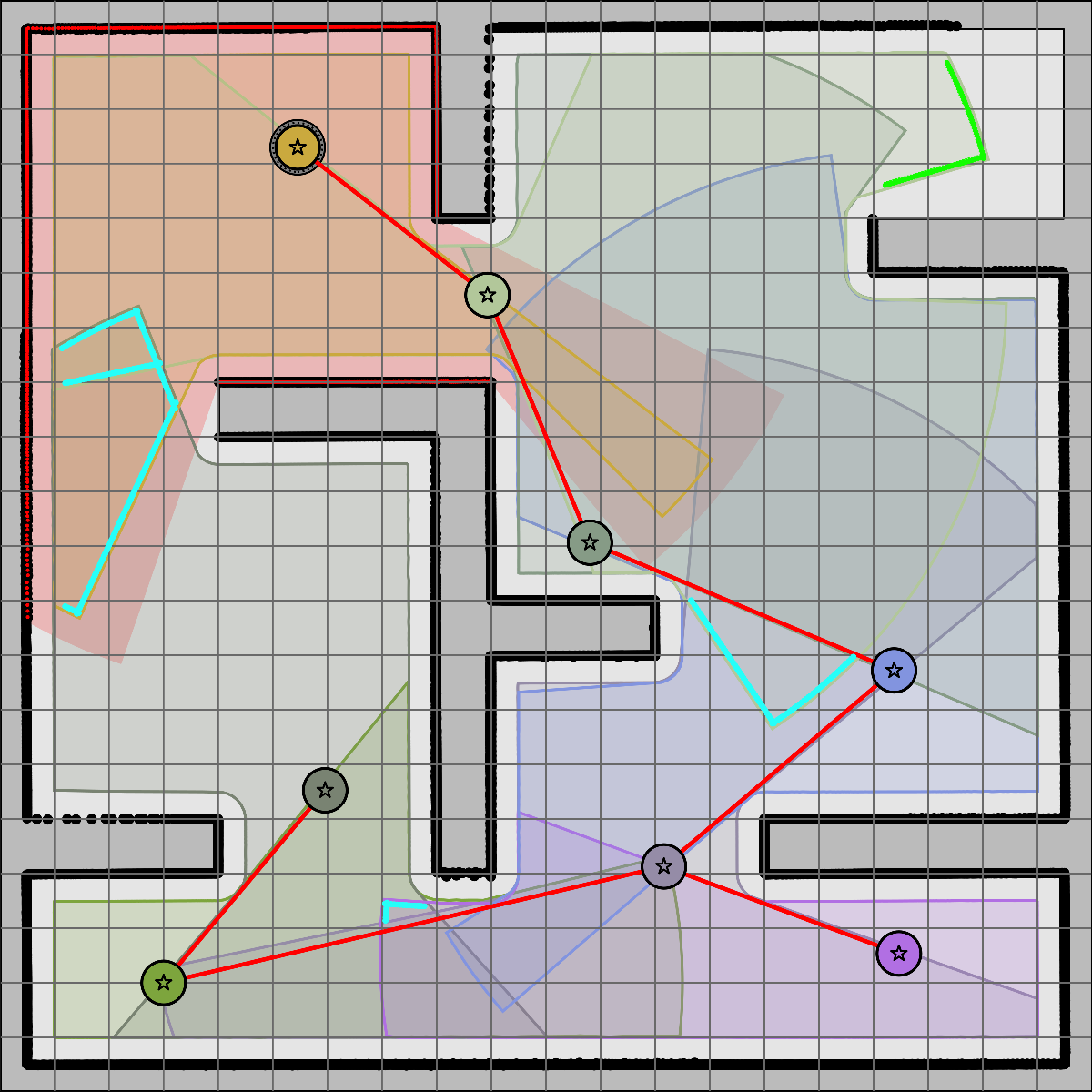}
\end{tabular}
&
\begin{tabular}{@{}c@{\hspace{1.0mm}}c@{}}
\includegraphics[width=0.235\linewidth]{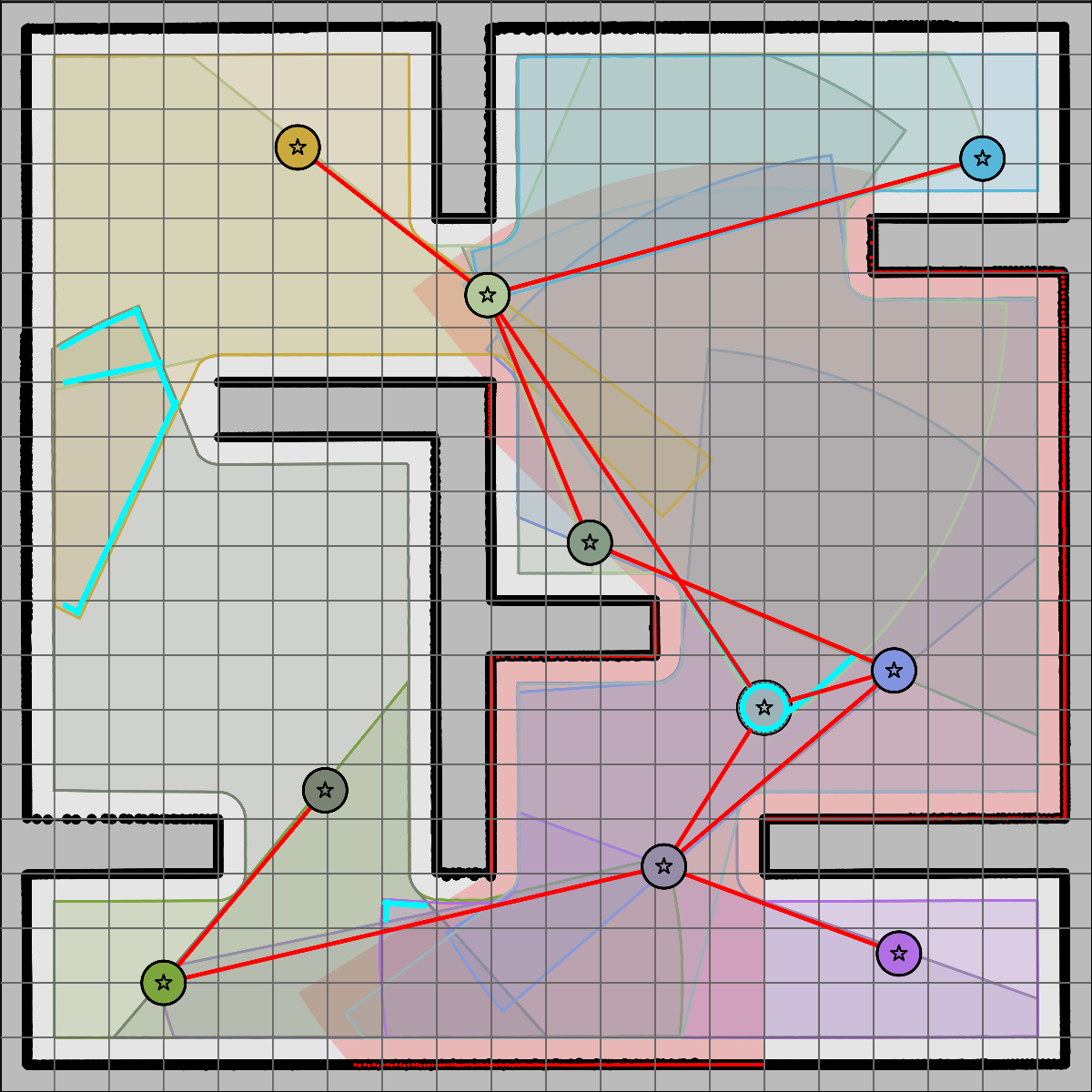}
&
\includegraphics[width=0.235\linewidth]{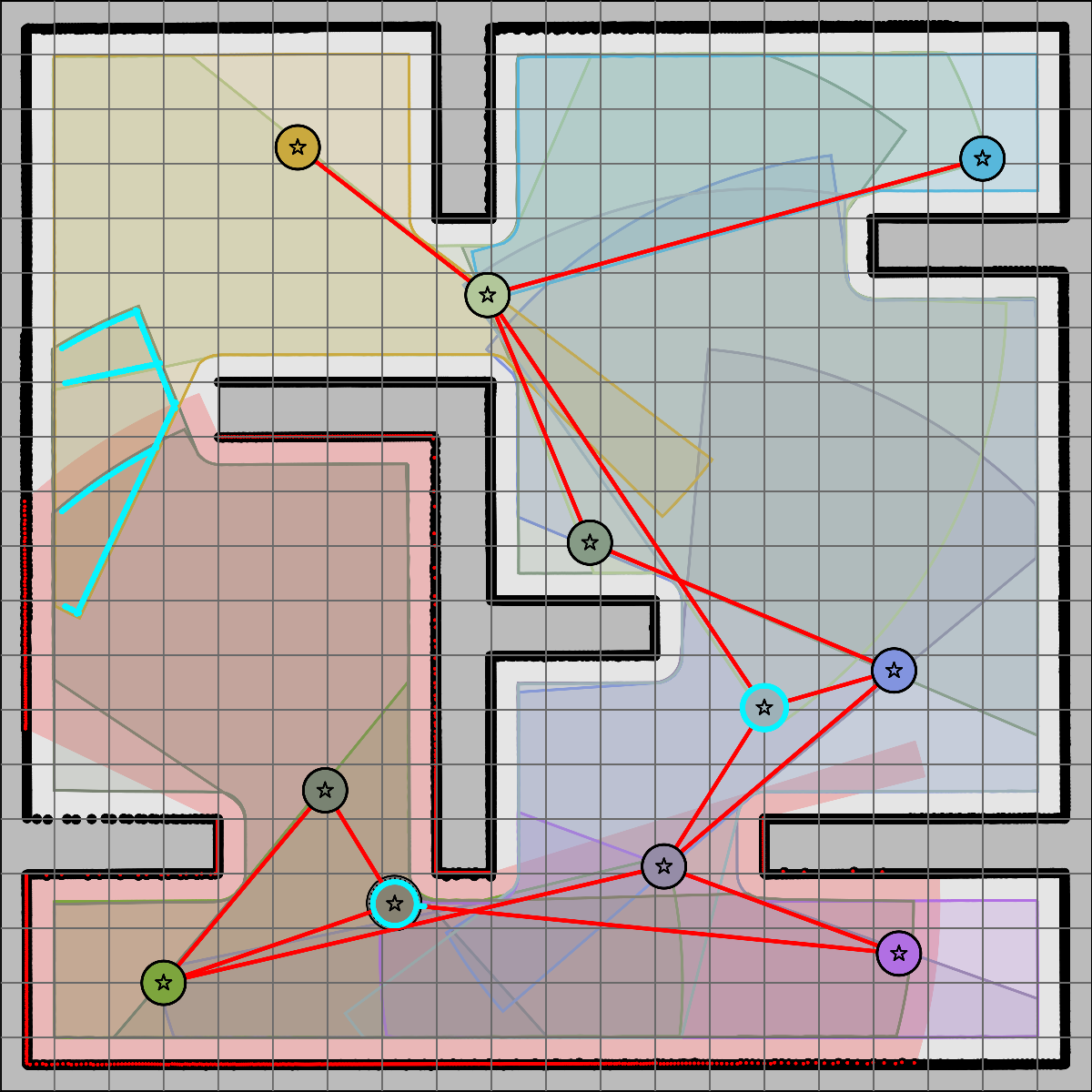}
\end{tabular}
\\[-0.5mm]
\begin{tabular}{@{}c@{}}
\includegraphics[width=0.48\linewidth]{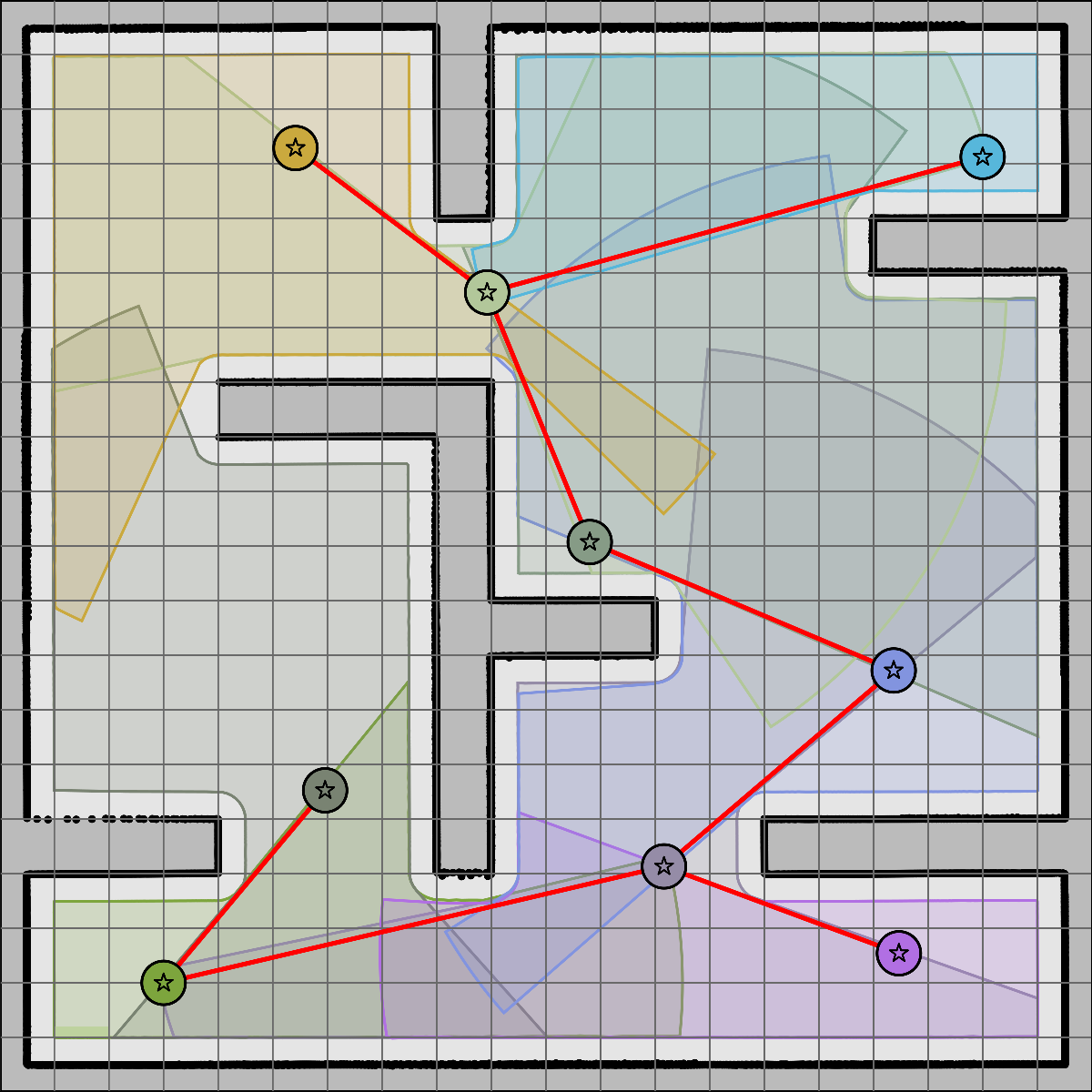}
\end{tabular}
&
\begin{tabular}{@{}c@{}}
\includegraphics[width=0.48\linewidth]{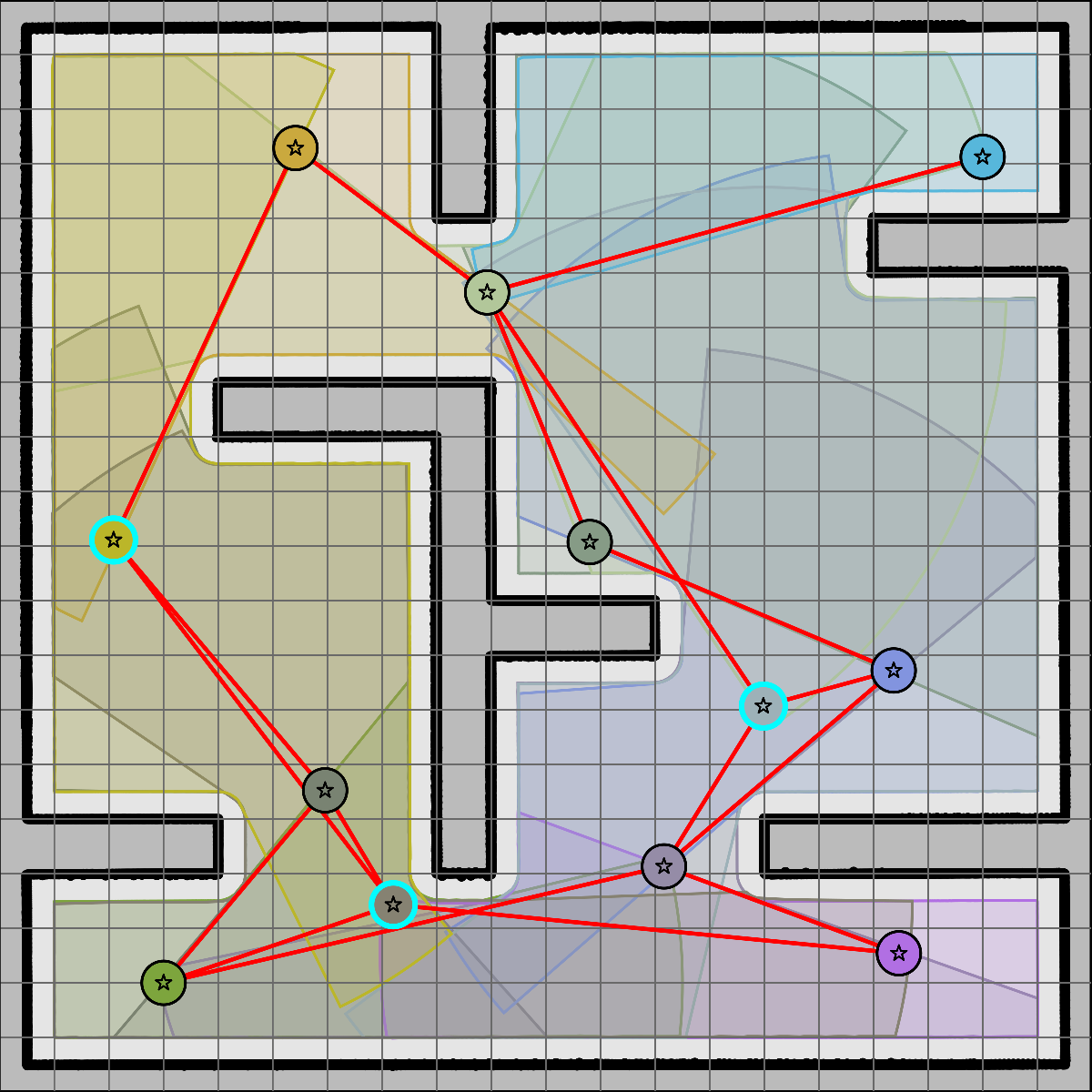}
\end{tabular}
\\[-0.5mm]
\begin{tabular}{@{}c@{}}
\includegraphics[width=0.48\linewidth]{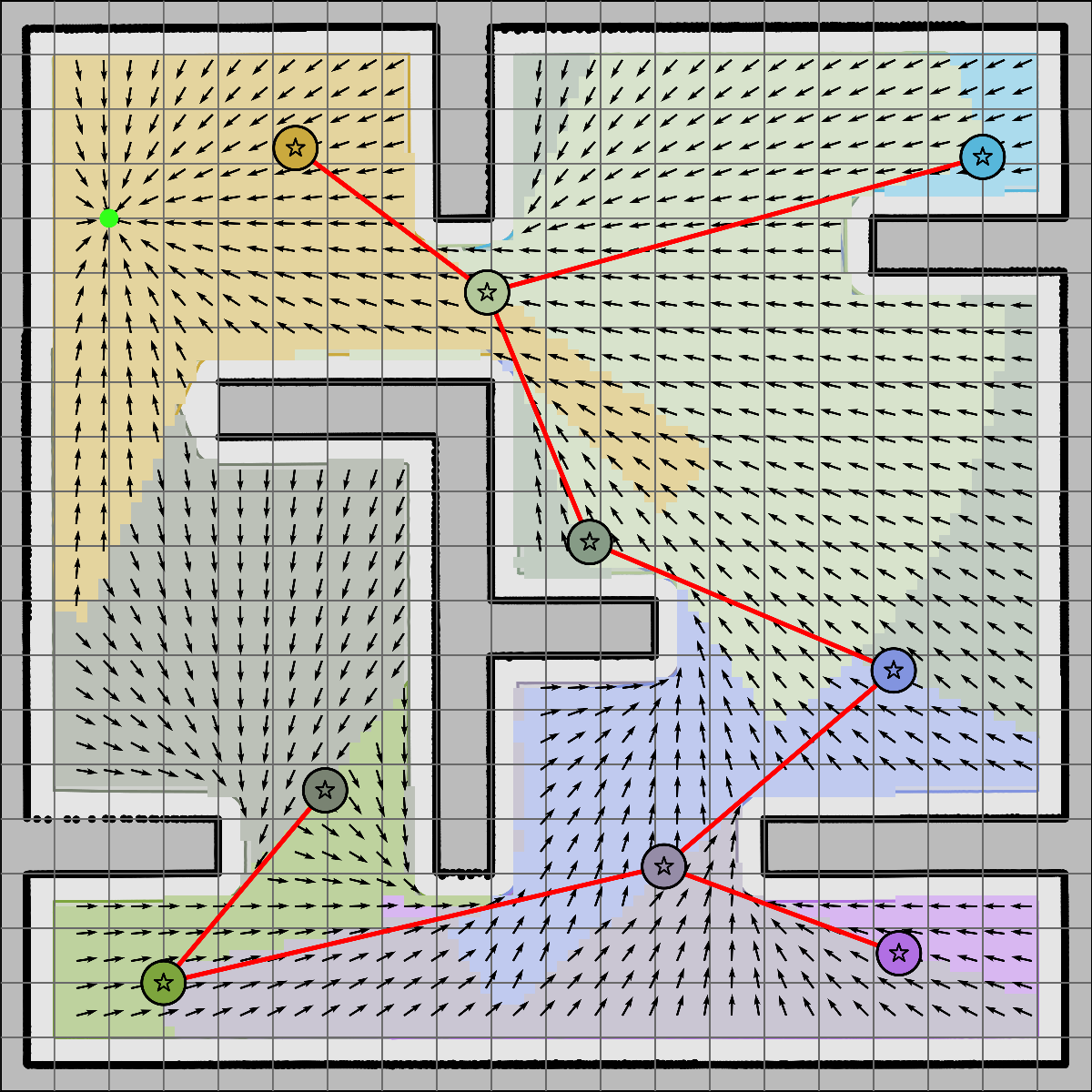}
\end{tabular}
&
\begin{tabular}{@{}c@{}}
\includegraphics[width=0.48\linewidth]{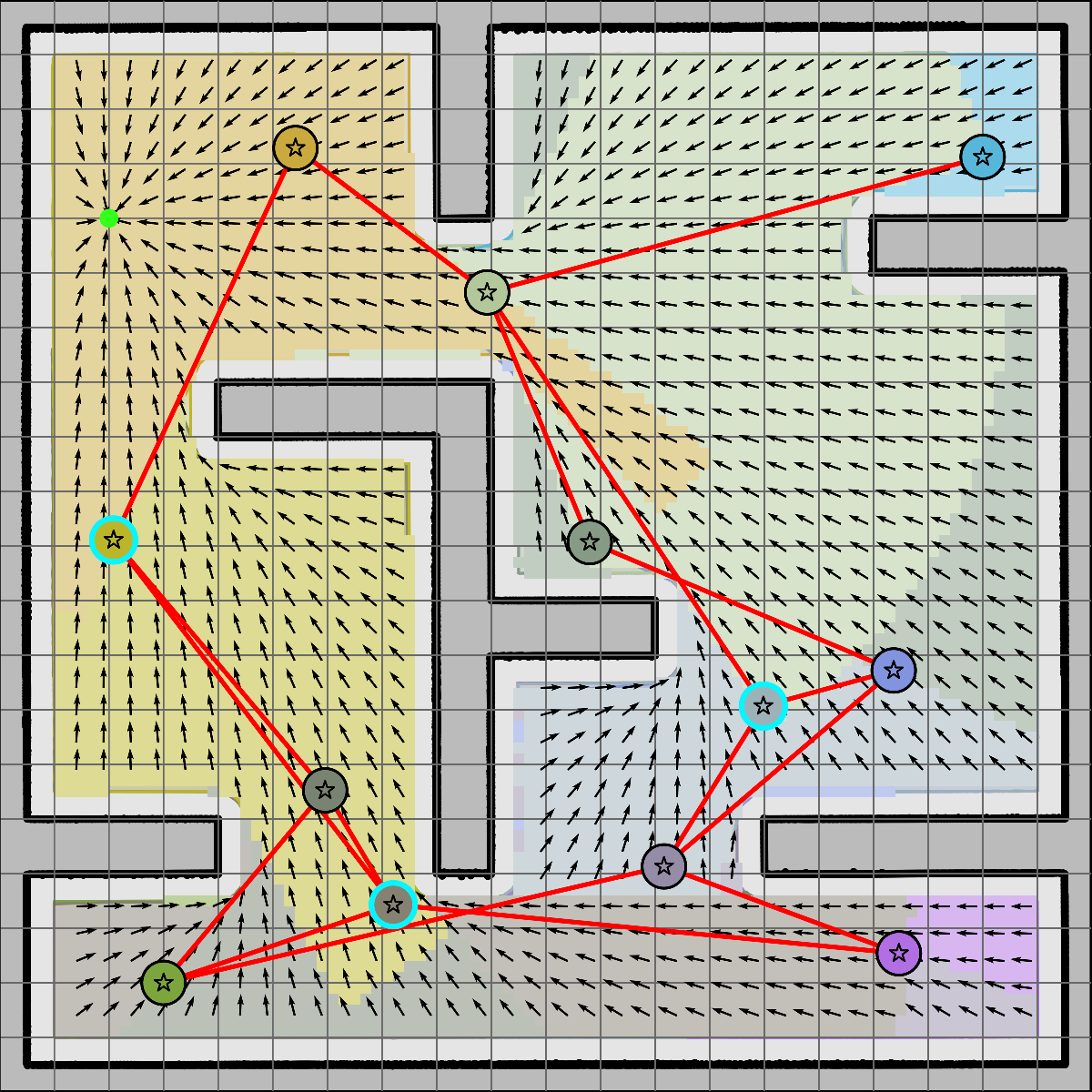}
\end{tabular}
\end{tabular}
\vspace{-2mm}
\caption{Autonomous exploration  with (left) only frontier scans (green) and (right) with additional bridging scans (cyan). Frontier exploration ensures complete metric mapping, while bridging exploration extends the motion graph with topological information and potential shortcuts (middle) for better global navigation as seen in the resulting vector fields (bottom).}        
\label{fig.autonomous_exploration_frontier_bridging_scans}
\vspace{-3mm}
\end{figure}

\section{Numerical Simulations \& Experiments}
\label{sec.numerical_simulations}

\begin{figure}[t]
\centering
\begin{tabular}{@{}c@{}}
\includegraphics[width=0.98\linewidth]{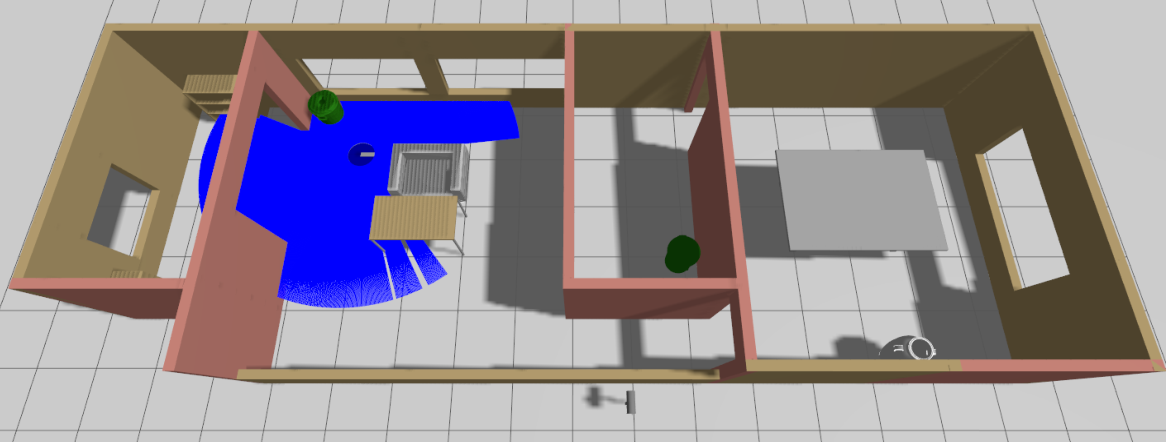}
\\[-0.5mm]
\includegraphics[width=0.98\linewidth]{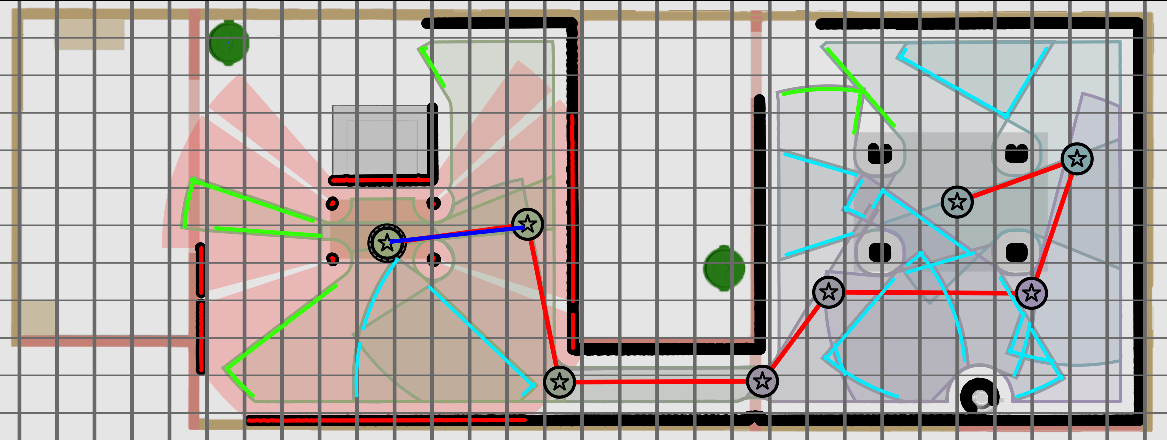}
\\[-0.5mm]
\includegraphics[width=0.98\linewidth]{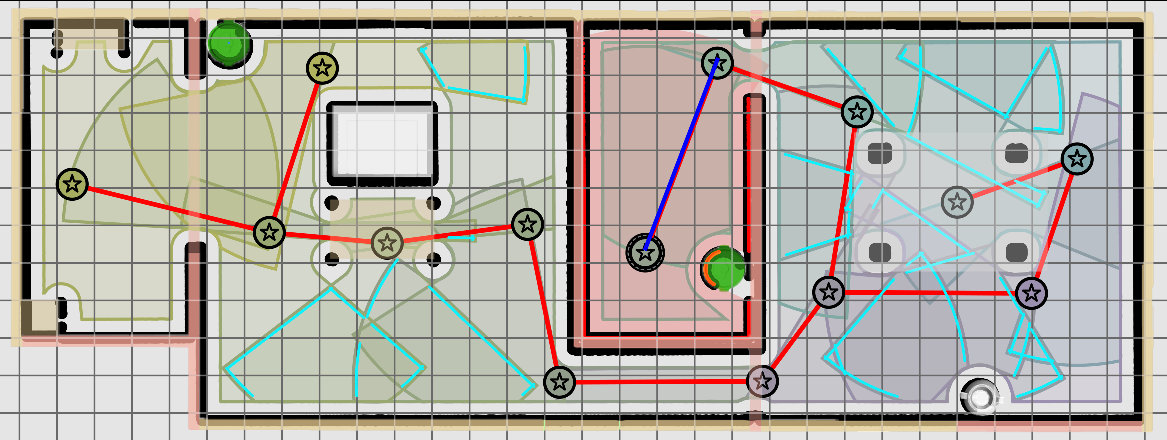}
\\[-0.5mm]
\includegraphics[width=0.98\linewidth]{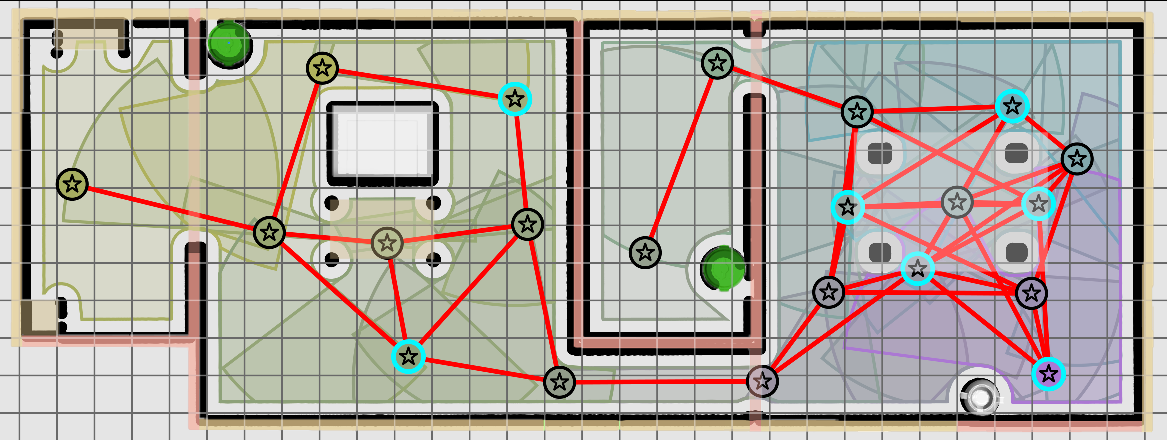}
\end{tabular}
\vspace{-1mm}
\caption{Autonomous exploration using frontier (green) and bridging (cyan) scans for key-scan-based mapping and navigation. (top) A mobile robot with a laser scanner in a simulated office-like cluttered environment.  (upper-middle) An intermediate stage of frontier-only exploration.  (lower-middle) The completed frontier-only exploration. (bottom) The complete motion graph of key scans built with frontier and bridging scan exploration.}        
\label{fig.autonomous_exploration_office_like_environment}
\vspace{-2mm}
\end{figure}

\subsection{Autonomous Exploration in an Office-Like Environment}

To demonstrate the effectiveness of our key-scan-based mapping and navigation framework in large environments, we consider autonomous exploration of a 6m $\times$ 16m office-like cluttered environment in ROS-Gazebo simulation%
\footnote{In numerical ROS-Gazebo simulations, we use a circular mobile robot with a body radius of 0.25m, equipped with a 2D $360^{\circ}$ laser scanner that generates 1081 samples with a maximum range of 3m at 30Hz. The robot's pose is obtained using a simulated motion capture system at 30Hz, and it is controlled using the move-to-projected-goal navigation policy at 30Hz with a linear gain of $\gain = 1.8$ and a maximum linear speed of 0.5m/s.}
using a fully-actuated velocity-controlled mobile robot, shown in \reffig{fig.autonomous_exploration_office_like_environment}~(top).  
As expected, by design, the robot prioritizes area coverage by exploring frontier scans first to complete the global mapping, as seen in \reffig{fig.autonomous_exploration_office_like_environment}~(lower-middle). 
It then continues exploring bridging scans to complete missing topological connections and shortcuts, improving the effectiveness of global navigation, as observed in \reffig{fig.autonomous_exploration_office_like_environment}~(bottom).
While frontier exploration always yields a spanning-tree-like motion graph of key scans with a complete metric map of the environment, bridging exploration increases the connectivity of the motion graph, capturing different ways of navigating around obstacles (e.g., around the legs of the large table).

\begin{figure*}[t]
\centering
\begin{tabular}{@{}c@{\hspace*{0.5mm}}c@{\hspace*{0.5mm}}c@{\hspace*{0.5mm}}c@{\hspace*{0.5mm}}c@{}}
\includegraphics[width=0.195\textwidth]{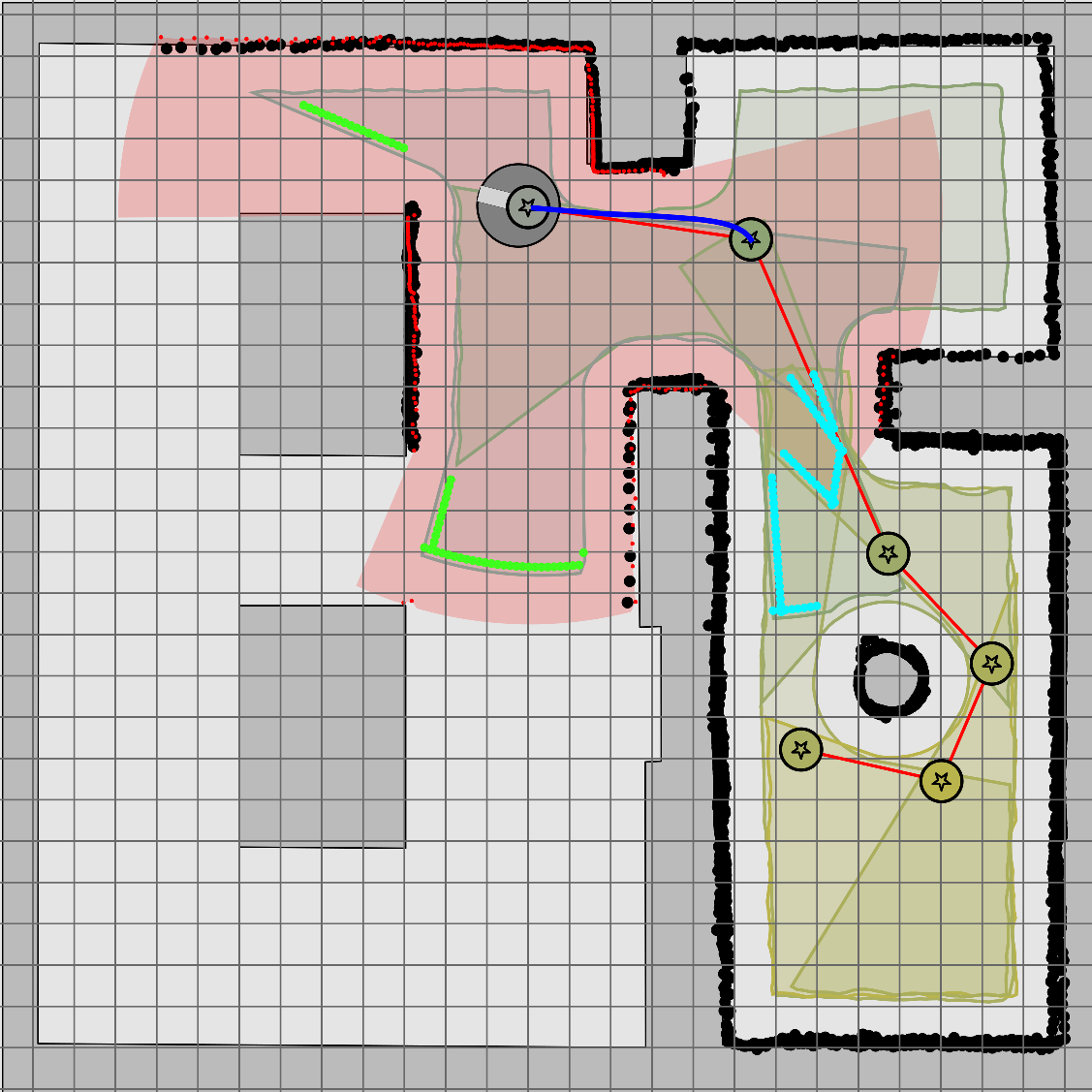}
&
\includegraphics[width=0.195\textwidth]{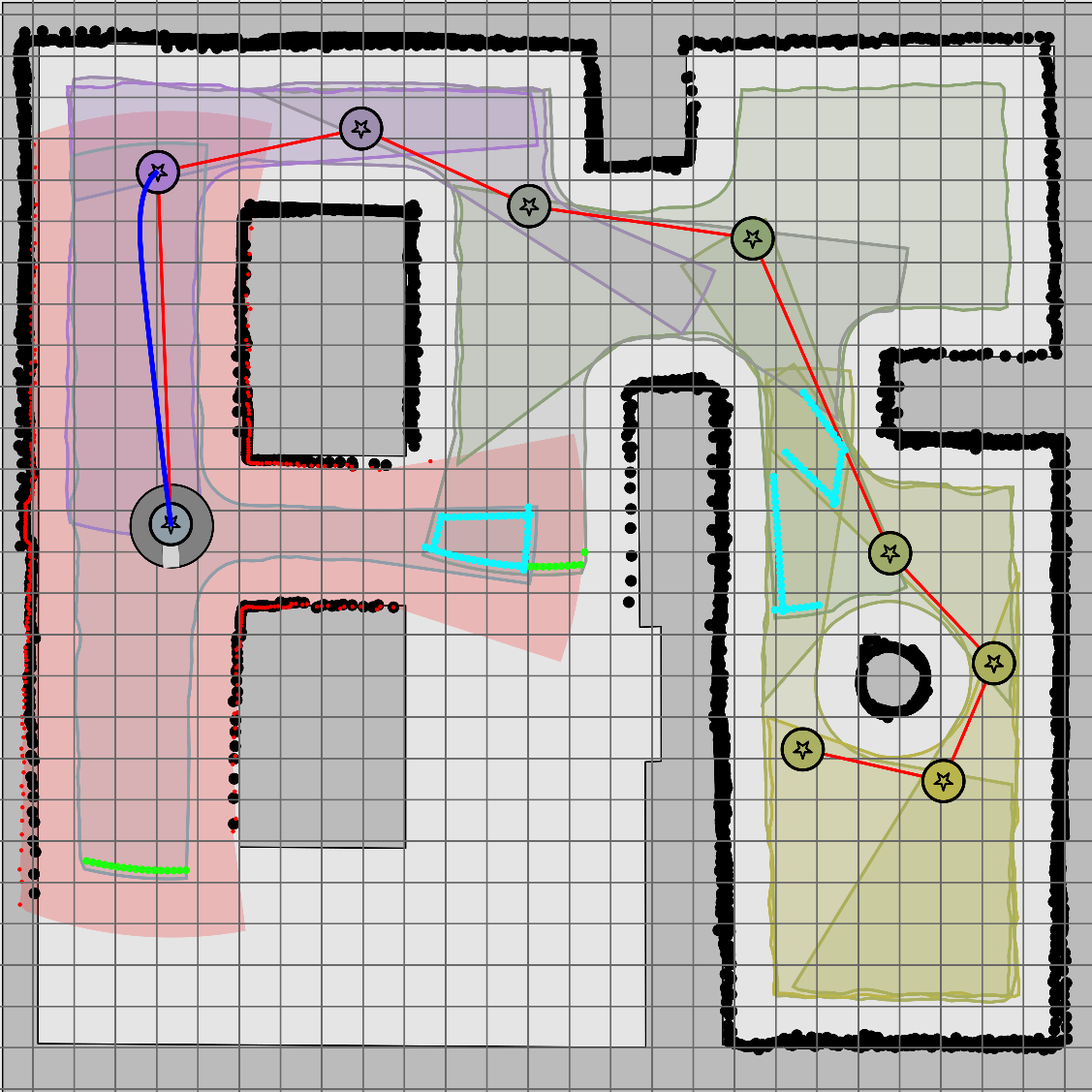}
&
\includegraphics[width=0.195\textwidth]{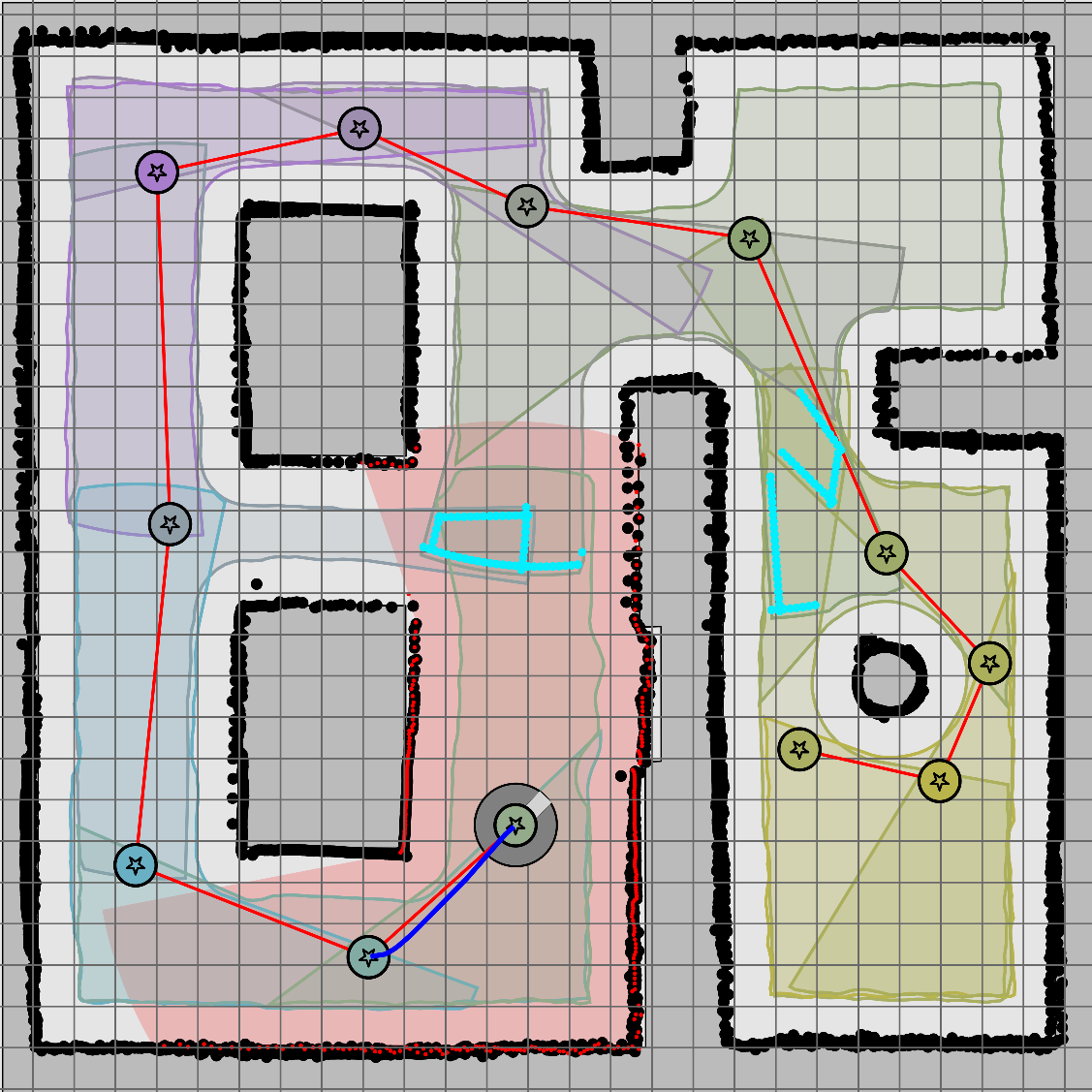}
&
\includegraphics[width=0.195\textwidth]{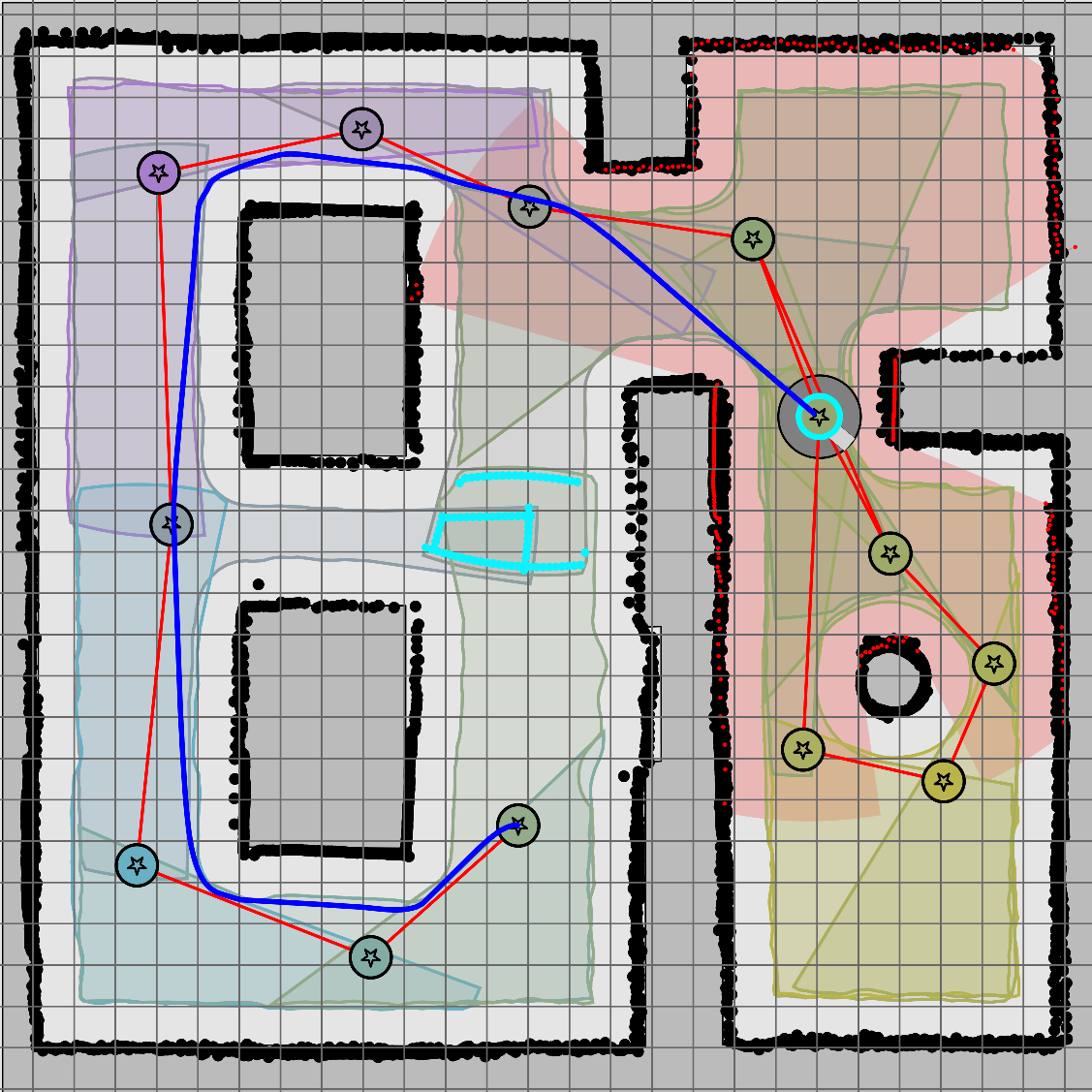}
&
\includegraphics[width=0.195\textwidth]{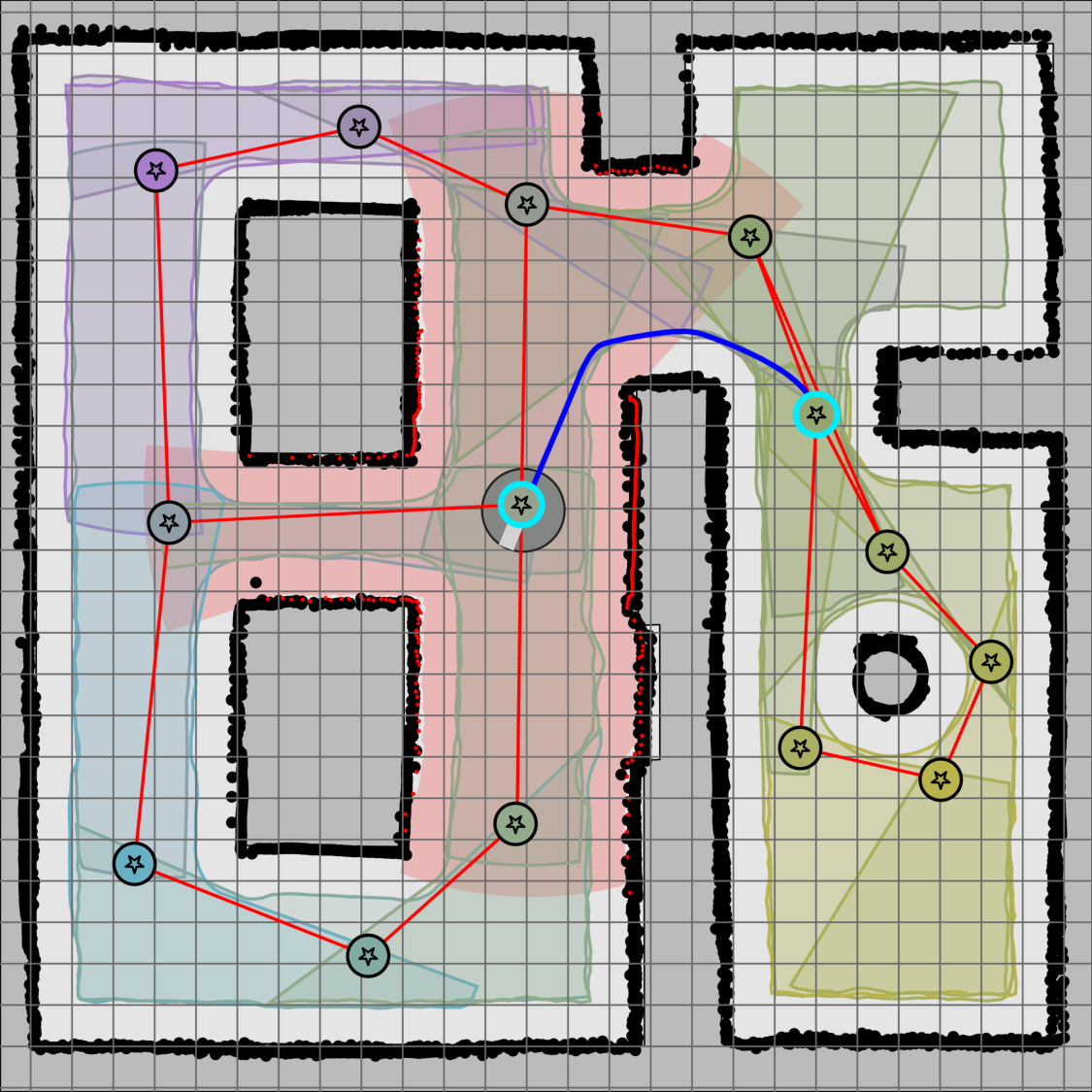}
\\[-1mm]
\includegraphics[width=0.195\textwidth]{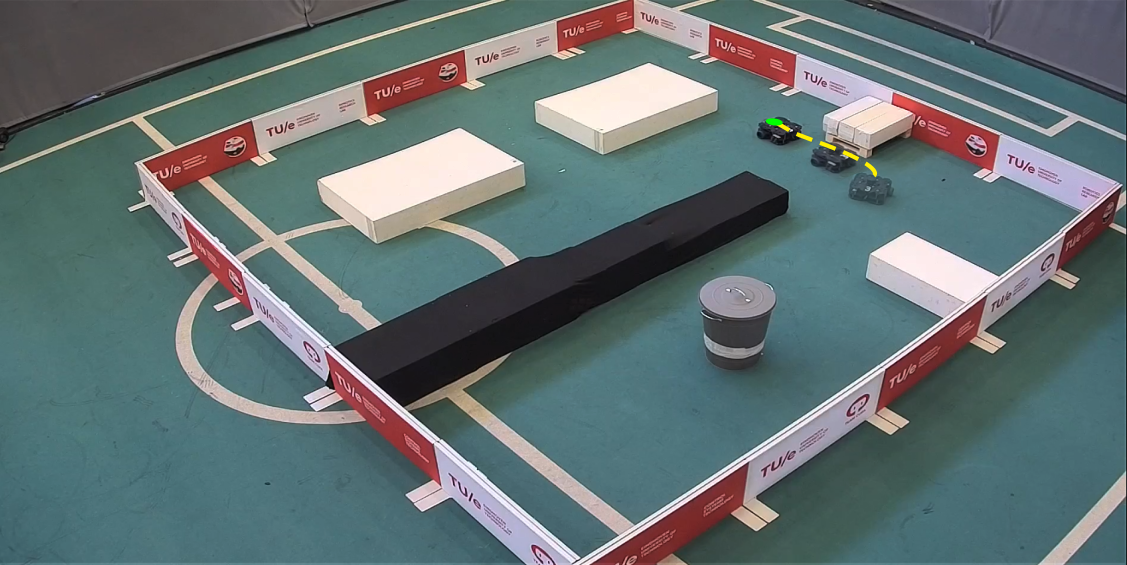}
&
\includegraphics[width=0.195\textwidth]{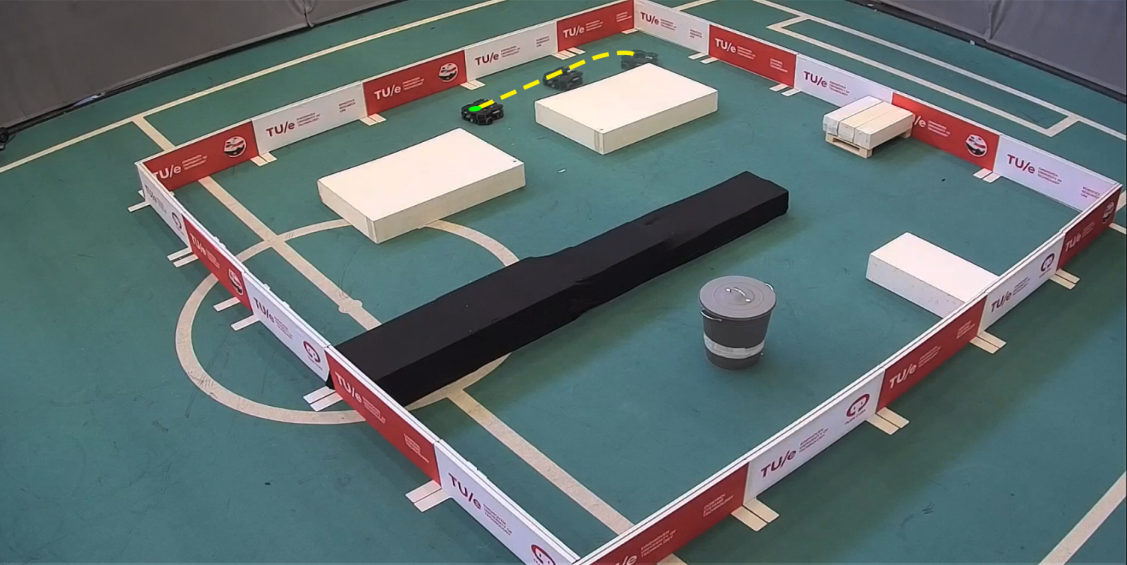}
&
\includegraphics[width=0.195\textwidth]{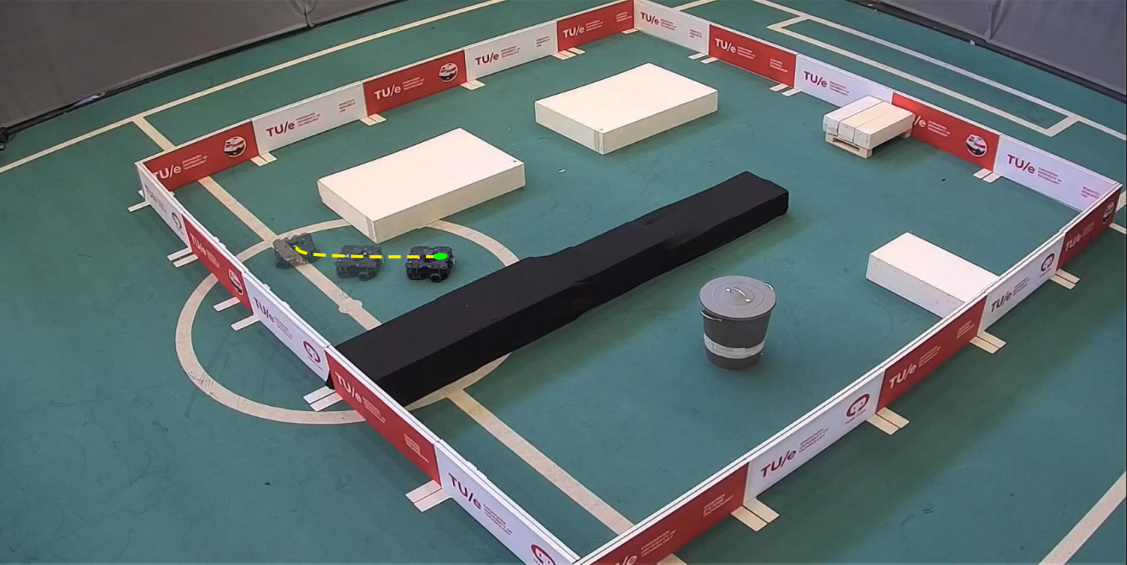}
&
\includegraphics[width=0.195\textwidth]{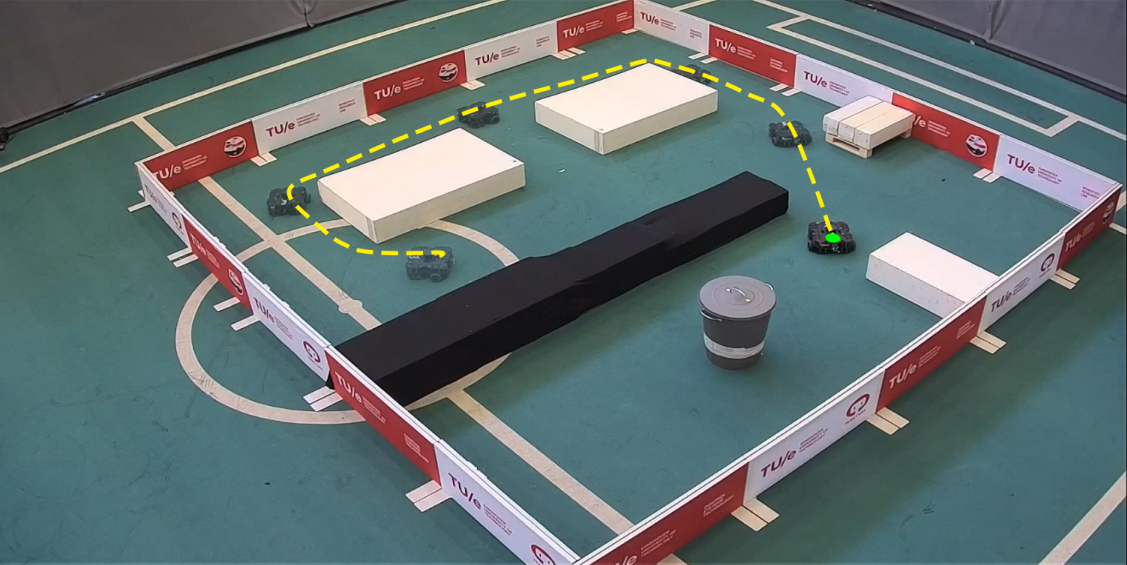}
&
\includegraphics[width=0.195\textwidth]{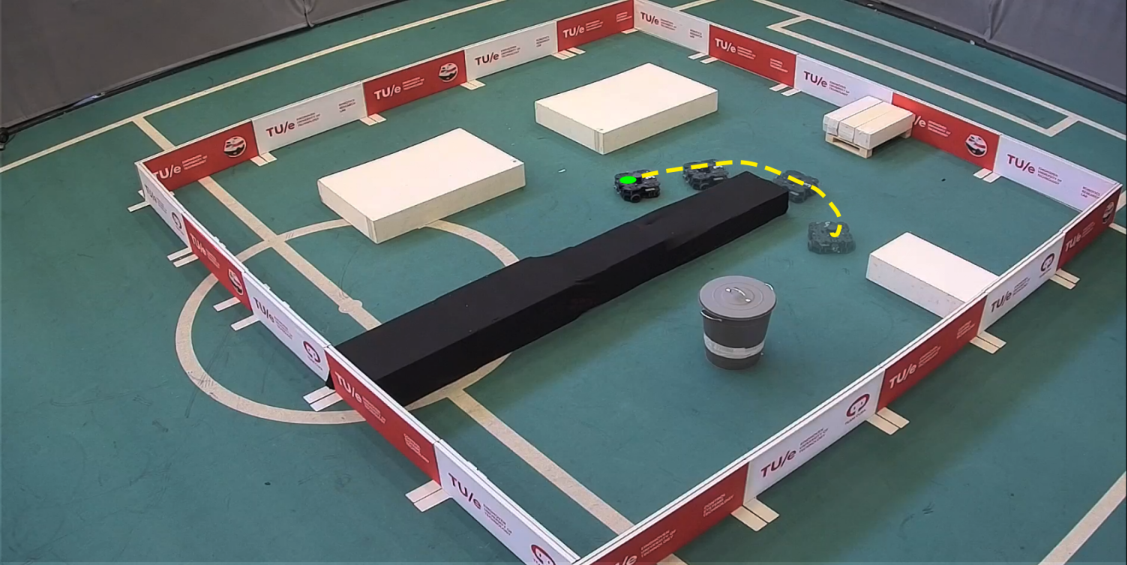}        
\\[-2.0mm]
\footnotesize{(a)} & \footnotesize{(b)} & \footnotesize{(c)} & \footnotesize{(d)} & \footnotesize{(e)}
\end{tabular}
\caption{Autonomous exploration experiment in a cluttered lab environment using a TurtleBot3 mobile robot equipped with a 2D laser scanner, employing frontier (green) and bridging (cyan) scan exploration for key-scan-based integrated mapping and navigation. (a, b) Intermediate stages of frontier-only exploration. (c) Completed frontier-only exploration. (d) First bridging exploration after the completion of frontier exploration. (e) Final completed exploration with both frontier and bridging scans. (Bottom) Example robot trajectories during autonomous exploration.}
\label{fig.autonomous_exploration_experiment}
\end{figure*}

\subsection{Autonomous Exploration Experiment with a Mobile Robot}

To demonstrate the real-time performance and applicability of our key-scan-based mapping and navigation framework in practice, we conduct physical experiments using a differential-drive TurtleBot3 Waffle Pi platform equipped with a 2D $360^{\circ}$ laser scanner (with a maximum sensing range of 2m) moving in a cluttered lab environment,%
\footnote{
The TurtleBot3 Waffle Pi robot platform has a body radius of 0.22m with respect to the motion center of its differential drive wheels and is equipped with a 2D $360^{\circ}$ LiDAR range scanner (LDS-01) generating 360 samples at 5Hz with a thresholded maximum sensing range of 2m. The robot's pose is obtained from an OptiTrack motion capture system at 30Hz, and it is controlled using a forward adaptive headway unicycle controller \cite{isleyen_vandewouw_arslan_CDC2023} based on the move-to-projected-scan-goal navigation policy, with a control gain of $\gain = 0.5$ and a headway coefficient $\kappa_{\varepsilon} = 0.5$, at a maximum linear velocity of 0.26m/s and a maximum angular velocity of 0.8rad/s.
}
tracked by an OptiTrack motion capture system for localization, see \reffig{fig.autonomous_exploration_experiment}.
We adapted our global feedback motion planner from the fully-actuated robot model to the unicycle robot dynamics using feedback motion prediction \cite{isleyen_vandewouw_arslan_IROS2023, isleyen_vandewouw_arslan_CDC2023} and reference governors \cite{isleyen_vandewouw_arslan_RAL2022}.
At the end of the frontier-based exploration, as seen in \reffig{fig.autonomous_exploration_experiment}(c), the robot automatically constructs a linear motion graph of key scans along a single path. 
While this provides complete metric mapping of the environment, it poorly represents its topological connectivity and shortcuts. 
Consequently, while navigating toward the first bridging scan, as shown in \reffig{fig.autonomous_exploration_experiment}(d), the robot takes a longer path than necessary. 
Fortunately, this issue is automatically resolved after completing both frontier and bridging exploration with an improved perception and action model of the environment, as shown in \reffig{fig.autonomous_exploration_experiment}~(e). 
As a result, the final motion graph of deployed key scans allows the robot to navigate effectively and safely in all possible directions around obstacles.

\section{Conclusions}
\label{sec.conclusions}

In this paper, we describe an integrated mapping, planning, and control framework for perception-driven mobile robot navigation in unknown unstructured environments using an incrementally built motion graph of key scan regions. 
By leveraging the star-convexity of scan regions, we present simple yet effective strategies for safe local navigation over star-convex scan polygons and apply a sequential composition of these local scan navigation policies for global feedback motion planning over the collective coverage and motion graph of the scan regions. 
We also show that the motion graph of scan regions can be used to determine informative bridging and frontier scan positions, which are then applied for key-scan selection and autonomous exploration, thereby facilitating active integrated mapping and navigation in unknown environments.
In particular, the new concept of bridging scans allows for loop closing and completing missing topological connections and shortcuts in the motion graphs of star-convex scan polygons, resulting in a more accurate perception and action model of the environment for global navigation.
We demonstrate the effectiveness and applicability of our key-scan-based mapping and navigation framework through numerical simulations and real physical hardware experiments.
We conclude that tightly coupling perception, planning, and control at the design stage is a key enabler for better action and perception in robotics.

We are currently working on the systematic integration of localization, mapping, planning, and control for fully autonomous, safe, and reliable mobile robot navigation in large indoor human environments over long durations. 
Another promising research direction is extending these concepts to 3D perception, planning, and control for aerial robots and 3D navigation settings, as well as optimizing the placement of key scans to minimize redundancy and avoid misplacements.





%

%



\bibliographystyle{IEEEtran}
\bibliography{references}

\begin{thebibliography}{10}
\providecommand{\url}[1]{#1}
\csname url@rmstyle\endcsname
\providecommand{\newblock}{\relax}
\providecommand{\bibinfo}[2]{#2}
\providecommand\BIBentrySTDinterwordspacing{\spaceskip=0pt\relax}
\providecommand\BIBentryALTinterwordstretchfactor{4}
\providecommand\BIBentryALTinterwordspacing{\spaceskip=\fontdimen2\font plus
\BIBentryALTinterwordstretchfactor\fontdimen3\font minus
  \fontdimen4\font\relax}
\providecommand\BIBforeignlanguage[2]{{%
\expandafter\ifx\csname l@#1\endcsname\relax
\typeout{** WARNING: IEEEtran.bst: No hyphenation pattern has been}%
\typeout{** loaded for the language `#1'. Using the pattern for}%
\typeout{** the default language instead.}%
\else
\language=\csname l@#1\endcsname
\fi
#2}}

\bibitem{renan_nascimento_RAS2021}
{\'{I}}.~R. da~Costa~Barros and T.~P. Nascimento, ``Robotic mobile fulfillment
  systems: A survey on recent developments and research opportunities,''
  \emph{Robot. Auton. Syst.}, vol. 137, p. 103729, 2021.

\bibitem{taranta_etal_INDIN2021}
D.~Taranta, F.~Marques, A.~Lourenço, P.~A. Prates, A.~Souto, E.~Pinto, and
  J.~Barata, ``An autonomous mobile robot navigation architecture for dynamic
  intralogistics,'' in \emph{IEEE International Conference on Industrial
  Informatics}, 2021, pp. 1--6.

\bibitem{bettencourt_lima_ICARSC2021}
R.~Bettencourt and P.~U. Lima, ``Multimodal navigation for autonomous service
  robots,'' in \emph{IEEE International Conference on Autonomous Robot Systems
  and Competitions}, 2021, pp. 25--30.

\bibitem{gul_rahiman_alhady_sahal_CE2019}
F.~Gul, W.~Rahiman, and S.~S. Nazli~Alhady, ``A comprehensive study for robot
  navigation techniques,'' \emph{Cogent Engineering}, vol.~6, no.~1, 2019.

\bibitem{paola_etal_IJARS2010}
D.~D. Paola, A.~Milella, G.~Cicirelli, and A.~Distante, ``An autonomous mobile
  robotic system for surveillance of indoor environments,'' \emph{International
  Journal of Advanced Robotic Systems}, vol.~7, no.~1, p.~8, 2010.

\bibitem{halder_afsari_AS2023}
S.~Halder and K.~Afsari, ``Robots in inspection and monitoring of buildings and
  infrastructure: A systematic review,'' \emph{Applied Sciences}, vol.~13,
  no.~4, 2023.

\bibitem{oriolo_ulivi_vendittelli_TSMC1998}
G.~Oriolo, G.~Ulivi, and M.~Vendittelli, ``Real-time map building and
  navigation for autonomous robots in unknown environments,'' \emph{IEEE
  Transactions on Systems, Man, and Cybernetics}, vol.~28, no.~3, pp. 316--333,
  1998.

\bibitem{placed_etal_TRO2023}
J.~A. Placed, J.~Strader, H.~Carrillo, N.~Atanasov, V.~Indelman, L.~Carlone,
  and J.~A. Castellanos, ``A survey on active simultaneous localization and
  mapping: State of the art and new frontiers,'' \emph{IEEE Transactions on
  Robotics}, vol.~39, no.~3, pp. 1686--1705, 2023.

\bibitem{thrun_AI1998}
S.~Thrun, ``Learning metric-topological maps for indoor mobile robot
  navigation,'' \emph{Artificial Intelligence}, vol.~99, no.~1, pp. 21--71,
  1998.

\bibitem{konolige_marder_marthi_ICRA2011}
K.~Konolige, E.~Marder-Eppstein, and B.~Marthi, ``Navigation in hybrid
  metric-topological maps,'' in \emph{IEEE International Conference on Robotics
  and Automation}, 2011, pp. 3041--3047.

\bibitem{canny_ComplexityRobotMotionPlanning1988}
J.~Canny, \emph{The complexity of robot motion planning}.\hskip 1em plus 0.5em
  minus 0.4em\relax MIT press, 1988.

\bibitem{choset_etal_PrinciplesOfRobotMotion2005}
H.~M. Choset, K.~M. Lynch, S.~Hutchinson, G.~Kantor, W.~Burgard, L.~Kavraki,
  S.~Thrun, and R.~C. Arkin, \emph{Principles of Robot Motion: Theory,
  Algorithms, and Implementations}.\hskip 1em plus 0.5em minus 0.4em\relax MIT
  Press, 2005.

\bibitem{siciliano_etal_RoboticsModellingPlanningControl2009}
B.~Siciliano, L.~Sciavicco, L.~Villani, and G.~Oriolo, \emph{Robotics:
  Modelling, Planning and Control}.\hskip 1em plus 0.5em minus 0.4em\relax
  Springer, 2009.

\bibitem{lynch_park_ModernRobotics2017}
K.~M. Lynch and F.~C. Park, \emph{Modern Robotics: Mechanics, Planning, and
  Control}.\hskip 1em plus 0.5em minus 0.4em\relax Cambridge University Press,
  2017.

\bibitem{aguiar_hespanha_kokotovic_Automatica2008}
A.~P. Aguiar, J.~P. Hespanha, and P.~V. Kokotović, ``Performance limitations
  in reference tracking and path following for nonlinear systems,''
  \emph{Automatica}, vol.~44, no.~3, pp. 598--610, 2008.

\bibitem{koenig_likhachev_TRO2005}
S.~Koenig and M.~Likhachev, ``Fast replanning for navigation in unknown
  terrain,'' \emph{IEEE Transactions on Robotics}, vol.~21, no.~3, pp.
  354--363, 2005.

\bibitem{ding_gao_wang_shen_TRO2019}
W.~Ding, W.~Gao, K.~Wang, and S.~Shen, ``An efficient b-spline-based
  kinodynamic replanning framework for quadrotors,'' \emph{IEEE Transactions on
  Robotics}, vol.~35, no.~6, pp. 1287--1306, 2019.

\bibitem{tordesillas_etal_TRO2021}
J.~Tordesillas, B.~T. Lopez, M.~Everett, and J.~P. How, ``Faster: Fast and safe
  trajectory planner for navigation in unknown environments,'' \emph{IEEE
  Transactions on Robotics}, pp. 1--17, 2021.

\bibitem{nguyen_etal_ECC2021}
H.~Nguyen, M.~Kamel, K.~Alexis, and R.~Siegwart, ``Model predictive control for
  micro aerial vehicles: A survey,'' in \emph{European Control Conference},
  2021, pp. 1556--1563.

\bibitem{deits_tedrake_ICRA2015}
R.~Deits and R.~Tedrake, ``Efficient mixed-integer planning for uavs in
  cluttered environments,'' in \emph{IEEE International Conference on Robotics
  and Automation}, 2015, pp. 42--49.

\bibitem{chen_liu_shen_ICRA2016}
J.~Chen, T.~Liu, and S.~Shen, ``Online generation of collision-free
  trajectories for quadrotor flight in unknown cluttered environments,'' in
  \emph{IEEE International Conference on Robotics and Automation}, 2016, pp.
  1476--1483.

\bibitem{liu_etal_RAL2017}
S.~Liu, M.~Watterson, K.~Mohta, K.~Sun, S.~Bhattacharya, C.~J. Taylor, and
  V.~Kumar, ``Planning dynamically feasible trajectories for quadrotors using
  safe flight corridors in 3-d complex environments,'' \emph{IEEE Robotics and
  Automation Letters}, vol.~2, no.~3, pp. 1688--1695, July 2017.

\bibitem{gao_etal_JFR2019}
F.~Gao, W.~Wu, W.~Gao, and S.~Shen, ``Flying on point clouds: Online trajectory
  generation and autonomous navigation for quadrotors in cluttered
  environments,'' \emph{Journal of Field Robotics}, vol.~36, no.~4, pp.
  710--733, 2019.

\bibitem{marcucci_etal_JoO2024}
T.~Marcucci, J.~Umenberger, P.~Parrilo, and R.~Tedrake, ``Shortest paths in
  graphs of convex sets,'' \emph{SIAM Journal on Optimization}, vol.~34, no.~1,
  pp. 507--532, 2024.

\bibitem{marcucci_etal_SR2023}
T.~Marcucci, M.~Petersen, D.~von Wrangel, and R.~Tedrake, ``Motion planning
  around obstacles with convex optimization,'' \emph{Science Robotics}, vol.~8,
  no.~84, 2023.

\bibitem{zhou_etal_TRO2021}
B.~Zhou, J.~Pan, F.~Gao, and S.~Shen, ``Raptor: Robust and perception-aware
  trajectory replanning for quadrotor fast flight,'' \emph{IEEE Transactions on
  Robotics}, vol.~37, no.~6, pp. 1992--2009, 2021.

\bibitem{khatib_IJRR1986}
O.~Khatib, ``Real-time obstacle avoidance for manipulators and mobile robots,''
  \emph{The International Journal of Robotics Research}, vol.~5, no.~1, pp.
  90--98, 1986.

\bibitem{rimon_kod_TRA1992}
E.~Rimon and D.~Koditschek, ``Exact robot navigation using artificial potential
  functions,'' \emph{IEEE Transactions on Robotics and Automation}, vol.~8,
  no.~5, pp. 501--518, 1992.

\bibitem{koren_borenstein_ICRA1991}
Y.~Koren and J.~Borenstein, ``Potential field methods and their inherent
  limitations for mobile robot navigation,'' in \emph{IEEE International
  Conference on Robotics and Automation}, 1991, pp. 1398--1404 vol.2.

\bibitem{barraquand_langlois_latombe_TSMC1992}
J.~Barraquand, B.~Langlois, and J.~C. Latombe, ``Numerical potential field
  techniques for robot path planning,'' \emph{IEEE Transactions on Systems,
  Man, and Cybernetics}, vol.~22, no.~2, pp. 224--241, 1992.

\bibitem{burridge_rizzi_koditschek_IJRR1999}
R.~R. Burridge, A.~A. Rizzi, and D.~E. Koditschek, ``Sequential composition of
  dynamically dexterous robot behaviors,'' \emph{The International Journal of
  Robotics Research}, vol.~18, no.~6, pp. 535--555, 1999.

\bibitem{belta_isler_pappas_TRO2005}
C.~Belta, V.~Isler, and G.~Pappas, ``Discrete abstractions for robot motion
  planning and control in polygonal environments,'' \emph{IEEE Transactions on
  Robotics}, vol.~21, no.~5, pp. 864--874, 2005.

\bibitem{conner_choset_rizzi_tro2009}
D.~Conner, H.~Choset, and A.~Rizzi, ``Flow-through policies for hybrid
  controller synthesis applied to fully actuated systems,'' \emph{IEEE
  Transactions on Robotics}, vol.~25, no.~1, pp. 136--146, 2009.

\bibitem{conner_howie_rizzi_AR2011}
D.~C. Conner, H.~Choset, and A.~A. Rizzi, ``Integrating planning and control
  for single-bodied wheeled mobile robots,'' \emph{Autonomous Robots}, vol.~30,
  no.~3, pp. 243--264, 2011.

\bibitem{arslan_saranli_TRO2012}
{\"O}.~Arslan and U.~Saranl{\i}, ``Reactive planning and control of planar
  spring-mass running on rough terrain,'' \emph{Robotics, IEEE Transactions
  on}, vol.~28, no.~3, pp. 567--579, 2012.

\bibitem{arslan_guralnik_kod_TRO2016}
O.~Arslan, D.~P. Guralnik, and D.~E. Koditschek, ``Coordinated robot navigation
  via hierarchical clustering,'' \emph{IEEE Transactions of Robotics}, vol.~32,
  no.~2, pp. 352--371, 2016.

\bibitem{arslan_koditschek_IJRR2019}
{\"O}.~Arslan and D.~E. Koditschek, ``Sensor-based reactive navigation in
  unknown convex sphere worlds,'' \emph{The International Journal of Robotics
  Research}, vol.~38, no. 2-3, pp. 196--223, 2019.

\bibitem{elfes_C1989}
A.~Elfes, ``Using occupancy grids for mobile robot perception and navigation,''
  \emph{Computer}, vol.~22, no.~6, pp. 46--57, 1989.

\bibitem{lavalle_PlanningAlgorithms2006}
S.~M. LaValle, \emph{Planning Algorithms}.\hskip 1em plus 0.5em minus
  0.4em\relax Cambridge University Press, 2006.

\bibitem{bormann_etal_ICRA2016}
R.~Bormann, F.~Jordan, W.~Li, J.~Hampp, and M.~Hägele, ``Room segmentation:
  Survey, implementation, and analysis,'' in \emph{IEEE International
  Conference on Robotics and Automation}, 2016, pp. 1019--1026.

\bibitem{blochliger_etal_ICRA2018}
F.~Blochliger, M.~Fehr, M.~Dymczyk, T.~Schneider, and R.~Siegwart, ``Topomap:
  Topological mapping and navigation based on visual slam maps,'' in \emph{IEEE
  International Conference on Robotics and Automation}, 2018, pp. 3818--3825.

\bibitem{chen_etal_IROS2022}
X.~Chen, B.~Zhou, J.~Lin, Y.~Zhang, F.~Zhang, and S.~Shen, ``Fast 3d sparse
  topological skeleton graph generation for mobile robot global planning,'' in
  \emph{IEEE/RSJ International Conference on Intelligent Robots and Systems},
  2022, pp. 10\,283--10\,289.

\bibitem{lu_milios_AR1997}
F.~Lu and E.~Milios, ``Globally consistent range scan alignment for environment
  mapping,'' \emph{Autonomous robots}, vol.~4, pp. 333--349, 1997.

\bibitem{grisetti_etal_MITS2010}
G.~Grisetti, R.~Kümmerle, C.~Stachniss, and W.~Burgard, ``A tutorial on
  graph-based slam,'' \emph{IEEE Intelligent Transportation Systems Magazine},
  vol.~2, no.~4, pp. 31--43, 2010.

\bibitem{mendes_koch_lacroix_SSRR2016}
E.~Mendes, P.~Koch, and S.~Lacroix, ``Icp-based pose-graph slam,'' in
  \emph{IEEE International Symposium on Safety, Security, and Rescue Robotics},
  2016, pp. 195--200.

\bibitem{lluvia_lazkano_ansuategi_Sensors2021}
I.~Lluvia, E.~Lazkano, and A.~Ansuategi, ``Active mapping and robot
  exploration: A survey,'' \emph{Sensors}, vol.~21, no.~7, 2021.

\bibitem{yamauchi_CIRA1997}
B.~Yamauchi, ``A frontier-based approach for autonomous exploration,'' in
  \emph{IEEE International Symposium on Computational Intelligence in Robotics
  and Automation}, 1997, pp. 146--151.

\bibitem{stachniss_hahnel_burgard_IROS2004}
C.~Stachniss, D.~Hahnel, and W.~Burgard, ``Exploration with active loop-closing
  for fastslam,'' in \emph{IEEE/RSJ International Conference on Intelligent
  Robots and Systems}, 2004, pp. 1505--1510 vol.2.

\bibitem{valencia_etal_IROS2012}
R.~Valencia, J.~Valls~Miró, G.~Dissanayake, and J.~Andrade-Cetto, ``Active
  pose slam,'' in \emph{IEEE/RSJ International Conference on Intelligent Robots
  and Systems}, 2012, pp. 1885--1891.

\bibitem{fermin-leon_neira_castellanos_ECMR2017}
L.~Fermin-Leon, J.~Neira, and J.~A. Castellanos, ``Tigre: Topological graph
  based robotic exploration,'' in \emph{European Conference on Mobile Robots},
  2017, pp. 1--6.

\bibitem{kim_eustice_IJRR2015}
A.~Kim and R.~M. Eustice, ``Active visual slam for robotic area coverage:
  Theory and experiment,'' \emph{The International Journal of Robotics
  Research}, vol.~34, no. 4-5, pp. 457--475, 2015.

\bibitem{suresh_etal_ICRA2020}
S.~Suresh, P.~Sodhi, J.~G. Mangelson, D.~Wettergreen, and M.~Kaess, ``Active
  slam using 3d submap saliency for underwater volumetric exploration,'' in
  \emph{IEEE International Conference on Robotics and Automation}, 2020, pp.
  3132--3138.

\bibitem{yang_etal_ICRA2021}
F.~Yang, D.-H. Lee, J.~Keller, and S.~Scherer, ``Graph-based topological
  exploration planning in large-scale 3d environments,'' in \emph{IEEE
  International Conference on Robotics and Automation}, 2021, pp.
  12\,730--12\,736.

\bibitem{segal_haehnel_thrun_RSS2009}
A.~Segal, D.~Haehnel, and S.~Thrun, ``Generalized-{ICP}.'' in \emph{Robotics:
  Science and Systems}, 2009.

\bibitem{dellaert_kaess_FTR2017}
F.~Dellaert and M.~Kaess, ``Factor graphs for robot perception,''
  \emph{Foundations and Trends{\textregistered} in Robotics}, vol.~6, no. 1-2,
  pp. 1--139, 2017.

\bibitem{webster_Convexity1995}
R.~Webster, \emph{Convexity}.\hskip 1em plus 0.5em minus 0.4em\relax Oxford
  University Press, 1995.

\bibitem{liu_JCO1995}
J.~Liu, ``Sensitivity analysis in nonlinear programs and variational
  inequalities via continuous selections,'' \emph{SIAM Journal on Control and
  Optimization}, vol.~33, no.~4, pp. 1040--1060, 1995.

\bibitem{isleyen_vandewouw_arslan_CDC2023}
A.~{\.I}{\c{s}}leyen, N.~van~de Wouw, and {\"O}.~Arslan, ``Adaptive headway
  motion control and motion prediction for safe unicycle motion design,'' in
  \emph{IEEE Conference on Decision and Control}, 2023, pp. 6942--6949.

\bibitem{isleyen_vandewouw_arslan_IROS2023}
------, ``Feedback motion prediction for safe unicycle robot navigation,'' in
  \emph{IEEE/RSJ International Conference on Intelligent Robots and Systems},
  2023, pp. 10\,511--10\,518.

\bibitem{isleyen_vandewouw_arslan_RAL2022}
------, ``From low to high order motion planners: Safe robot navigation using
  motion prediction and reference governor,'' \emph{IEEE Robotics and
  Automation Letters}, vol.~7, no.~4, pp. 9715--9722, 2022.

\bibitem{blanchini_Automatica1999}
F.~Blanchini, ``Set invariance in control,'' \emph{Automatica}, vol.~35,
  no.~11, pp. 1747 -- 1767, 1999.

\bibitem{khalil_NonlinearSystems2001}
H.~K. Khalil, \emph{Nonlinear Systems}.\hskip 1em plus 0.5em minus 0.4em\relax
  Prentice Hall, 2001.

\end{thebibliography}


\appendices 

\section{Proofs}
\label{sec.proofs}

\subsection{Proof of \refprop{prop.move_through_scan_center_convergence}}
\label{app.move_through_scan_center_convergence}

\begin{proof}
Since $\ctrl_{\ctrlgoal, (\scancenter, \scanpoints)}(\pos)$ in \refeq{eq.local_navigation_policy} is a convex combination of $-\gain(\pos - \scancenter)$ and $-\gain(\pos - \ctrlgoal)$, the robot velocity always points towards the tangentially inside of $\conv(\pos, \scancenter,  \ctrlgoal)$, ensuring that the triangle $\conv(\scancenter, \pos, \ctrlgoal)$ shrinks over time along the closed-loop robot trajectory $\pos(t)$ due to the Nagumo theorem of sub-tangentiality for positive invariance \cite{blanchini_Automatica1999}, i.e.
\begin{align*}
\conv(\pos(t), \scancenter, \ctrlgoal) \supseteq \conv(\pos(t'), \scancenter, \ctrlgoal) \quad  \forall t \leq t'.
\end{align*}  
Hence, both the area and perimeter of $\conv(\scancenter, \pos, \ctrlgoal)$ are nonincreasing under the move-through-scan-center policy in \refeq{eq.local_navigation_policy}. 
Note that both the area and the perimeter of $\conv(\scancenter, \pos, \ctrlgoal)$ strictly decrease for a nontrivial triangle with nonzero area, and the robot moves directly towards the goal if the triangle $\conv(\scancenter, \pos, \ctrlgoal)$ is trivial with zero area and collinear vertices. 
Since the scan center is assumed to have positive body clearance from obstacles (\refasm{asm.star_polygon_safety}) and the goal is strictly inside the safe scan polygon, the move-through-scan-center navigation strategy switches in finite time from moving towards the scan center $\scancenter$ to moving towards the goal $\ctrlgoal$.  
Moreover, since the robot always moves towards a visibly safe goal point $\scancenter$ or $\ctrlgoal$ in $\safescanpoly(\scancenter, \scanpoints)$,  the safety of the local navigation policy follows from the positive invariance of the safe scan polygon $\safescanpoly(\scancenter, \scanpoints)$. 
Therefore, using the perimeter of $\conv(\scancenter, \pos, \ctrlgoal)$ as a Lyapunov function, it follows from LaSalle's invariance principle \cite{khalil_NonlinearSystems2001} that the robot's position $\pos$ asymptotically converges to the goal $\ctrlgoal$.
\end{proof}

\subsection{Proof of \ref{prop.move_to_project_goal_convergence}}
\label{app.move_to_project_goal_convergence}

\begin{proof}
The safety of the closed-loop robot motion follows from the positive invariance of the safe scan polygon $\safescanpoly(\scancenter, \scanpoints)$ under the move-to-projected-scan-goal law which is due to due to the Nagumo theorem of the sub-tangentiality condition of positive invariance \cite{blanchini_Automatica1999} since the robot always moves towards a visibly safe projected goal point $\proj_{\pos, (\scancenter, \scanpoints)}(\ctrlgoal)$ in $\safescanpoly(\scancenter, \scanpoints)$.
Moreover, the move-to-projected-scan-goal policy $\overline{\ctrl}_{\goal, (\scancenter, \scanpoints)}(\pos)$, by design, continuously moves the robot directly towards the visible projected goal $\proj_{\pos, (\scancenter, \scanpoints)}(\ctrlgoal)$ which continuously get closer to the goal $\goal$ away from it, i.e.,
\begin{align*}
\norm{\proj_{\pos(t), (\scancenter, \scanpoints)}(\ctrlgoal) - \ctrlgoal} \geq \norm{\proj_{\pos(t'), (\scancenter, \scanpoints)}(\ctrlgoal) - \ctrlgoal} \quad \forall t \leq t'. 
\end{align*}
where the inequality is strict after a finite time since the scan center $\scancenter$ and the goal $\ctrlgoal$ have positive robot-body-clearance from obstacles (\refasm{asm.star_polygon_safety}).
Hence, due to the Nagumo theorem of the sub-tangentiality condition of positive set invariance \cite{blanchini_Automatica1999}, the triangle $\conv\plist{\pos, \proj_{\pos, (\scancenter, \scanpoints)}(\ctrlgoal), \ctrlgoal}$ strictly shrinks over time along the closed-loop motion trajectory $\pos(t)$ towards the projected scan goal $\proj_{\pos(t), (\scancenter, \scanpoints)}(\ctrlgoal)$, i.e.,
\begin{align*}
\conv\plist{\pos(t), \proj_{\pos(t), (\scancenter, \scanpoints)}\!(\ctrlgoal), \ctrlgoal} \supseteq \conv\plist{\pos(t'), \proj_{\pos(t'), (\scancenter, \scanpoints)}\!(\ctrlgoal), \ctrlgoal}
\end{align*} 
for all $t \leq t'$, where the equality only holds if and only if $\pos = \ctrlgoal$. 
Therefore, both the perimeter and the area of $\conv\plist{\pos, \proj_{\pos, (\scancenter, \scanpoints)}(\ctrlgoal), \ctrlgoal}$ strictly decrease%
\footnote{
Let $\overline{\ctrlgoal} := \proj_{\pos, (\scancenter, \scanpoints)}(\ctrlgoal)$ which satisfies $\dot{\overline{\ctrlgoal}} = -\alpha(\pos) (\overline{\ctrlgoal} - \ctrlgoal)$ under the move-to-projected-scan-goal law for some positive function $\alpha(\pos) \geq 0$. Hence, we have from $\dot{\pos} = - \gain(\pos - \overline{\ctrlgoal})$ that
\begin{align*}
\tfrac{\diff}{\diff t} \plist{\norm{\pos - \overline{\ctrlgoal}} + \norm{\overline{\ctrlgoal} - \ctrlgoal}} &= 2\tfrac{\tr{(\pos - \overline{\ctrlgoal})}}{\norm{\pos - \overline{\ctrlgoal}}} (\dot{\pos} - \dot{\overline{\ctrlgoal}}) + 2 \tfrac{\tr{(\overline{\ctrlgoal} - \ctrlgoal)}}{\norm{\overline{\ctrlgoal} - \ctrlgoal}} \dot{\overline{\ctrlgoal}} \\
& \hspace{-20mm} = - 2 \gain \norm{\pos - \overline{\ctrlgoal}} + 2 \alpha(\pos)  \tfrac{\tr{(\pos - \overline{\ctrlgoal})}}{\norm{\pos - \overline{\ctrlgoal}}}(\overline{\ctrlgoal} - \ctrlgoal) - 2 \alpha(\pos) \norm{\overline{\ctrlgoal} - \ctrlgoal}
\\
& \hspace{-20mm} \leq  - 2 \gain \norm{\pos - \overline{\ctrlgoal}}
\end{align*}
which is strictly negative when $\pos \neq \overline{\ctrlgoal}$ and so $\pos \neq \ctrlgoal$.
}%
when $\pos \neq \ctrlgoal$.
Moreover, since the robot moves directly from $\pos$ towards $\proj_{\pos, (\scancenter, \scanpoints)}(\ctrlgoal)$ and the distance $\norm{\proj_{\pos, (\scancenter, \scanpoints)}(\ctrlgoal) - \ctrlgoal}$ is non-increasing, the length $\norm{\pos - \proj_{\pos, (\scancenter, \scanpoints)}(\ctrlgoal)} + \norm{\proj_{\pos, (\scancenter, \scanpoints)}(\ctrlgoal) - \ctrlgoal}$ of the piecewise linear path between $\pos$ and $\ctrlgoal$, connected through the visible projected goal $\proj_{\pos, (\scancenter, \scanpoints)}(\ctrlgoal)$, strictly decreases when $\pos \neq \ctrlgoal$.
Thus, using the total distance $\norm{\pos - \proj_{\pos, (\scancenter, \scanpoints)}\!(\ctrlgoal)\!} + \norm{\proj_{\pos, (\scancenter, \scanpoints)}\!(\ctrlgoal) \!-\! \ctrlgoal}$ of the robot and the goal to the projected goal as a Lyapunov function, it follows from Lyapunov stability theory \cite{khalil_NonlinearSystems2001} that the move-to-projected-scan-goal policy asymptotically and safely brings all robot positions in the positively invariant safe scan polygon $\safescanpoly(\scancenter, \scanpoints)$ to the goal $\ctrlgoal$, which completes the proof.
\end{proof}

\subsection{Proof of \refthm{thm.sequential_composition_convergence}}
\label{app.sequential_composition_convergence}

\begin{proof}
The global convergence of the sequential composition of local scan navigation policies in \refeq{eq.global_navigation_policy} follows from the convergence properties of the local scan navigation policies, due to the guaranteed finite-time, cycle-free, and strictly decreasing cost transitions between these local policies. 
Therefore, the result can be verified using the following facts.

\indent $\bullet$ (Connected Motion Graph) The connectivity of the motion graph $\graph(\scanset)$ ensures that there is a sequence of local controllers whose domain $\safescanpoly(\scancenter_i, \scanpoints_i)$ contains a path joining $\pos, \goal \in \bigcup_{i=1}^{m} \safescanpoly(\scancenter_i, \pointcloud_i)$ since each point in the safe scan polygon is safely connected to the scan center via a straight line, and the centers of adjacent scans in the motion graph are also safely connected by straight lines.  
 
\indent $\bullet$ (Prioritized Local Controllers) Each local navigation policy associated with scan $(\scancenter_i, \scanpoints_i)$ has a priority inversely proportional with its cost-to-go estimate $\policycost_{\goal,\scanset}(i)$ to the goal position $\goal$ in  \refalg{alg.planning_over_star_convex_regions}. Since an active scan in \refeq{eq.active_scan} is selected with the minimum cost-to-go, a transition to another local navigation policy ensures a strict decrease in the cost and thus a finite number of cycle-free transitions.
 
\indent $\bullet$ (Goal Controller) If the goal position $\goal$ is in the active safe scan polygon $\saferpoly{\scancenter_i, \scanpoints_i}$, then the robot asymptotically moves towards the goal $\goal$ due to the convergence property of the local scan navigation policy. 
Note that the active scan might switch to another scan with a smaller cost that contains the goal, but this still ensures that the robot continues to asymptotically move towards the global goal.

\indent $\bullet$ (Local Navigation Transitions) Otherwise, the robot moves under the active local navigation policy with the highest priority (i.e., the smallest cost-to-go estimate to the global goal) containing the robot's position and moves towards its local goal $\policygoal_{\goal, \scanset}(i)$ which is the scan center of the another scan at a smaller cost (see lines 15-17 in \refalg{alg.planning_over_star_convex_regions}). 
Since the local scan navigation policies ensure the positive invariance of their safe scan polygons and the scan centers of adjacent scan regions are strictly contained within their safer scan polygons (\refdef{def.motion_graph}), the robot enters the safe scan polygon of the next local navigation policy in finite time while asymptotically reaching the next scan center. 
Hence, the active scan includes the global goal in finite time after a finite number of transitions between local navigation policies and so guarantees asymptotic global convergence to the global goal.

\indent $\bullet$ (Safety) Finally, the safety of the resulting robot motion follows from the fact that each local navigation policy keeps its safe scan polygon positively invariant and safe scan polygons are assumed to contain no obstacles inside their interiors (\refasm{asm.star_polygon_safety}).   \qedhere
\end{proof}

\subsection{Proof of \reflem{lem.CircularSafety}}
\label{app.CircularSafety}

\begin{proof}
The necessity  can be verified as follows. If $\conv(\scancenter_i, \scancenter_j, \scancenter_k) \subseteq \freespace$, then the points $\scancenter_i$ and $\scancenter_j$ can be sensed by the scan center $\scancenter_k$ without occlusion by obstacles within the maximum safe sensing range $\maxsenserange - \radius$, and so $\scancenter_i, \scancenter_j \in \safescanpoly(\scancenter_k, \scanpoints_k)$.  
Since any point $ \scancenter \in \blist{\scancenter_i, \scancenter_j}$ is within the maximum safe sensing range (i.e., $\norm{\scancenter \!-\! \scancenter_k} \leq \max(\norm{\scancenter_i \!-\! \scancenter_k}, \norm{\scancenter_j \!-\! \scancenter_k})\! \leq\! \maxsenserange \!-\! \radius$) and is visible from~$\scancenter_k$ without occlusion with obstacles (i.e., $\blist{\scancenter, \scancenter_k} \!\subseteq\! \freespace$), we~have $\blist{\scancenter, \scancenter_k} \! \subseteq \! \safescanpoly(\scancenter_k, \scanpoints_k)$, implying that  $\blist{\scancenter_i, \scancenter_j} \! \subseteq \! \safescanpoly(\scancenter_k, \scanpoints_k)$.

The sufficiency is due to \refasm{asm.star_polygon_safety} and the star convexity of safe scan polygons: If $\blist{\scancenter_i, \scancenter_k} \! \in \! \safescanpoly(\scancenter_k, \scanpoints_k)$, then $\conv(\scancenter_i, \scancenter_j, \scancenter_k) \subseteq \safescanpoly(\scancenter_k, \scanpoints_k) \subseteq \freespace$, where the first inclusion follows from star convexity and the last is due to \refasm{asm.star_polygon_safety}. Thus, this holds similarly for the rest. 
\end{proof}

\end{document}